\documentclass[twoside]{article}

\usepackage[accepted]{aistats2022}
\usepackage[round]{natbib}

\usepackage[utf8]{inputenc} 
\usepackage[T1]{fontenc}    
\usepackage{hyperref}       
\usepackage{url}            
\usepackage{booktabs}       
\usepackage{amsfonts}       
\usepackage{nicefrac}       
\usepackage{microtype}      
\usepackage{xcolor}         

\usepackage{subcaption}
\usepackage{url}
\usepackage{multirow}
\usepackage{tabularx}
\usepackage{amsmath}
\usepackage{amsthm}
\usepackage{amssymb}
\usepackage{commath}
\usepackage{dsfont}
\usepackage{graphicx}

\newtheorem{theorem}{Theorem}
\newtheorem{fact}{Fact}
\newtheorem{definition}{Definition}
\newtheorem{proposition}{Proposition}

\newtheorem{claim}{Claim}

\newtheorem{lemma}{Lemma}

\DeclareMathOperator{\sign}{sign}
\newcommand{\hypothesis}{\mathcal{H}}
\newcommand{\explanation}{\mathcal{E}_{h^*}}
\newcommand{\margin}{\mathcal{M}_r(\mathcal{X})}
\newcommand{\features}{\mathcal{X}}

\newcommand{\consistent}{\mathcal{H}_C}

\newcommand{\VC}{\text{VC}}

\newcommand{\Bernoulli}{\mathrm{Bernoulli}}
\newcommand{\region}{\mathcal{R}_r}
\newcommand{\Xcal}{\mathcal{X}}
\newcommand{\Hcal}{\mathcal{H}}
\newcommand{\Ucal}{\mathcal{U}}
\newcommand{\Ncal}{\mathcal{N}}

\newcommand{\Rcal}{\mathcal{R}}
\newcommand{\Ecal}{\mathcal{E}}
\newcommand{\Mcal}{\mathcal{M}}
\newcommand{\PP}{\mathbb{P}}
\newcommand{\RR}{\mathbb{R}}
\newcommand{\inner}[2]{\langle #1,#2 \rangle}

\newcommand{\hide}[1]{}

\usepackage{natbib}
\setcitestyle{authoryear}

\begin{document}

%

%

\twocolumn[

\aistatstitle{Margin-distancing for safe model explanation}

\aistatsauthor{Tom Yan \And Chicheng Zhang}

\aistatsaddress{Carnegie Mellon University \And University of Arizona} ]

\begin{abstract}
The growing use of machine learning models in consequential settings has highlighted an important and seemingly irreconcilable tension between transparency and vulnerability to gaming. While this has sparked sizable debate in legal literature, there has been comparatively less technical study of this contention. In this work, we propose a clean-cut formulation of this tension and a way to make the tradeoff between transparency and gaming. We identify the source of gaming as being points close to the \emph{decision boundary} of the model. And we initiate an investigation on how to provide example-based explanations that are expansive and yet consistent with a version space that is sufficiently uncertain with respect to the boundary points' labels. Finally, we furnish our theoretical results with empirical investigations of this tradeoff on real-world datasets.
\end{abstract}

\section{INTRODUCTION}
With the increasing use of machine learning models in automating decision making, there is growing concern over the opacity of these models. Such concerns have given rise to laws, such as the European GDPR, which aim to provide a ``Right to Explanation''\citep{wachter2017right, edwards2017slave, selbst2018meaningful}. However, one stumbling block to this solution is the tension between transparency and gaming: greater transparency into the model gives rise to gaming -- individuals strategically misreporting their features to induce desired classification outcomes from the ML model. 

As a result, some government agencies are still to this day reluctant about revealing details on the deployed algorithms. This has in turn lead to Freedom of Information requests, such as those submitted by civil interest groups in the Netherlands, calling for greater transparency \citep{wieringa2020account}, as well as organized movements such as the OpenSCHUFA project~\citep{openschufa}, through which citizens take matters in their own hands and try to crowd-source data in an effort to reverse-engineer the algorithms. 

In this work, we formalize this tension in a natural, formal model, which to the best of our knowledge, is the \emph{first formal model} capturing the tradeoff between transparency and gaming in machine learning.

The setting we will study is one where an organization uses model $h^*: \features \to \cbr{-1,+1}$ to perform classification over feature space $\features$ and provides transparency through model explanations. We focus on example-based explanations $\explanation$, which have been found to be one of the most intuitive types of explanations in a recent human study \citep{jeyakumar2020can}, and in particular on prototype-based explanations (e.g $k$-medoid or MMD-critic \citep{kim2016examples}). 

In more detail, the explanation mechanism $\explanation: \Xcal \to 2^{\Xcal}$ will select a representative subset of $\features$ to label and explanations $\{(x, h^{*}(x)) \mid x \in \explanation(\features)\}$ will be released. For example, for loan applications, such explanation could be in the form of past, anonymized (un)successful profiles.

Intuitively, the concern with releasing explanations is that applicants may use the knowledge of the hypothesis class $\Hcal \ni h^*$ along with the explanations to construct the version space (VS), $\consistent = \{h \in \Hcal \mid h(x) = h^{*}(x), \forall x \in \explanation(\features)\}$, to infer $h^*$. If the explanation is ``good'' and allows for ``simulatability'' of $h^*$ \citep{murdoch2019definitions}, then the few models in $\consistent$ would be constrained by the explanations to have very similar predictions on $\Xcal$ as $h^*$. And so, even though the VS does not \emph{directly} identify $h^*$, the VS allows one to estimate $h^*$'s prediction with high certainty. This we will be formalize soon.

To address this issue, we propose \emph{margin-distancing} as a simple and general method that can make the \emph{tradeoff} between transparency and gaming. We show that with margin-distancing it need not be one or the other: it is possible to offer individuals \emph{some idea} of how the model works while still preventing gaming. 

Concretely, given classification models $h^*$ and input example $x$, we use $f^*: \Xcal \to \RR$ to denote a function that outputs an underlying margin score, $h^*(x) = \sign(f^*(x))$, where $\sign(a) = +1$ for $a \geq 0$ and $\sign(a) = -1$ otherwise.
Margin-distancing selects a subset of $\Xcal$ whose margin score $\abs{f^*(x)}$ is greater than some threshold $\alpha$. This is done to induce a sufficiently large $\consistent$ and, as a result, sufficiently low certainty on how $h^*$ predicts to dissuade gaming. 

This approach is compatible with any example-based explanations. We note that our approach is also applicable with local surrogate based methods with bounded fidelity region. Indeed, these methods may be viewed as example-based explanation methods that impart labels for all points within the fidelity regions.

\textbf{Our Contributions:} 

(1) We formalize the tradeoff between transparency and gaming, and propose \emph{margin-distancing} as a way of making this tradeoff. 

(2) We prove that margin-distancing does \emph{monotonically} decreases decision boundary certainty under a uniform prior over homogeneous linear models and spherical feature space. We also give a set of complementary negative results showing that monotonicity does not hold in general.

(3) We evaluate boundary points' certainty using sampling for general model classes. Our empirical studies suggest margin-distancing does reduce boundary certainty in a relatively monotonic fashion, and in some cases, completely monotonically, which would enable binary search as a computationally efficient means of finding the optimal amount of explanations to release.

\section{RELATED WORKS}
\textbf{Transparency vs Gaming:} To the best of our knowledge, there has been only one technical paper~\citep{tsirtsis2020decisions} that examines the tension between explanation and gaming. In this work, an organization focuses on releasing an optimal set of counterfactual explanations $S$ to induce agents to change their reports in a way that maximizes the organization's utility; this work does not focus on examining the tradeoff explored in our paper. Moreover, the key assumption that differs from our setting is that all feature alteration is viewed as being causal. Lastly, in our work, we do not assume that agents can only change to points in $S$ (if possible), but rather to any point $\hat{x}$ in the neighborhood of $x$.

\textbf{Strategic ML:} Similar to most of strategic classification literature \citep{hardt2016strategic, dong2018strategic, kleinberg2020classifiers, chen2018strategyproof}, we assume strategic behavior is gaming. However, different from most, past formulations, agents in our setting do not have \emph{full knowledge} of $h^*$ and have to best respond with only partial knowledge (explanations) of $h^*$.

In the interest of space, we have included further related works on topics including Improvement vs Gaming, Explanation Manipulation in Appendix \ref{sec:add-related-works}.

\section{PROBLEM FORMULATION}
\textbf{Gaming:} We assume all individuals desire to be classified the positive label (e.g ``loan granted'') by $h^*$.
An individual with profile $x$ may use the explanations of $h^*$ to compute and misreport $\hat{x} \neq x$ so as to improve the chance of being classified as the positive label. As is standard in strategic classification, this act of misreporting is referred to as \emph{gaming}~\citep{hardt2016strategic}. 

In face of gaming, the organization wishes to have its predictions be unaffected by the release of explanations $\explanation(\features)$: $h^*(\hat{x}) = h^*(x)$, $\forall x \in \mathcal{X}$. 

For our analysis, we first assume that applicants cannot report arbitrary profiles -- otherwise everyone will simply report some $x \in \explanation(\features)$ with a positive label. This assumption may also be motivated as follows: in strategic ML literature, individuals are typically assumed to have a cost function. This naturally induces a region beyond which it is too costly to change to. For modeling purposes, we assume that if an applicant has feature $x$, then $\hat{x} \in \region(x) := \{x' \mid \| x - x' \| < r, x' \in \features \}$, with $r > 0$ being the maximum extent of manipulation. Additionally, we assume that applicants are aware of the model class $\hypothesis \ni h^*$ used by the organization.

Next, since the explanations only allow one to conclude that $h^* \in \consistent$, we need to specify how individuals reason about whether to misreport $x'$ or report $x$ truthfully with only \emph{partial knowledge} about $h^*$. To model this calculus, as is common in Economics, we assume that the individual is Bayesian and calculates the \emph{increased} chance of obtaining positive label under $x'$ instead of $x$ through a prior distribution $\mathcal{U}$ that gets updated to posterior $\Ucal(\Hcal_C)$ (the restriction of $\Ucal$ on the set $\Hcal_C$) with knowledge of $\explanation(\features)$:
\begin{align*}
    \pi(x, x')  & = \Pr_{h \sim \mathcal{U}(\consistent)}(h(x') = 1) \\
    & \quad \quad \quad -  \Pr_{h \sim \mathcal{U}(\consistent)}(h(x) = 1).
\end{align*}
A natural choice for $\mathcal{U}$ is the uniform distribution, though it need not be so. We assume that the organization also knows $\Ucal$.

Naturally, individuals will choose to misreport if there is a sufficiently high certainty of success, since they obtain positive utility for getting the positive label (i.e if $h^*$ is s.t $h^*(\hat{x}) = 1$).  However, in misreporting, they incur negative utility for the cost of manipulation: $x \rightarrow x'$. These two may be weighted linearly in rational agents or nonlinearly in behavioral agents due to risk-aversion \citep{kahneman2013prospect}. Following the formal model of the rationality of crime as introduced by Becker \citep{becker1968crime}, we abstract this away by assuming that there is some threshold $\kappa$ such that if $\pi(x, x') \leq \kappa$, the individual is too risk-averse to misreport $\hat{x} = x'$: the cost of manipulation offsets the increased likelihood of obtaining positive utility through positive classification. 

This brings us to our main insight: we only need $\consistent$ to be sufficiently ambiguous near the decision boundary because \emph{only} individuals with points near the boundary can misreport in a way that flips $h^*$'s prediction.

Formally, define the set of 
\emph{boundary points} to be all $x$'s where such a label flip is possible: $\Ncal_r(\Xcal) := \{x \in \features \mid \exists x' \in \region(x) \land h^*(x') \neq h^*(x) \}$. 
Similarly, we define \emph{boundary pairs} to be pairs $(x, x')$ that are within a distance of $r$, but predicted differently by $h^*$; formally, $\Mcal_r(\Xcal) := \cbr{ (x,x') \in \Xcal^2 \mid x' \in \region(x) \land h^*(x') \neq h^*(x) }$. Observe that $\Mcal_r(\Xcal) \subset \Ncal_r(\Xcal)^2$.


\textbf{Margin-distancing:} To make it difficult to infer the decision boundary through $\consistent$, it is natural to remove explanations that are close to the decision boundary. This gives rise to our approach of \emph{margin-distancing}. We will designate some indicator function $\Lambda_{\alpha}$ for choosing explanations, which evaluates to $1$ iff the examples' classification margin score is greater than cutoff $\alpha$; formally, $\Ecal_{h^*}(\Xcal, \alpha) = \cbr{x \in \Xcal: \Lambda_{\alpha}(x) = 1}$.
Note that $\consistent$ is a function of $\alpha$, since $\consistent$ is a function of the explanations, which are in turn a function of $\alpha$. Intuitively, a big $\alpha$ that only retains explanations with large margins would decrease \emph{boundary certainty}, which we define as  $\max_{(x,x') \in \Mcal_r(\Xcal)}\pi(x,x')$.

\textbf{Policy Goals:} Herein lies the tradeoff for the organization:

1) Provide explanation $\Ecal_{h^*}(\Xcal, \alpha)$ such that the boundary certainty is made sufficiently low: $\max_{(x,x') \in \Mcal_r(\Xcal)}\pi(x,x') \leq \kappa$. This makes all individuals $x \in \Xcal$ too risk-averse to misreport $\hat{x} \in \Rcal_r(x)$ with $h^*(\hat{x}) \neq h^*(x)$, thus preventing gaming.



2) The explanation provided $\Ecal_{h^*}(\Xcal, \alpha)$ is as transparent as possible. That is, $\alpha$ is as small as possible to retain as many explanations from the full set of explanations as possible. Naturally, in our setting, we define transparency to be the amount of explanations that remain after margin-distancing.

The technical problem we study is:


\begin{quote}
    How can we search for the smallest threshold $\alpha$ possible such that $\max_{(x,x') \in \Mcal_r(\Xcal)}\pi(x,x') \leq \kappa$, which is needed to prevent gaming?
\end{quote}

\begin{figure*}
\begin{subfigure}[b]{0.332\textwidth}
    \includegraphics[width=\linewidth]{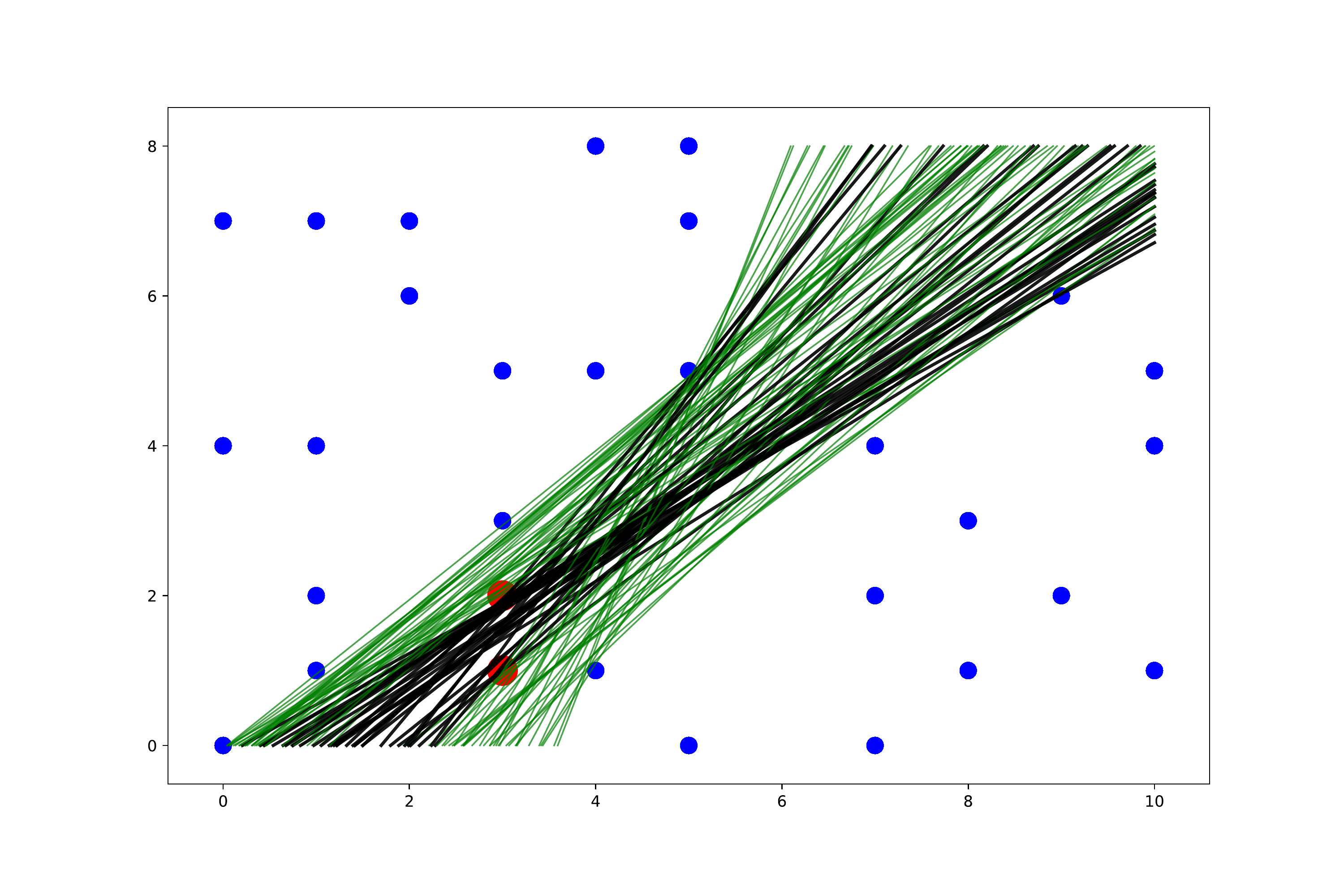}
\end{subfigure}%
\begin{subfigure}[b]{0.332\textwidth}
    \centering
    \includegraphics[width=\linewidth]{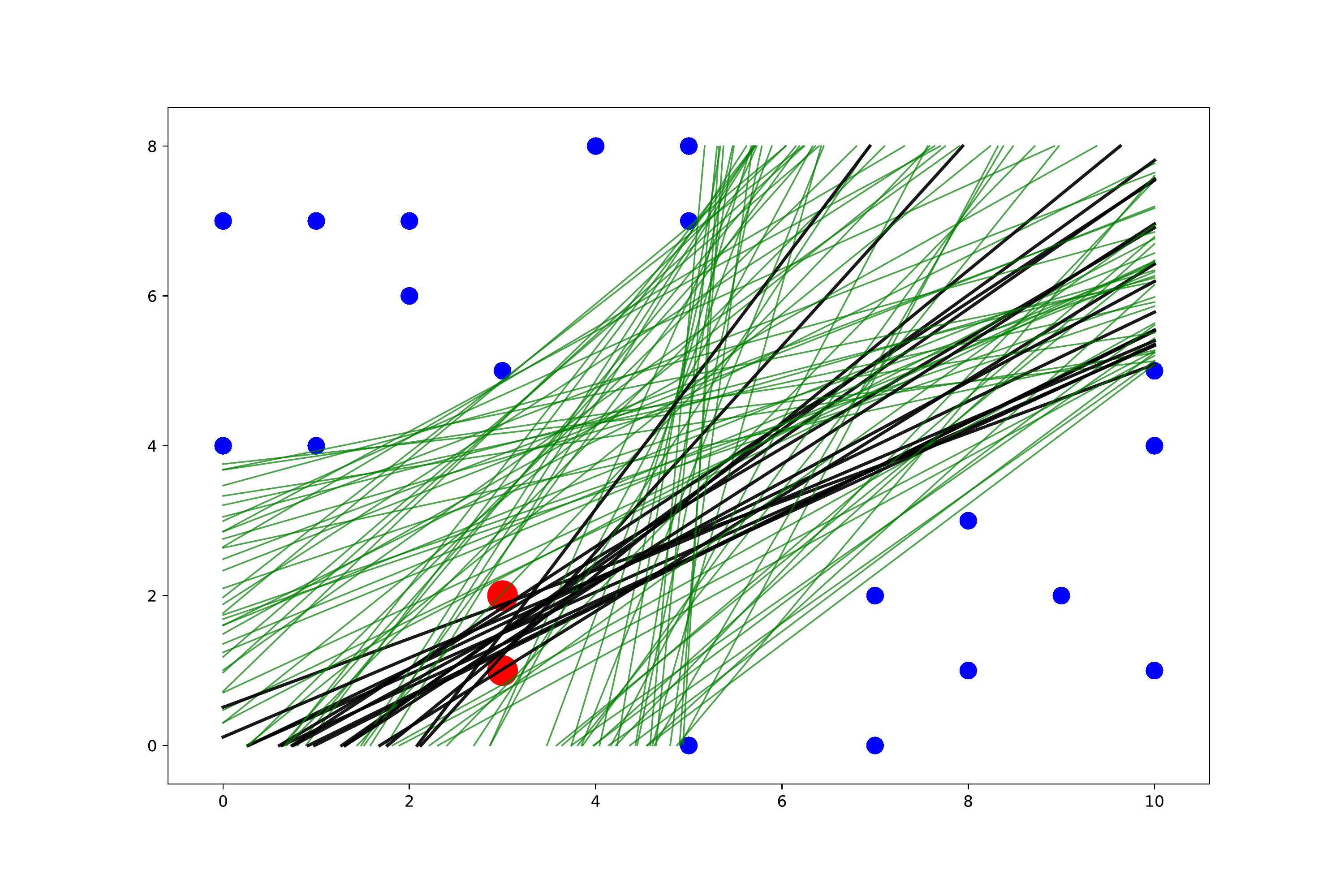}
\end{subfigure}%
\begin{subfigure}[b]{0.332\textwidth}
    \includegraphics[width=\linewidth]{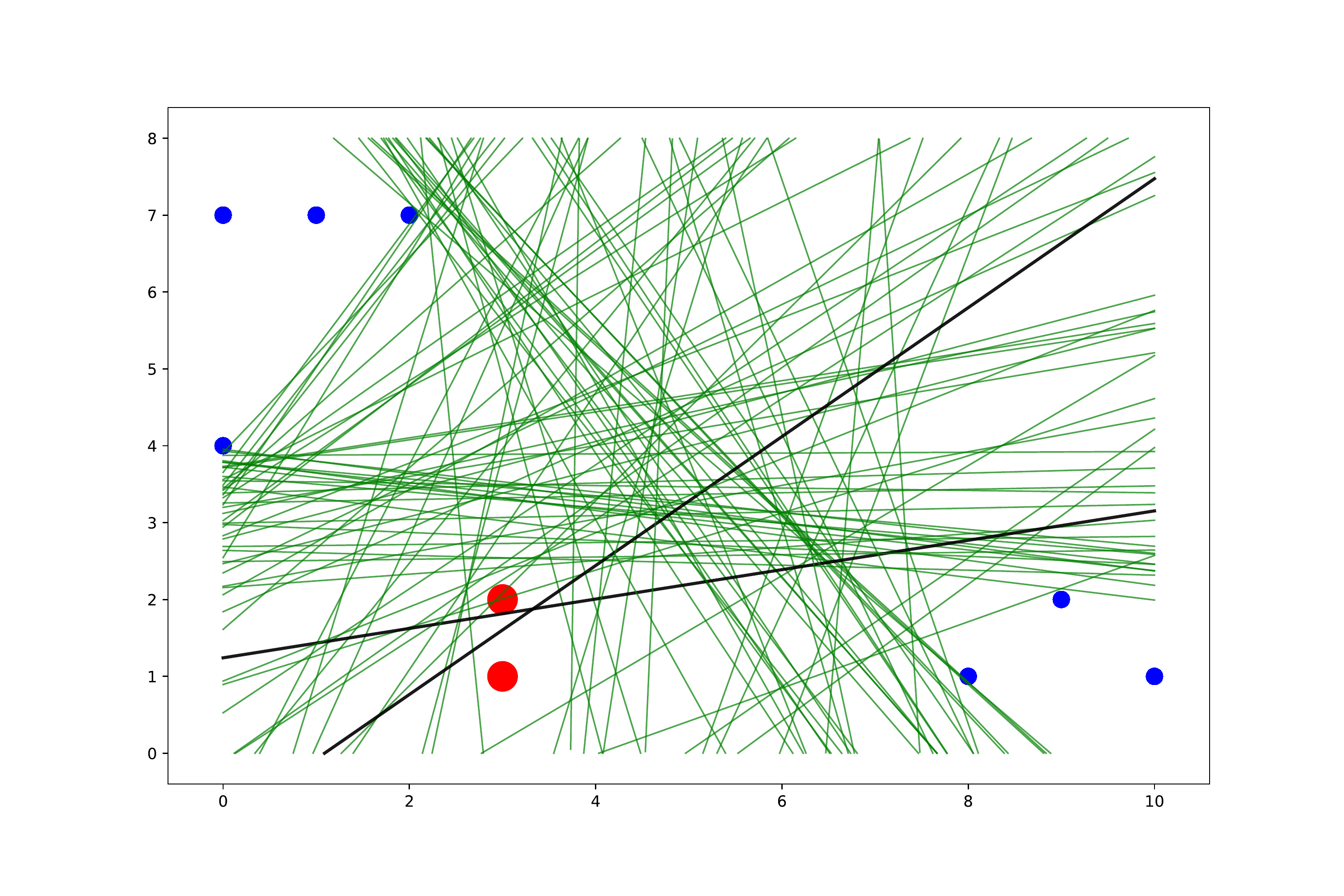}
\end{subfigure}%
\caption{Visualization of $\consistent$ in a toy example where the amount of explanations (blue points) is varied ($80, 50, 20$ percent of all explanations is kept). In red is one randomly chosen boundary pair. $100$ lines (green and black) are randomly sampled from $\consistent$; in black are lines that predict the pair like $h^*$ (opposite labels), and green the same.}
\label{fig: 2d_linear}
\end{figure*}

Before we proceed, we obtain some intuition first through a qualitative visualization of $\consistent$ in a toy example, Figure~\ref{fig: 2d_linear}. This figure helps to confirm that allowing explanations with small margins ``boxes in'' the version space too much, and makes models in $\consistent$ too similar to $h^*$. And so, removing explanations with small margin help enlarge $\consistent$ and decrease boundary-certainty.

\textbf{Simple Example:} Next, for a quantitative toy example, consider when $\Xcal = [0, 1]$ and $\Hcal = \{h_w(x) := \sign(x - w) \mid w \in [0, 1]\}$ is the class of 1D thresholds. Let $\Ucal$ be the uniform distribution over $\Hcal$. 
We know then that $w^* \in [x^-, x^+]$, where $x^-$ is the largest negative point in $\explanation(\features)$ and $x^+$ the smallest positive point. Therefore, for $x \in (x^-, x^+)$ and some $x' \in \Rcal_r(x) > x$, we have that $\pi(x, x') = \frac{\min\{x', x^+\} - x}{x^+ - x^-}$. In this case, it is evident that margin-distancing (i.e increasing $x^+$ and decreasing $x^-$) decreases boundary certainty $\pi(x, x')$.

In the section that follow, we study a more general hypothesis class and verify that the intuitive trend of removing information around the decision boundary does make it more difficult to infer the decision boundary, thus reducing boundary certainty.

\section{HOMOGENEOUS LINEAR MODELS}
\label{sec:sphere}

We focus our theoretical study on the property of \emph{monotonicity}, which if true, allows for binary search as an efficient way to compute the optimal $\alpha$. In this section, we identify homogeneous linear models in $\RR^d$,  i.e. $\hypothesis = \cbr{h_w \mid \| w \|_2 = 1}$ (where $h_w := x \mapsto \sign(\inner{w}{x})$), as one setting where margin-distancing monotonically leads to decreased boundary certainty.

For the results that follow, we also assume that individuals have uniform prior $\Ucal$ over $\Hcal$. We will also focus on when the feature space $\Xcal$ is the origin-centered unit sphere, i.e., $\Xcal = \cbr{x \in \RR^d \mid \| x \|_2 = 1}$, which means that $r \leq 2$. Intuitively, this corresponds to a normalized dataset with profiles of ``all kinds'', which is not unreasonable for profiles of a general population. We handle more general settings in the following section. 

For linear models, it is natural to take $\Lambda$ to be a function of the margin of a point with respect to $w^*$ (the parameter of $h^*$):  $\Lambda_{\alpha}(x) = \mathds{1}\{ \abs{\langle w^*, x \rangle} > \alpha\}$, for $\alpha \in [0,1)$. Therefore, for every $\alpha$, its associated set of explanations is $\Ecal_{h^*}(\Xcal, \alpha) = \cbr{x \in \Xcal: \abs{\langle w^*, x \rangle} > \alpha}$.

Under this ``nice'' setting, we first show that we can give a simple characterization of the version space in terms of $\alpha$:


\begin{lemma}
\label{lemma: circle_vs}
Fix $\alpha \in [0,1)$. Recall that $\consistent = \{h \in \Hcal \mid h(x') = h^*(x'), \;  \forall x' \in \Ecal_{h^*}(\Xcal, \alpha) \} $ is the version space induced by explanation $\Ecal_{h^*}(\Xcal, \alpha)$. $\consistent$ can be equivalently written as:  
\[
\consistent = \cbr{ h_w \mid \| w \|_2 = 1,  w \cdot w^* \geq \sqrt{ 1 - \alpha^2} }.
\]
\end{lemma}

For ease of the exposition of the next theorem, we reason in the spherical counterpart to $\alpha$ and $r$:
\begin{itemize}
    \item Define $\phi$ to be the maximum angle between any $w \in \consistent$ and $w^*$. From Lemma~\ref{lemma: circle_vs}, under explanation $\Ecal_{h^*}(\Xcal, \alpha)$, $\phi = \arccos (\sqrt{1 - \alpha^2}) = \arcsin\alpha$. Intuitively, $\phi$ measures how large $\consistent$ is and shrinks with a bigger set of explanations.
    
    \item Define $\psi = 2 \arcsin(\frac{r}{2}) = \arccos (1 - \frac{r^2}{2})$. The boundary region $\Ncal_r(\Xcal)$ may then be described as the set of points $\{x \in \features \mid \langle w^*, x\rangle \in [-\sin\psi, \sin\psi) \}$.  Intuitively, $\psi$ measures how ``thick'' the boundary region is. Geometrically, this means that $\theta(x, w^*) \in [\pi/2 - \psi, \pi / 2 + \psi]$ for $x$ in the boundary region, where $\theta(x, w^*)$ denotes the angle between $x$ and $w^*$ the decision boundary: $\theta(u,v) = \arccos(\frac{\inner{u}{v}}{\|u\|_2 \| v\|_2}) \in [0, \pi]$.
    

    %


\end{itemize}

Please refer to Figure~\ref{fig:version_space} for an illustration of notation $\phi$ and $\psi$, which we note are both acute by definition, and refer to Table~\ref{tab:sphere} for a summary of definitions. 

\begin{table*}
  \centering
\begin{tabular}{ |l|l| }
  \hline
  $r$ & max extent of manipulation\\
  $\alpha$ & min distance from the margin \\
  $\Pi(\alpha)$ & boundary certainty, $\Pi(\alpha) = \max_{(x, x') \in \margin} \pi_{\alpha}(x, x')$ \\
  $\phi$ & max angle between $w \in \consistent$ and $w^*$; related to $\alpha$ by $\alpha = \sin \phi$ \\
  $\psi$ & max angle: related to $r$ by $\cos \psi =  1 - r^2 / 2$ \\
  \hline
\end{tabular}
\caption{A table of notations that appears in Section~\ref{sec:sphere}.}
\label{tab:sphere}
\end{table*}

\begin{figure}
    \centering
    \includegraphics[scale=0.25]{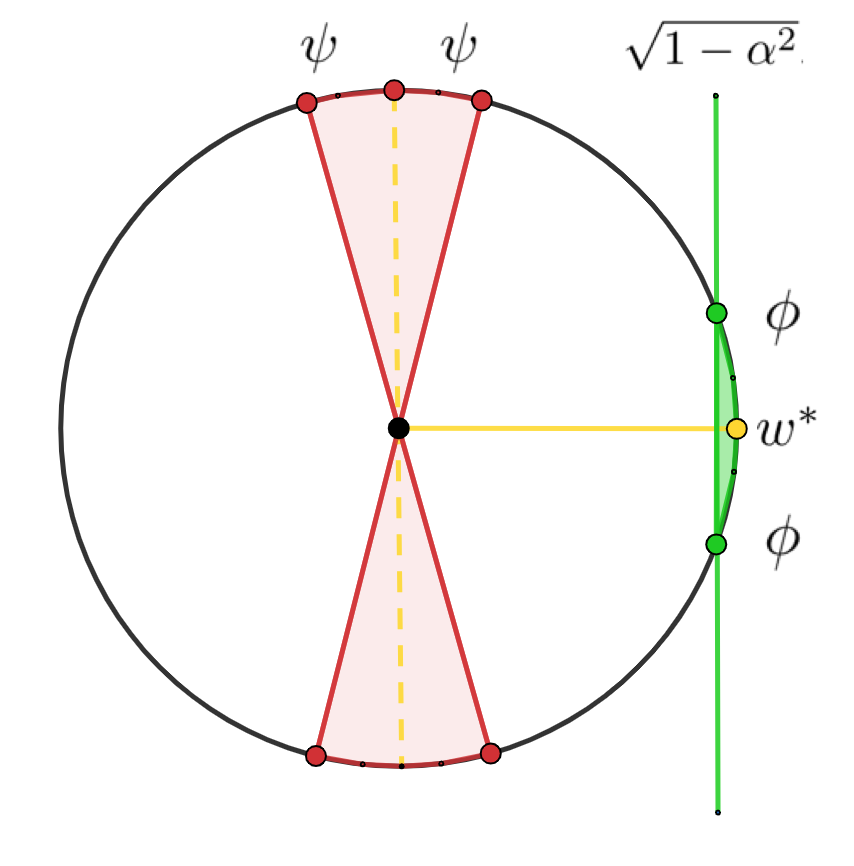}
    \caption{Visualization of the notation: $\consistent$ in green, boundary region in red and true model $w^*$ in yellow.}
    \label{fig:version_space}
\end{figure}




Firstly, it is clear that increasing boundary thickness $\psi$ leads to a larger $\Mcal_r(\Xcal)$, therefore a higher $\max_{(x, x') \in \margin} \pi_{\alpha}(x, x')$. 
We derive an analytical form of $\max_{(x, x') \in \margin} \pi_{\alpha}(x, x')$ below that formalizes this.






\begin{theorem}
We have:

\[
\max_{(x, x') \in \margin} \pi_{\alpha}(x, x') = \begin{cases}
\frac{\int_0^{\psi/2} F(\theta) d\theta}{\int_0^\phi F(\theta) d\theta }  & \psi \leq 2\phi \\
1 & \psi > 2\phi,
\end{cases}
\]
where $F(\theta) = (1 - \frac{\cos^2 \phi}{\cos^2 \theta})^{d / 2 - 1}$; therefore, it is strictly increasing for $\psi$ in $[0, 2\phi]$.
\label{thm:max-pi-sphere}
\end{theorem}

Our next two theorems consider the margin-distancing effect in terms of $\alpha$. For simplicity and to relate $\alpha$'s effect on $\consistent$ through explanations $\Ecal_{h^*}(\Xcal, \alpha)$, we subsequently abbreviate boundary certainty $\max_{(x, x') \in \margin} \pi_{\alpha}(x, x')$ as $\Pi(\alpha)$. 

To recap, a higher threshold $\alpha$, corresponding to more margin-distancing, leads to a smaller set of explanations $\Ecal_{h^*}(\Xcal, \alpha)$ (lowered transparency since more explanations are removed) and thus a bigger $\consistent$. This leads to lower boundary certainty $\Pi(\alpha)$, preventing gaming.

In the next result, we show that $\Pi(\alpha)$ is provably monotonically decreasing in $\alpha$. Thus, this enables the use of binary search to efficiently find the optimal $\alpha$. Indeed, it is not clear that decreasing the amount of explanations and enlarging the version space will always decrease $\Pi(\alpha)$. The reason is that enlarging $\consistent$ increases both  models that agree with $h^*$ on $x, x'$ (black lines in Figure~\ref{fig: 2d_linear}) and models that do not (green lines). If \emph{proportionally} more of them do predict like $h^*$, then the new $\pi_{\alpha}(x, x')$ will actually increase. We prove Theorem~\ref{thm:monotonicity} that shows this is not so in this ``nice'' setting; the proof may be found in Appendix~\ref{sec:proofs-sphere}.

\begin{theorem}
$\Pi(\alpha)$ is decreasing in $\alpha$, for $\alpha \in [0, 1)$, and is strictly decreasing in $[\sin(\psi/2), 1)$. 
\label{thm:monotonicity}
\end{theorem}

Finally, in some cases, we may skip the search if we can analytically derive conditions on $\phi, \psi$ in which $\Pi(\alpha)$ is upper bounded. Next, we show that there exists some constant $c$ such that $\lim_{\alpha \to 1} \Pi(\alpha) \leq c \psi$. Thus, when $\psi$ is small and $\alpha$ increases to $1$, $\Pi(\alpha)$ decreases to a small value. 



\begin{theorem}
\label{thm:pi-bounds}
\begin{enumerate} 
\item If $\alpha \geq 1 - \frac{1}{8d}$, then $\Pi(\alpha) \leq 9 \psi$.
\item For any $C_1 \in (0,1)$, there exists $C_2 > 0$ such that the following holds: if $\alpha \leq 1 - \frac{1}{\sqrt{d}}$ and $\psi \geq \frac{C_2}{d^{1/4}}$, then $\Pi(\alpha) \geq 1 - C_1$.
\end{enumerate}
\end{theorem}

A more refined version of this theorem and proofs of other theorems may be found in Appendix~\ref{sec:proofs}.

\section{GENERAL MODELS}
For arbitrary feature spaces, it is unclear if it is possible to explicitly characterize $\consistent$ even for non-homogeneous linear models. Still, let us suppose we have devised some function $\Lambda$ parameterized by threshold parameter $\alpha$. Algorithmically, how do we search for the smallest $\alpha$ such that $\Pi(\alpha) < \kappa$ for a given $\kappa$?

First, we will need an approach to approximate  $\Pi(\alpha)$ under a given threshold $\alpha$. Indeed, there is generally no closed-form expression for $\Pi(\alpha)$, so we will assume access to an algorithm that can sample from the posterior  distribution $\mathcal{U}(\consistent)$. Our approach is simply to draw samples $h_1, ..., h_n$ using the algorithm and evaluate: $\hat{\rho}(x') - \hat{\rho}(x) = \frac{1}{n} \sum_{i=1}^n \mathds{1}\{ h_i(x') = 1\} - \frac{1}{n} \sum_{i=1}^n \mathds{1}\{ h_i(x) = 1\}$.

To understand the sample complexity needed, we see that, $\hat{\rho}(x) = \frac{1}{n} \sum_{i=1}^n \mathds{1}\{ h_i(x) = 1\} = \frac{1}{n} \sum_{i=1}^n \mathds{1}\{ H^*_{x}(h_i) = 1\}$, where for a fixed $x$, $H^*_{x}: h \mapsto h(x)$ is its associated \emph{dual} function.

\begin{definition}[Dual Class]
For any domain $\features$ and set of functions $\Hcal$ whose image is $\cbr{-1,+1}$, the dual class of $\Hcal$ is defined as $\Hcal^* := \{ H_x^{*} \; | \; x \in \features \}$. 
\end{definition}

As introduced in \citep{assouad1983densite}, $\VC(\Hcal^*)$ is finite as long as $\VC(\Hcal)$ is finite. And so, with $O\left( \frac{\VC(H^{*}) + \log 1 / \delta}{\epsilon^2} \right)$ random draws, we may obtain an $2\epsilon-$accurate estimation of $\hat{\rho}(x) - \hat{\rho}(x')$ for \emph{all} boundary pairs $x, x'$, due to uniform convergence. This gives us a $4\epsilon$-accurate estimation of $\Pi(\alpha)$. In the case of linear models, due to point-line duality, we know that $\VC(\Hcal^{*}) = \VC(\Hcal) = O(d)$, which informs us how many samples are needed to calculate a high fidelity approximation of $\pi_{\alpha}(x, x')$.

\textbf{Search:} Once we know how to approximate $\max \pi_{\alpha}(x, x')$ for a given $\alpha$, if monotonicity does hold, then search for the optimal threshold may be efficiently done through binary search. Recall from Theorem~\ref{thm:monotonicity} that, if a) the feature space is spherical, and b) the prior distribution over the hypothesis class is uniform, and c) the hypothesis class is  homogeneous halfspaces, then $\Pi(\alpha)$ decreases monotonically to $O(\psi)$. To complement this result, we next show that removing one of a, b or c (and keeping the rest) breaks this pattern.

Our next two proposition show that, removing the spherical feature space condition, or removing the assumption of $\Ucal$ being uniform, can cause boundary certainty to \emph{increase} with increasing margin distancing parameter $\alpha$ in worst-case settings.




\begin{proposition}
Suppose $d=2$. We have uniform prior over homogeneous linear models $\hypothesis = \{w \in \mathbb{R}^d \mid \|w\| = 1\}$, there exists a feature space $\features$ and thresholds $0 < \alpha_2 < \alpha_1$ such that $\Pi(\alpha_2) < \Pi(\alpha_1)$.
\end{proposition}

\begin{proposition}
Suppose $\features$ is the $d$-dimensional unit sphere with $d \geq 3$. There exists a non-uniform distribution $\Ucal$ over homogeneous linear models $\hypothesis$,  such that there exists thresholds $0 < \alpha_2 < \alpha_1$ with $\Pi(\alpha_2) < \Pi(\alpha_1)$.
\end{proposition}

Finally, we show that by removing the assumption that the hypothesis class is the set of homogeneous linear models, $\Pi(\alpha)$ can stay at a high value for all $\alpha \in (0,1]$ and all $\psi \in (0,\pi]$. This is in sharp contrast with the homogeneous linear model class setting, in which $\lim_{\alpha \to 1} \Pi(\alpha) \leq O(\psi)$ and could thus be made arbitrarily small with $\psi \rightarrow 0$.


\begin{proposition}
There exists a class of non-homogeneous linear models, with spherical $\features$  such that $\Pi(\alpha)$ decreases monotonically (and strictly so at some point) with increasing $\alpha$, and yet $\Pi(\alpha) \geq 1 / 3$ for all $\alpha \in [0,1)$ and $\psi \in (0,\pi]$.
\end{proposition}

Thus, we have that in general monotonicity does not hold. However, our negative results are worst-case in nature. Next, we turn to experiments to examine the relationship between margin-distancing and boundary-certainty on real-world, non-worst case datasets.

\section{EXPERIMENTS}
In this section, we empirically chart the relationship between margin distancing (the amount of explanation omission) and boundary certainty. We experiment with linear and multi-layer Perceptron (MLP) models. 

\textbf{Explanation Methods:} As mentioned in the formulation, we focus on example-based explanation methods that can return a subset of prototypical instances that serve as explanations. This leads us to use $k$-medoid and MMD-critic~\citep{kim2016examples}, and rules out other example-based explanation methods such as~\citep{koh2017understanding} that return a single (and not subset), most ``influential'' data point out of the training set. Note also, that counterfactual and contrastive-based explanations are ruled out by the need to margin-distance. Indeed, by construction, counterfactual/contrastive-based explanations are boundary points, whose release greatly increase the users' boundary certainty -- in fact, $\max_{(x, x') \in \margin} \pi(x, x') = 1$. Thus, if manipulation (gaming) is to be prevented, the use and release of this type of explanations is a non-starter.

Our experimental procedure goes as follows:

1) The explanation method (e.g $k$-medoid) is used to compute the full set of explanations.
    
2) Then, we vary the degree of margin-distancing and remove explanations that are too close to the decision boundary. To measure the closeness of an explanation point with respect to the decision boundary, we look at its percentile in the distribution of all explanations’ margin scores. This allows us to identify which points are in the top $l$ percent of all explanations closest to the margin. We do this separately for positive and negative explanations as they have different distributions of margin scores.
    
3) To compute boundary certainty, we remove this top $l$ percent closest explanations, compute models $\consistent$ consistent with the remaining explanations $\Ecal_{h^*}(\Xcal, \alpha)$ and compute $\pi(x, x')$ using $\consistent$.
    
4) To generate our plots, we vary $l$ for $l$ ranging from $0$ to $75$ (on the x-axis) and plot this against three metrics that capture boundary certainty (on the y-axis). The three metrics that summarize $\pi(x, x')$ for all boundary pairs $(x, x')$ are: $\max_{(x, x') \in \margin} \pi(x, x')$ (worst boundary pair), average of top $5$ percent of $\pi(x, x')$'s (somewhat worse case) and average of all $\pi(x, x')$.

\subsection{Linear Models}

\textbf{Procedure:} We train a linear model on the \verb|Credit Card Default| dataset \citep{yeh2009comparisons} using Logistic Regression to obtain $w^*$.  We focus on mutable features only that preclude features age and marital status. We take $\Lambda$ to be margin distance $\langle w^*, x \rangle $. For these experiments, at a given $r$, we focus on and use $w^*$ to find the set of all pairs of boundary points $(x, x')$ that lead to a positive flip: $\cbr{ (x,x'): w^\star \cdot x < 0, w^\star \cdot x' \geq 0 }$. This is relatively cheap since by Cauchy-Schwarz, we only need to try all pairs of points whose margin score is $\leq r$, a much smaller set.


For a given set of explanations, we construct and sample from $\consistent$, which is a polytope. Sampling from polytopes is a well-studied problem and we use the state-of-the-art John's Walk \citep{chen2018fast} with mixing time $O(d^2)$. We assume uniform $\Ucal$ over $\Hcal$. Thus, with these samples, we compute the empirical $\max_{(x, x') \in \margin} \hat{\pi}(x, x')$ with $w$'s sampled uniformly from $\consistent$. We repeat this sampling $16$ times for each set of explanations corresponding to a margin-distance percentile.

\textbf{Monotonicity:} We present our results in Figure~\ref{fig: credit_linear}. Qualitatively, we observe a generally smooth decreasing trend with increased distance of explanations from the margin and we observe some non-monotonicity under all three metrics, most prominently under the $\max$ metric. For all three metrics, we see that the trend levels out quickly. This suggests that trying smaller values of $\alpha$ (small amounts of explanation omission) can quickly decrease various measures of boundary certainty and this strategy is effective in this setting.

Quantitatively, we check if the trend is generally monotonic in an experiment that goes as follows. We pick $10$ target boundary certainty values evenly spaced out from the attainable boundary certainties as found on the y-axis. Then, for each target value, we find the minimum percent of explanation points that need to be removed to bring the boundary certainty below the target; this optimal percentage is found simply by sweeping through all (percentage, certainty) pairs we have from left to right. Finally, we obtain the percentage that need to be removed as found by binary search and compute the difference between the percentage found by binary search against the optimal.

Under $k$-medoid explanations for linear model, we summarize the results by looking at the average of the difference and the max difference, which we report as follows. For plots of the $\max_{(x, x') \in \margin} \pi(x, x')$: $r=0.1$, $7$, $35$; $r=0.2$, $11$, $55$; $r=0.3$, $0$, $0$. For plots of average of top $5$ percent of all $\pi(x, x')$: $r=0.1$, $7$, $35$; $r=0.2$, $10$, $50$; $r=0.3$, $0$, $0$. For plots of average of all $\pi(x, x')$: $r=0.1$, $6$, $30$; $r=0.2$, $11$, $55$; $r=0.3$, $0$, $0$. We record the full set of differences in tables in  Appendix~\ref{sec:monot_tables}.


As a synopsis, we observe that the difference is generally small for higher $r$'s and larger for lower $r$'s. The relatively jagged line means that binary search is likely to be quite far off. Here we wish to note that this problem may be alleviated by electing to try the smaller amounts of explanation omission instead of binary search, in the case that we find that the boundary certainties are close at the extremes. Indeed, the closeness would suggest that not much decrease in boundary certainty could be obtained by significantly increasing the percentage of explanation omission.

We also observe the result from varying the allowed extent of manipulation $r$. As expected, the larger the manipulation extent $r$, the higher the $\pi(x, x')$ that may be attainable.

\begin{figure*}
\begin{subfigure}[b]{0.33\textwidth}
    \includegraphics[width=\linewidth]{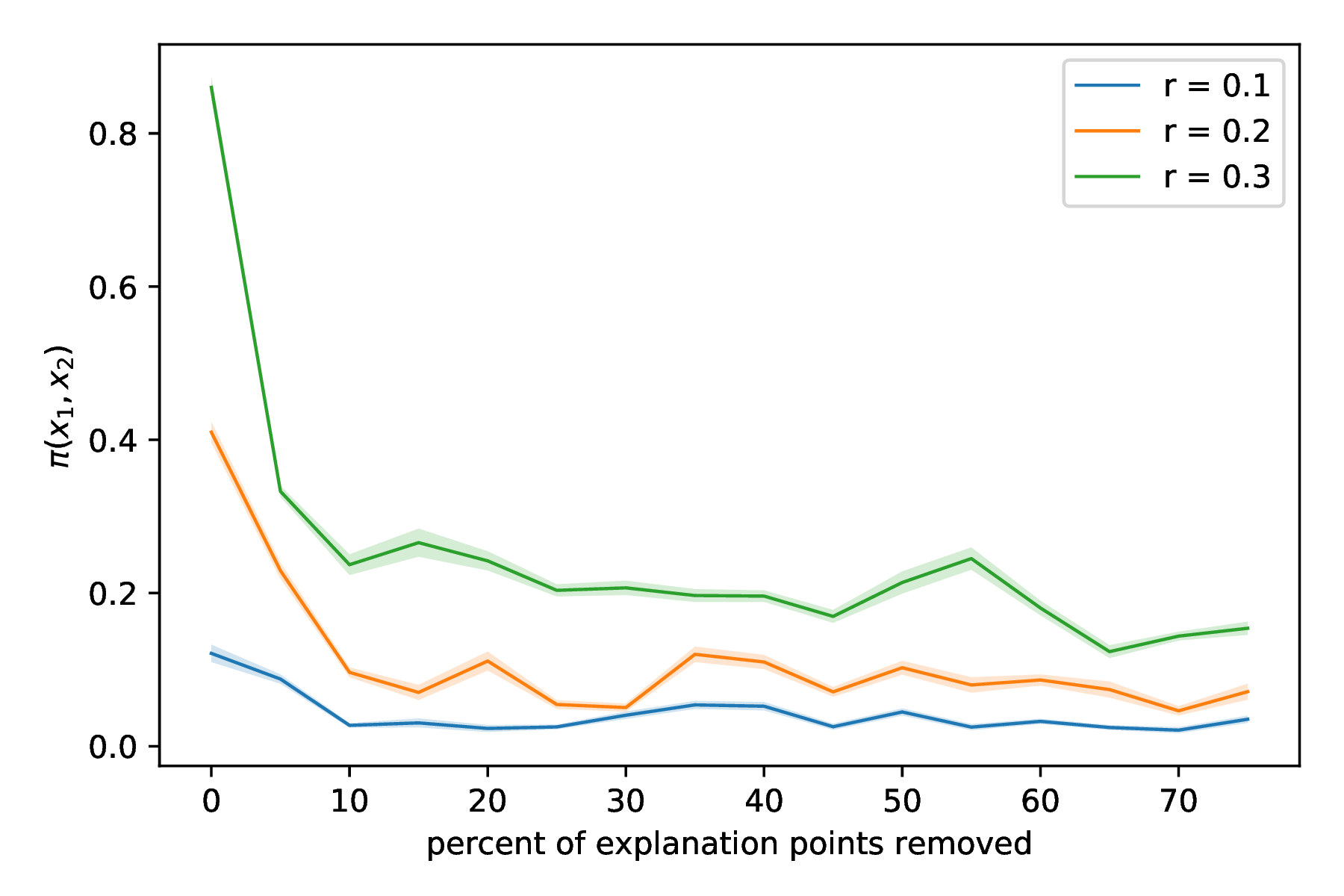}
\end{subfigure}%
\begin{subfigure}[b]{0.33\textwidth}
    \centering
    \includegraphics[width=\linewidth]{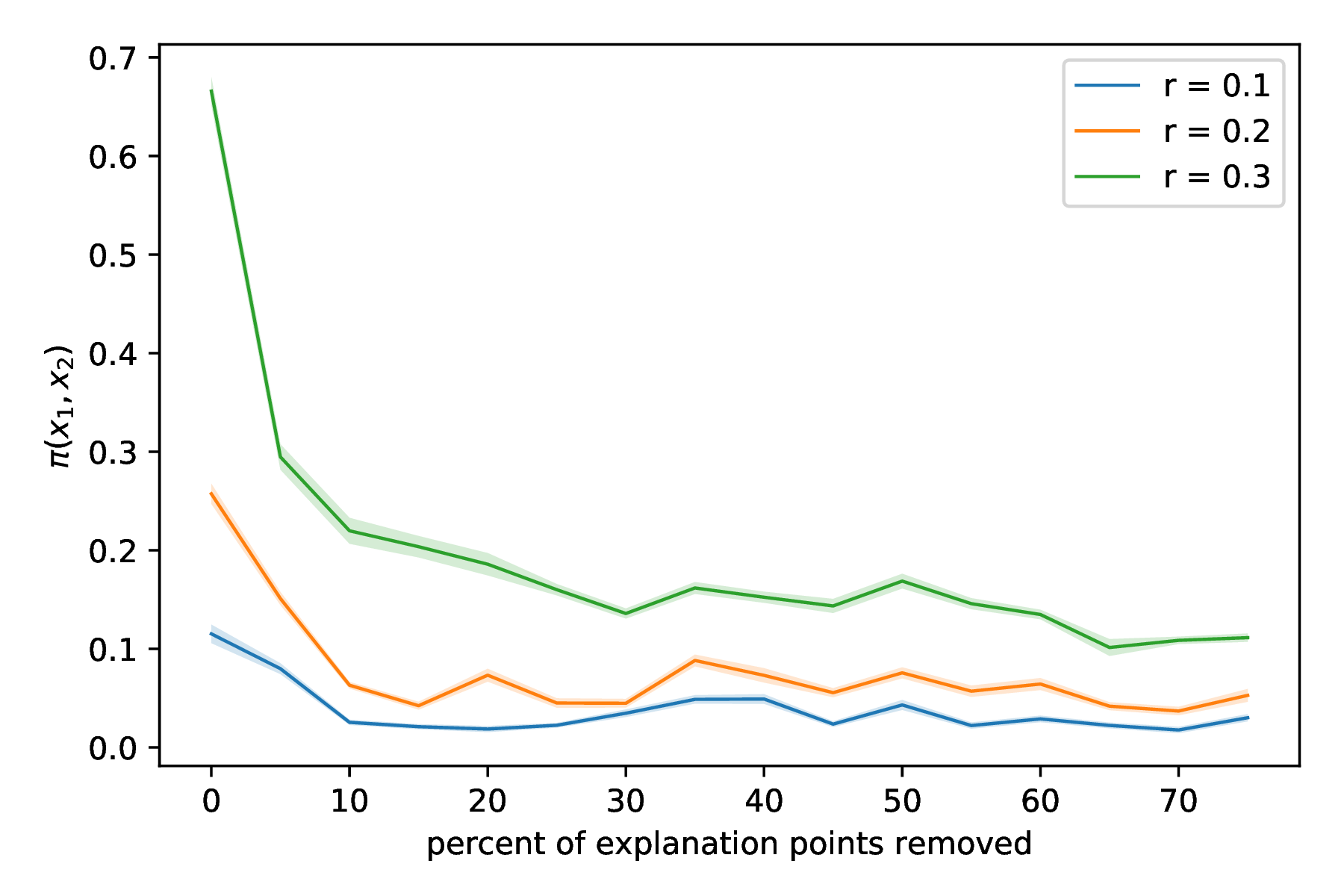}
\end{subfigure}%
\begin{subfigure}[b]{0.33\textwidth}
    \centering
    \includegraphics[width=\linewidth]{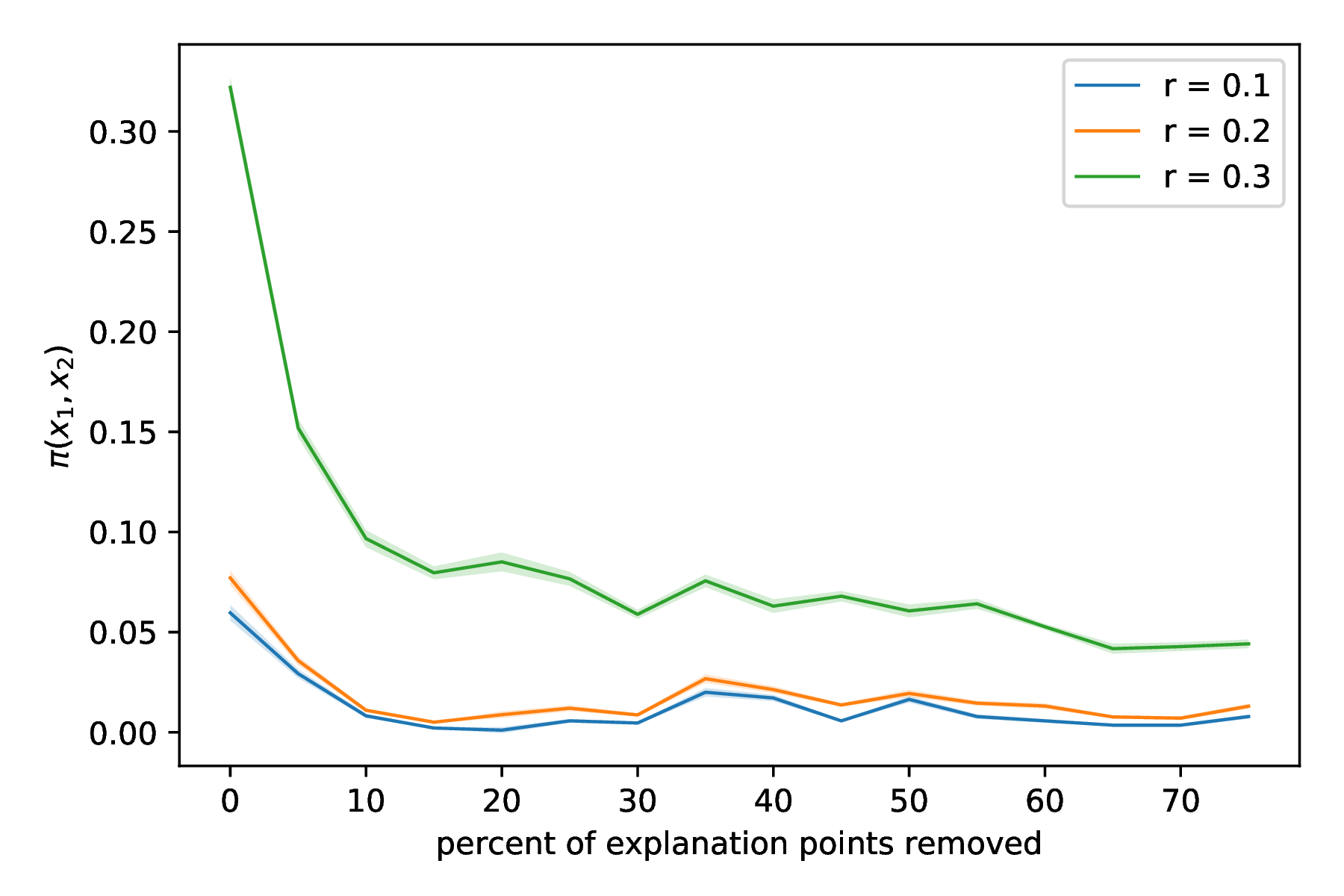}
\end{subfigure}%
\caption{Plots of the $\max_{(x, x') \in \margin} \pi(x, x')$ (left), average of top $5$ percent of all $\pi(x, x')$ (middle) and average of all $\pi(x, x')$ (right) under $k$-medoid explanations for linear models.}
\label{fig: credit_linear}
\end{figure*}

\subsection{Neural Network Models}

\textbf{Procedure:} We train MLPs with one or two hidden layers on the \verb|givemecredit|\footnote{http://www.kaggle.com/c/GiveMeSomeCredit/} dataset. We present the one layer MLP experiment results in the main body and the two layer in the appendix. We experiment with $k$-medoid and MMD-critic~\citep{kim2016examples}, whose results we present in the appendix. To measure of distance from margin, we take $\Lambda_\alpha(x)$ to be the model's confidence of a point: $\Lambda_\alpha(x) = \mathds{1}\{| f^*(x)| \geq \alpha \}$, where $f^*: \Xcal \to [-\frac12,\frac12]$ represents the MLP's predictive probability of class $1$, offset by $-\frac12$. 

To the best of our knowledge, there is no known algorithm that provably sample uniformly from neural network version spaces. Indeed, this is an important problem described by recent works on the ``Rashomon effect'' \citep{d2020underspecification, semenova2019study, marx2020predictive}. We use the procedure in \citep{d2020underspecification} used to probe the version space: randomly initialize the network with different seeds to obtain different models consistent with the explanations. For computational tractability, we sample $100$ MLPs this way with $4$ repetitions per margin-distance percentile.

\textbf{Observations:} Our first observation is that varying just the initialization is not an effective sampling procedure under the \verb|givemecredit| dataset. We find small variation in the MLPs produced. To showcase this, we randomly sample $100$ pairs of MLPs from the $\consistent$ we collected and calculate their label agreement on the boundary points, $\Pr_{h, h' \sim \Ucal(\consistent), x \sim \text{Unif}(\margin)}(h(x) = h'(x))$. The high average consistency of $\consistent$ is charted in green in Figure~\ref{fig: bar_nonrandom_mlp}.

We also compute the three metrics in this setting (Figure \ref{fig: mlp_nonrandom}), which interestingly are very high despite the overall low agreement with respect to $h^*$ -- defined as $\Pr_{h \sim \Ucal(\consistent), x \sim \text{Unif}(\margin)}(h(x) = h^*(x))$ (please see right figure in Figure~\ref{fig: bar_nonrandom_mlp}). This seems to be due to a small fraction of points which most MLPs in $\consistent$ consistently agree with $h^*$ on. The large values of $\max_{(x, x') \in \margin} \pi(x, x')$ in this case suggests the difficulty of preventing worst-case manipulation when the full set of hyperparameters used to train the network is known. 

\begin{figure*}
  \begin{subfigure}[b]{.33\textwidth}
    \centering
    \includegraphics[width=\linewidth]{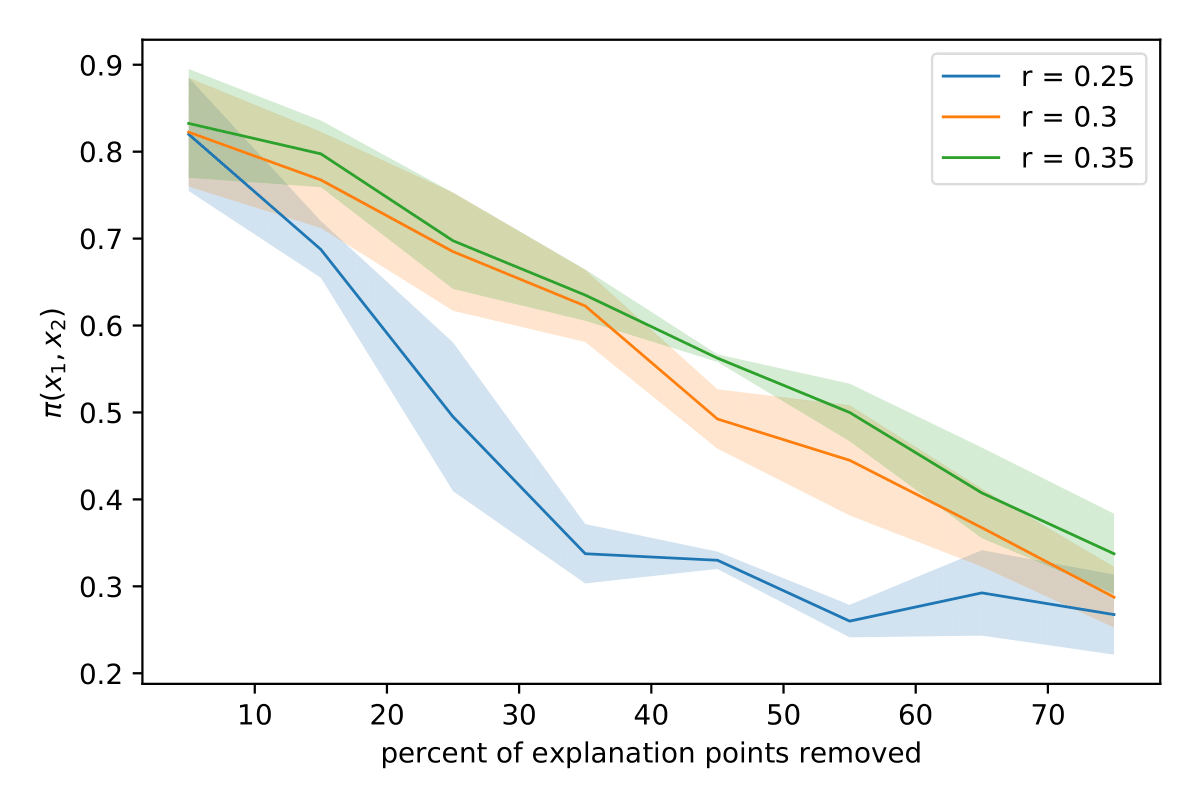}
  \end{subfigure}
  \begin{subfigure}[b]{.33\textwidth}
    \centering
    \includegraphics[width=\linewidth]{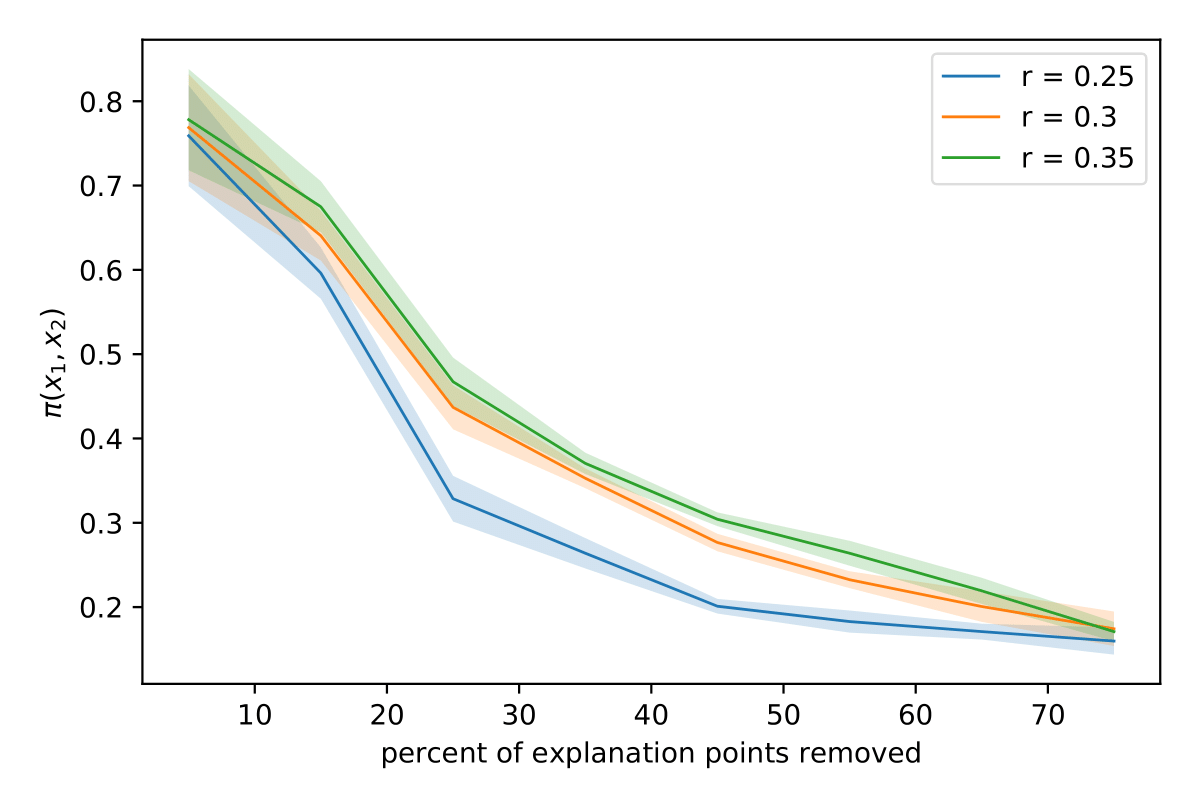}
  \end{subfigure}
  \begin{subfigure}[b]{.33\textwidth}
    \centering
    \includegraphics[width=\linewidth]{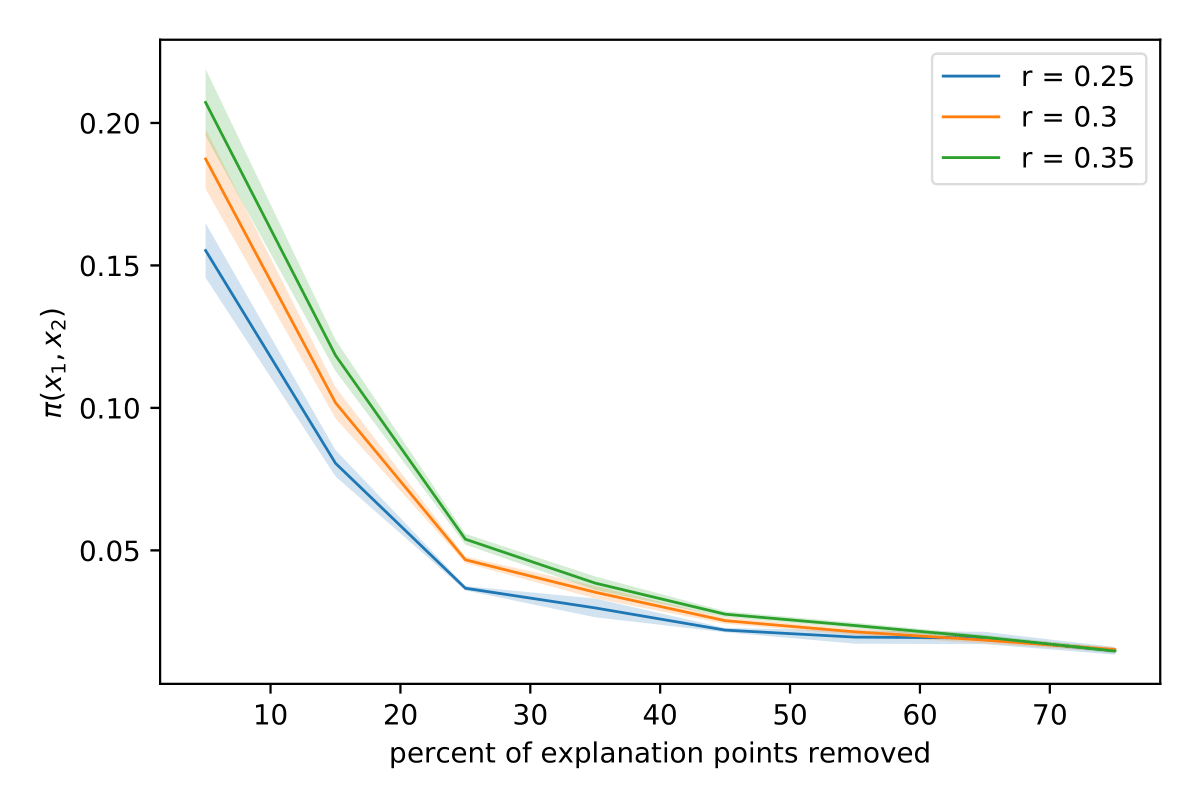}
  \end{subfigure}
  \smallskip
  \begin{subfigure}[b]{.33\textwidth}
    \centering
    \includegraphics[width=\linewidth]{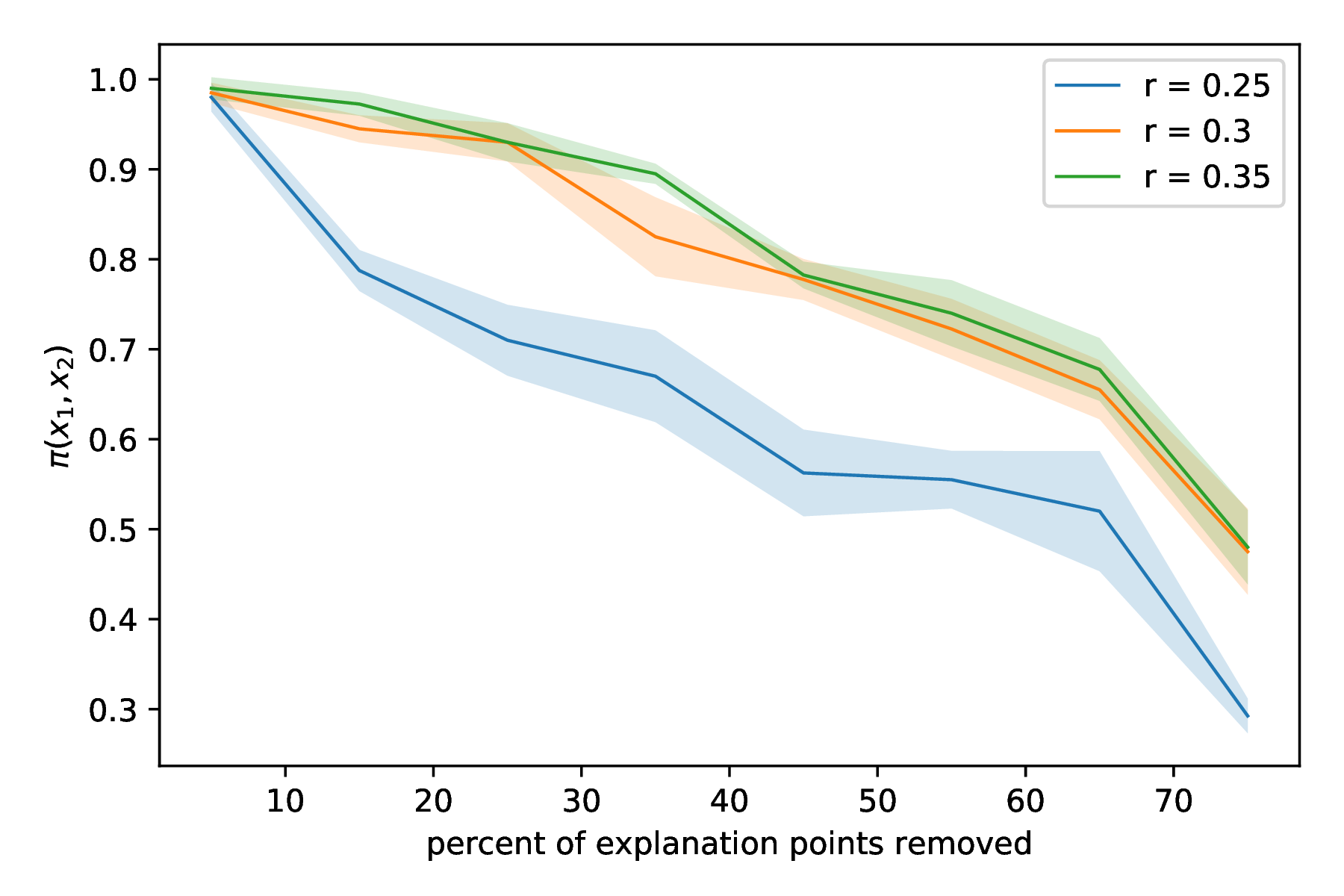}
  \end{subfigure}
  \begin{subfigure}[b]{.33\textwidth}
    \centering
    \includegraphics[width=\linewidth]{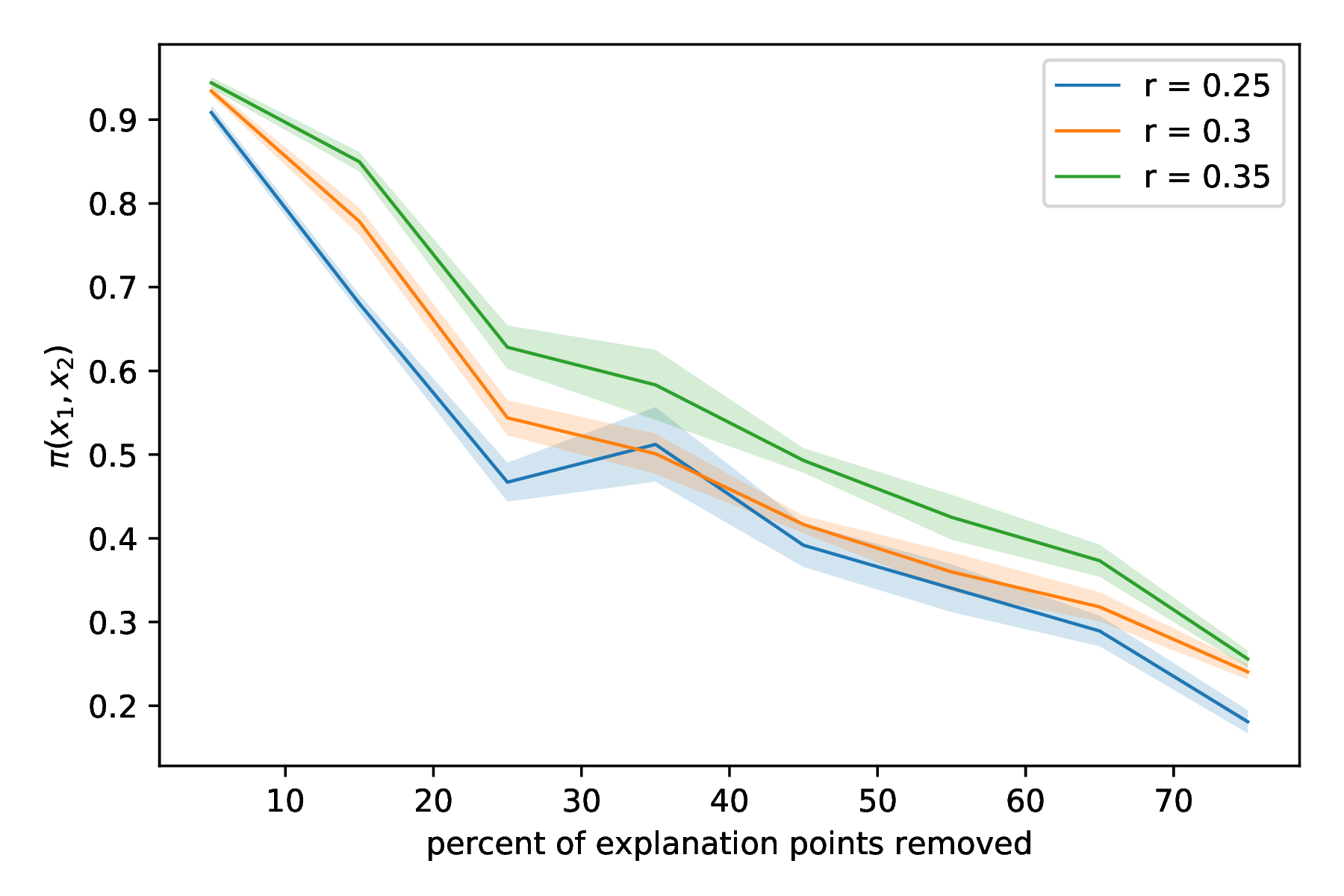}
  \end{subfigure}
  \begin{subfigure}[b]{.33\textwidth}
    \centering
    \includegraphics[width=\linewidth]{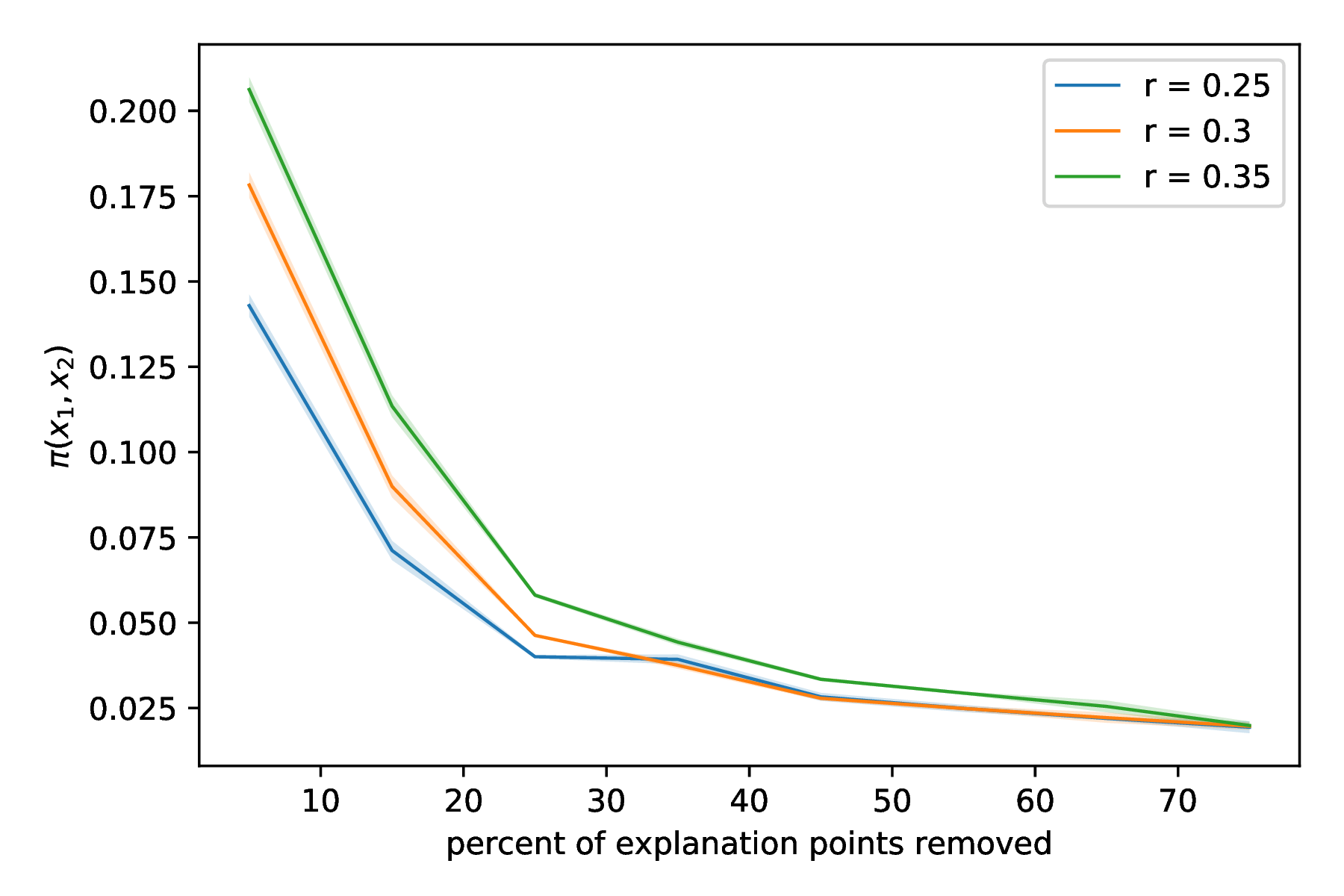}
  \end{subfigure}
  \smallskip
  \begin{subfigure}[b]{.33\textwidth}
    \centering
    \includegraphics[width=\linewidth]{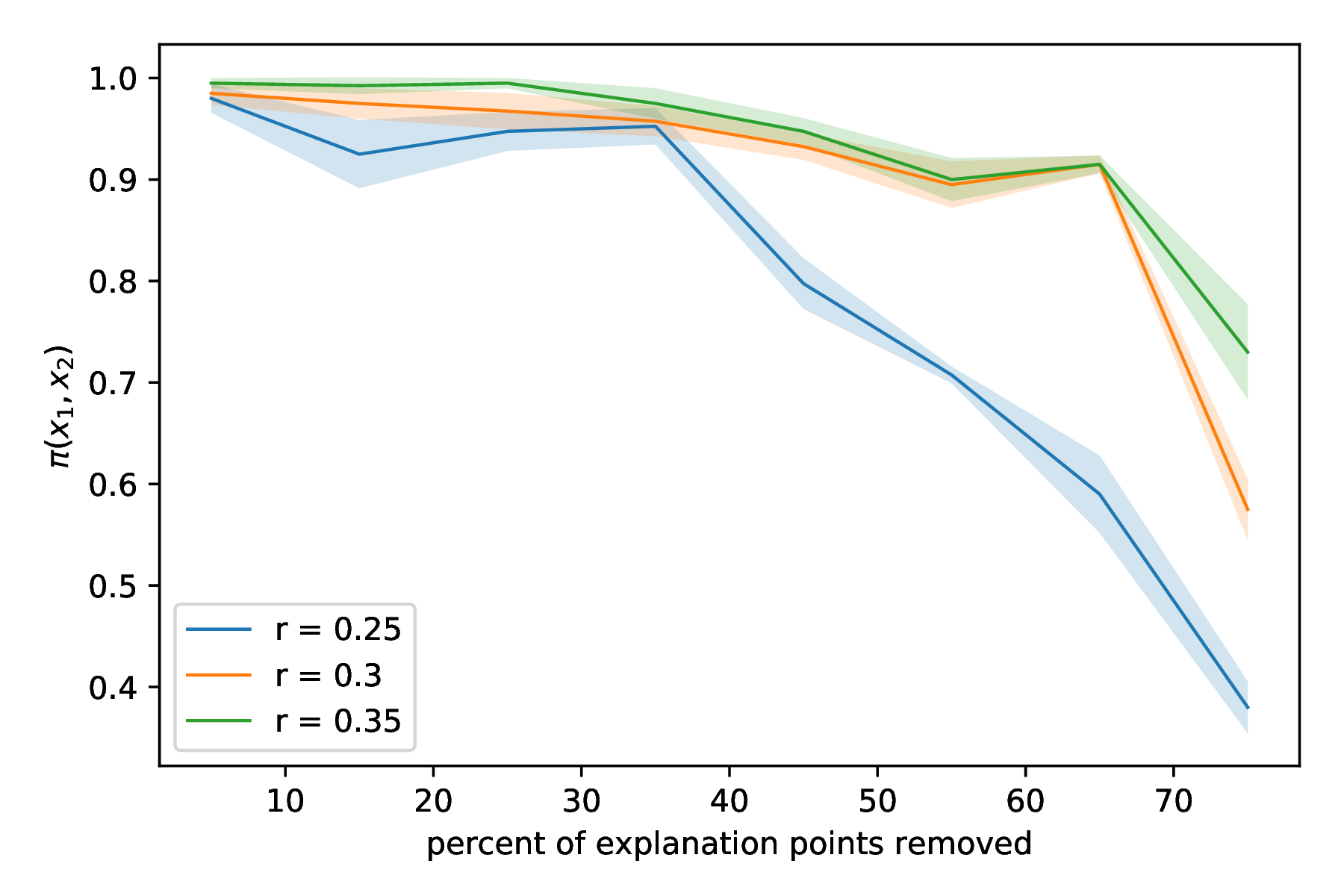}
  \end{subfigure}
  \begin{subfigure}[b]{.33\textwidth}
    \centering
    \includegraphics[width=\linewidth]{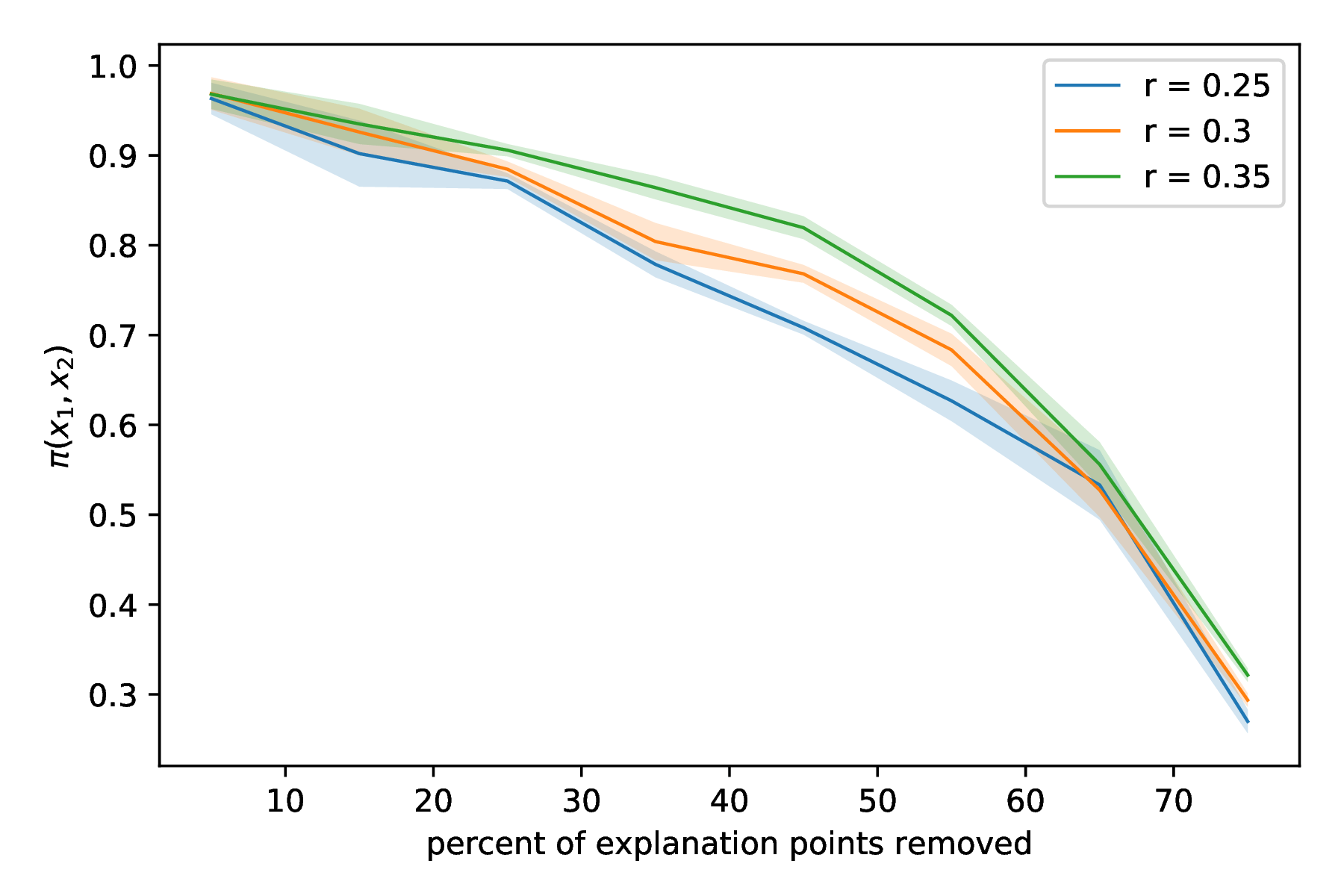}
  \end{subfigure}
  \begin{subfigure}[b]{.33\textwidth}
    \centering
    \includegraphics[width=\linewidth]{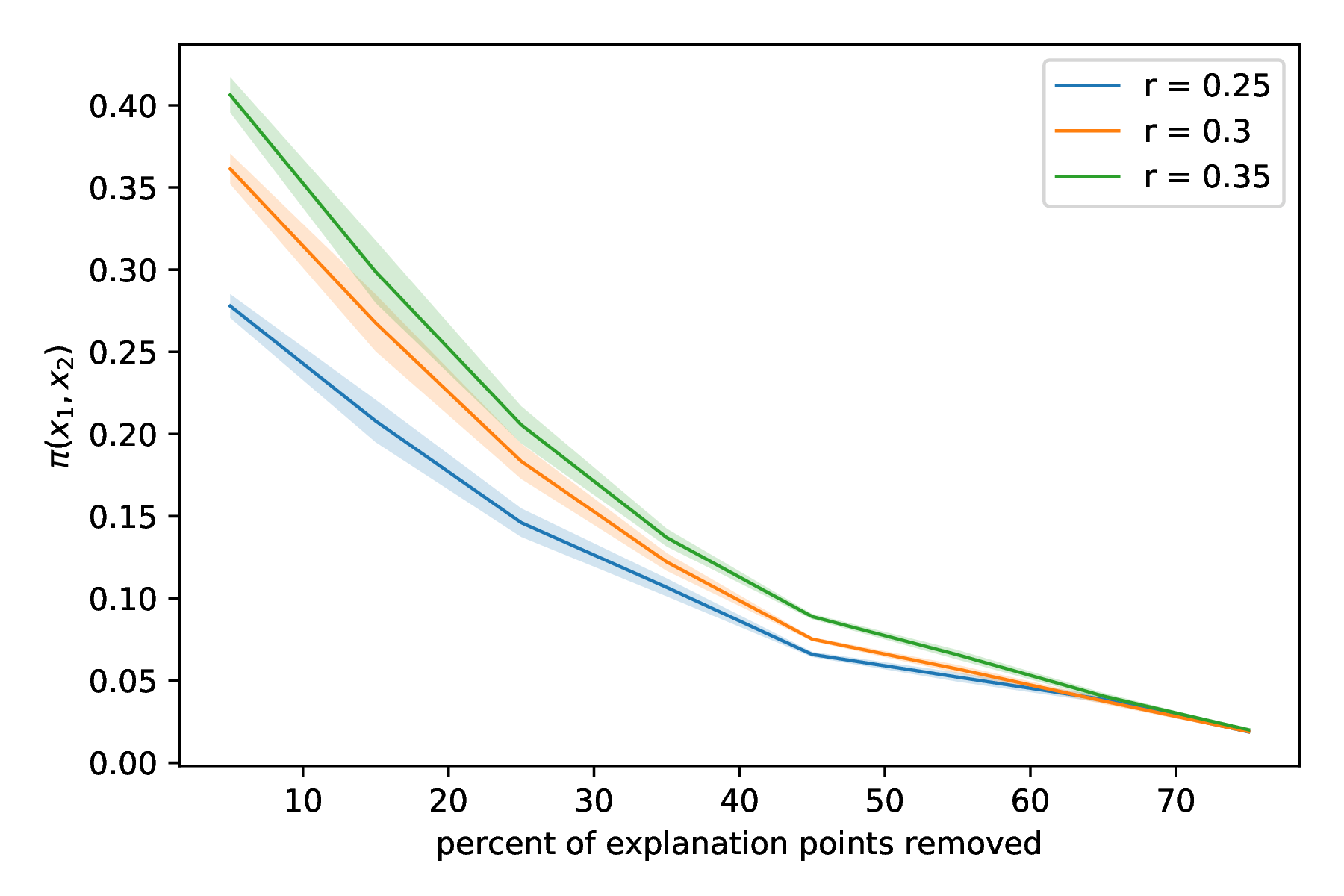}
  \end{subfigure}
  \caption{MLP results: $k$-medoid explanations (top), $k$-medoid explanations + random draws from small balls around the explanations (middle), full $\{x \mid x \in \features, \Lambda_{\alpha}(x) = 1\}$ (bottom). 
  The three metrics are in column: max $\pi(x,x')$ (left), top $5$ percent of all $\pi(x,x')$'s (middle), average $\pi(x,x')$ (right).
  }
  \label{fig: mlp_random}
\end{figure*}

\begin{figure*}
\begin{center}
\begin{subfigure}[b]{0.33\textwidth}
    \centering
    \includegraphics[width=\linewidth]{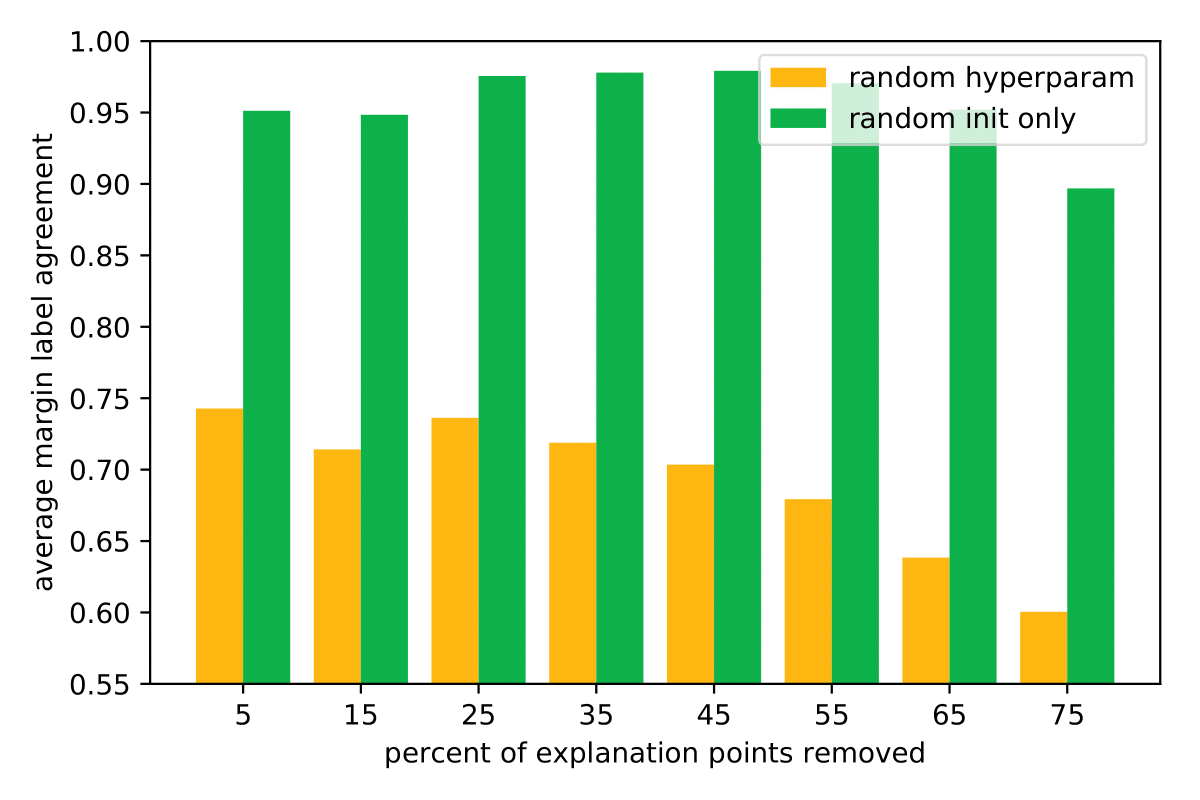}
\end{subfigure}%
\begin{subfigure}[b]{0.33\textwidth}
    \centering
    \includegraphics[width=\linewidth]{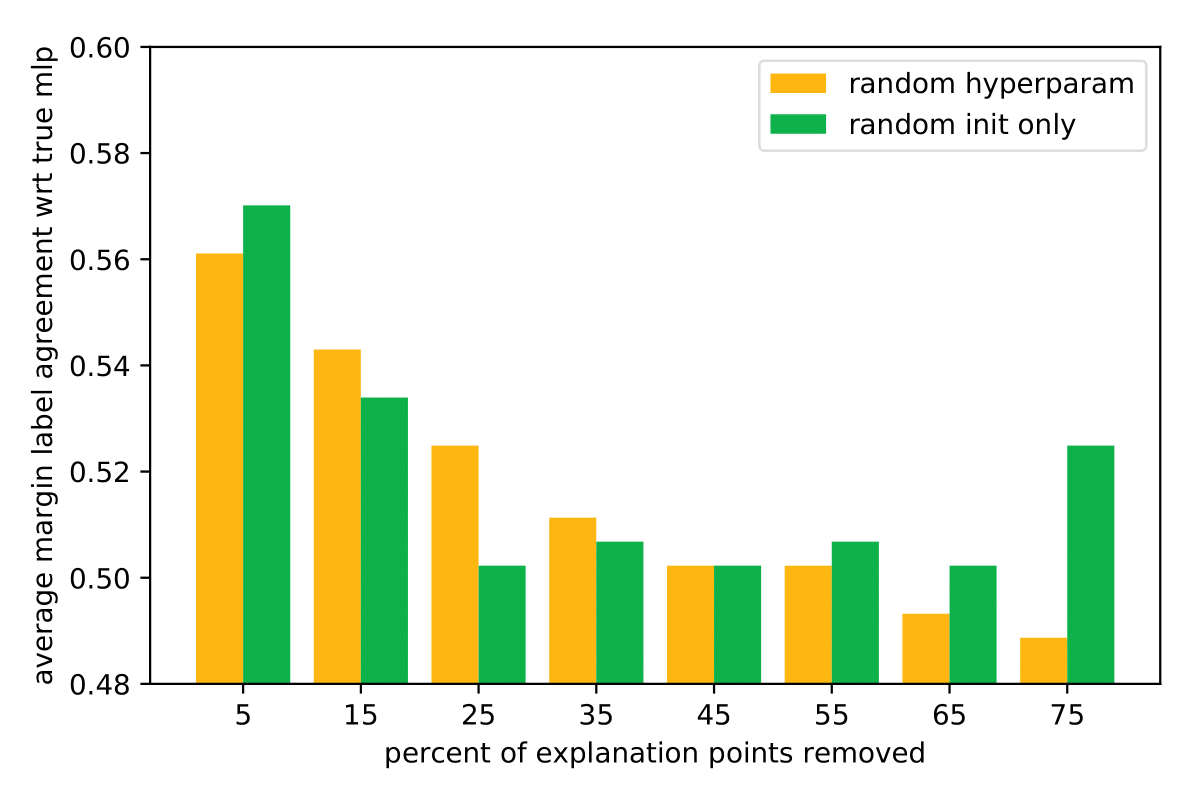}
\end{subfigure}%
\caption{Boundary point label agreement within $\consistent$ (left), and boundary point label agreement of $\consistent$ with respect to $h^*$ (right). This is estimated by sampling $h$ from version space using random initializations of parameters (green) and hyperparameters (yellow), respectively.}
\label{fig: bar_nonrandom_mlp}
\end{center}
\end{figure*}

\begin{figure*}
  \begin{subfigure}[b]{.33\textwidth}
    \centering
    \includegraphics[width=\linewidth]{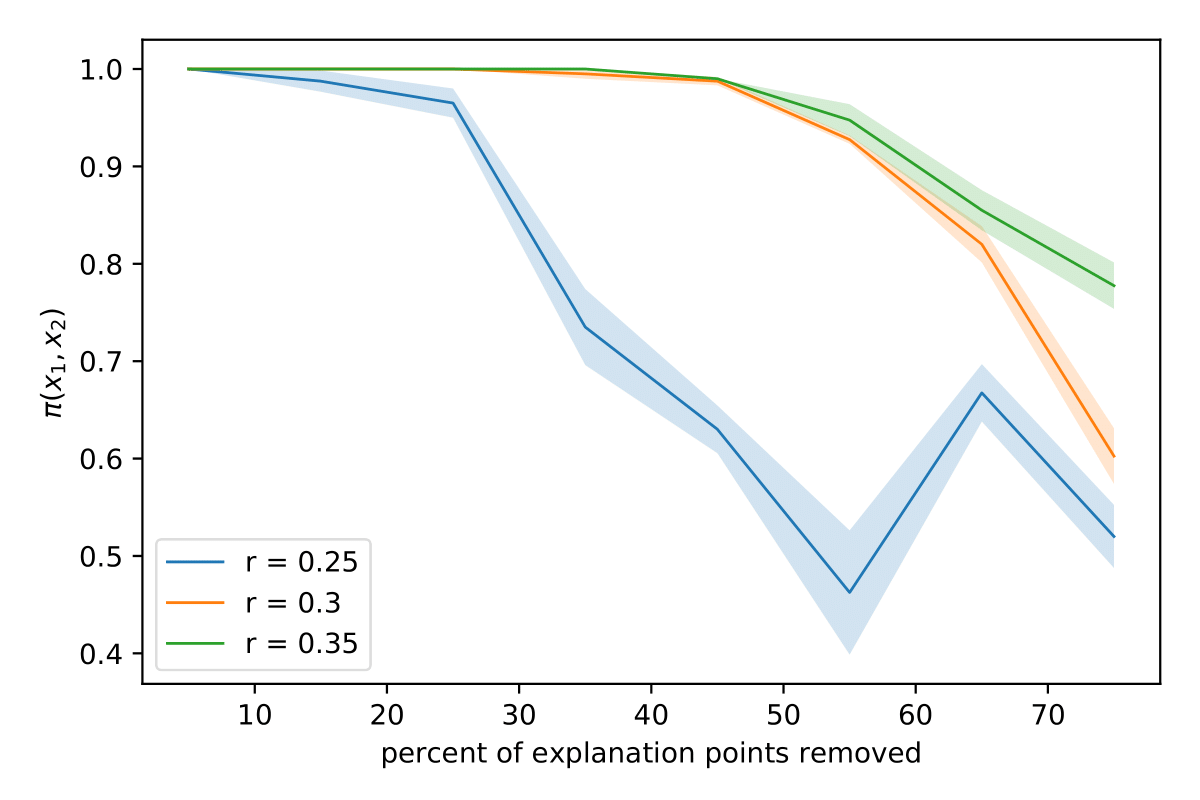}
  \end{subfigure}
  \begin{subfigure}[b]{.33\textwidth}
    \centering
    \includegraphics[width=\linewidth]{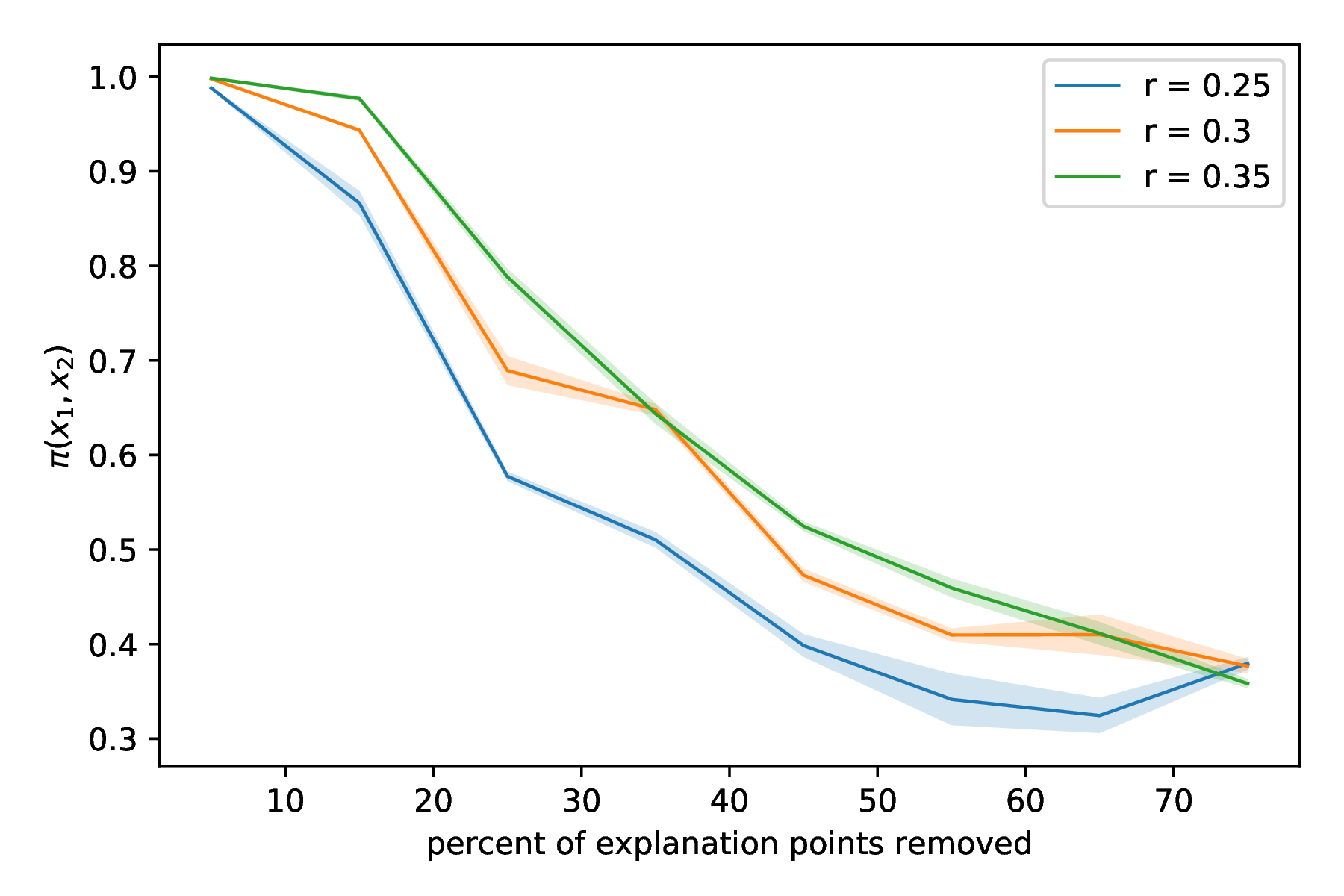}
  \end{subfigure}
  \begin{subfigure}[b]{.33\textwidth}
    \centering
    \includegraphics[width=\linewidth]{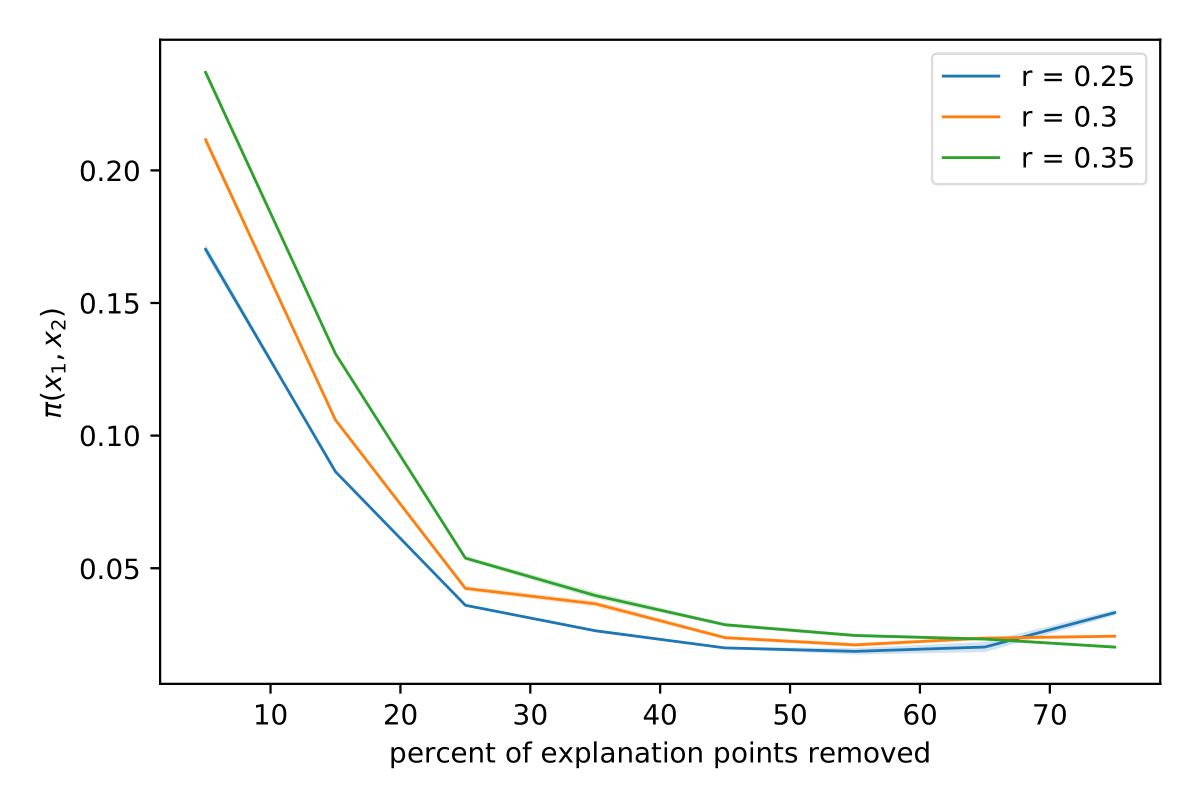}
  \end{subfigure}
  \caption{Plots of the $\max_{(x, x') \in \margin} \pi(x, x')$ (left), average of top $5$ percent of all $\pi(x, x')$'s (middle) and average of all $\pi(x, x')$ (right) for the MLP case with random initialization only under $k$-medoid explanations.}
  \label{fig: mlp_nonrandom}
\end{figure*}

Indeed, as is noted in \citep{jagielski2020high}, it seems generally implausible for attackers to know the \emph{exact} hyperparameters used to train the networks, which has been the assumption in the past model extraction works. And so, from hereon, we experiment with the natural, sampling procedure in the absence of such knowledge, which is just to randomly initialize the network and also the set of hyperparameters ($\ell_2$ regularization constant, learning rate, momentum, batch size). These are randomly sampled from uniform distributions that contain the hyperparameters' true values. Verily, this leads to greater variation (please see the yellow barplots in Figure~\ref{fig: bar_nonrandom_mlp}).

Since neural networks may require higher sample complexity, we also examine data augmentation techniques that one might consider to enhance the explanation set. In addition to 1) just the explanations, we consider 2) explanations plus random draws from Gaussian balls of radius $0.1$ around the explanations 3) the full $\{x \mid x \in \features, \Lambda_{\alpha}(x) = 1\}$, which would correspond to ``perfect'' extrapolation of the feature space based off of $\explanation(\features)$. The plots are given in Figure~\ref{fig: mlp_random}. 

Comparing the effectiveness of the data augmentation, We observe small change in the $\pi$ with mildly augmented data as in 1). However, the full knowledge of the $\{x \mid x \in \features, \Lambda_{\alpha}(x) = 1\}$ results in higher measures of boundary certainty. Indeed, this is to be expected since more labeled data naturally induces higher boundary certainty.

\textbf{Monotonicity:} In terms of the general trend for monotonicity, we again observe that margin-distancing does help to reduce all three metrics. Qualitatively, $\max_{(x, x') \in \margin} \pi(x, x')$ trend is non-monotonic and jagged at places, but smooths out with even a bit of averaging (the latter two metrics). In fact, we see that the average of top $5$ percent of $\pi(x, x')$'s and average of all $\pi(x, x')$ metrics are monotonic. This is instructive in that it suggests that binary search could be used to efficiently search for the appropriate threshold.

Quantitatively, we verify if this trend is generally monotonic as before. We pick $10$ target boundary certainty values evenly spaced out from the attainable boundary certainties as found on the y-axis. For each target value, we find the minimum percent of explanation points that need to be removed to bring the boundary certainty below the target and compare against the percentage found by binary search.

Under $k$-medoid explanations for MLP models, we again summarize the results by looking at the average of the difference and the max difference. Here, due to the much smoother curves (relative to those of the linear models) and the large discrepancy in boundary certainties at the two extremes, we find that under all three $r$'s, binary search is able to match the optimal percentage needed to bring the boundary certainty below the target value.

\subsection{Fair accessibility to explanations}

A notable concern that may arise with margin distancing is that though omission of prototypical explanations is necessary, it may disproportionately affect individuals in regions close to the boundary. We plot the composition of the boundary region in the appendix under linear models logistic and SVM models. We observe that margin-distancing does disparately affect the release of explanations to different groups. Verily, this is another important factor that needs to be taken into account in the explanation release process.

\section{CONCLUSION}
In this paper, we propose margin-distancing as a way of making the tradeoff between transparency and gaming. We identify the source of the tension as boundary points. Our technical contribution is an ``average-case'' analysis of strategic manipulation with partial knowledge of the true model through model explanations. Altogether, this work puts the intersection between strategic ML and explainability on firmer theoretical foundation.

Our paper opens up several novel directions: 1) For what other settings can we prove monotonicity or upper bounds? Especially useful would be upper bounds on whether a certain threshold $\kappa$ is achievable with at least $l$ percent of all explanations. With this, one can avoid futile searches for non-realizable $\kappa$'s. 2) How could we induce small boundary certainty for other types of explanations such as global explanations? 3) How else can we adapt explainability methods to account for gaming?
For this, we believe our proposal of measuring $\explanation(\features)$ quality in terms of the boundary certainty of $\consistent$ may still be helpful as a measure of how much a strategic agent can infer about $h^*$ from $\explanation(\features)$.

\paragraph{Acknowledgments.} We thank the anonymous reviewers for helpful comments that improve the presentation of this paper. TY wishes to thank Ariel Procaccia, Yiling Chen and Chara Podimata for discussions.

\bibliography{refs}

\onecolumn

\appendix
    

\section{Proofs}
\label{sec:proofs}

\subsection{Section 4 Proofs}
\label{sec:proofs-sphere}

Recall that in Section~\ref{sec:sphere}, $\Xcal$ is the origin-centered unit sphere in $\RR^d$, and $\Hcal$ is the set of homogeneous linear classisfiers in $\RR^d$, and $\Ucal$ denotes the uniform distribution over $\Hcal$. 

In the proofs that follow, we will mainly work in terms of polar angles $\phi$ and $\psi$. Recall $\phi = \arcsin \alpha$ is defined to be the maximum angle between any $w \in \consistent$ and $w^*$, and $\psi = 2 \arcsin(\frac{r}{2})$ measures the thickness of the boundary region $\Ncal_r(\Xcal)$. 

Now, we prove a characterization of the boundary region in terms of $\psi$. 

\begin{fact}
$\Ncal_r(\Xcal) = \cbr{x \in \Xcal \mid  \inner{w^*}{x} \in [-\sin\psi, \sin\psi) }$.
\end{fact}

\begin{proof}
Recall our definition that $\Ncal_r(\Xcal) := \{x \in \features \mid \exists x' \in \region(x) \land h^*(x') \neq h^*(x) \}$, where $h^*(x) = \sign(\inner{w^*}{x})$. Thus, it suffices to show that

\[
\inner{w^*}{x} \in [-\sin\psi, \sin\psi)
 \Longleftrightarrow \exists x' \in \region(x) \centerdot \sign(\inner{w^*}{x'}) \neq \sign(\inner{w^*}{x}).
\]

We show the implications in both directions.
\paragraph{($\Rightarrow$):} Suppose we are given $x$ such that $\inner{w^*}{x} \in [-\sin\psi, \sin\psi)$. Then $x$ can be represented as $x = \beta w^* + \sqrt{1 - \beta^2} x_{\perp}$, for some $\beta \in [-\sin\psi, \sin\psi)$, and $x_{\perp}$ is a unit vector perpendicular to $w^*$.  
Observe that $x - x_{\perp} = \beta w^* + (\sqrt{1-\beta^2} - 1)x_{\perp}$, and therefore,
\[
\| x - x_{\perp} \|_2
=
\sqrt{ \beta^2 + (\sqrt{1-\beta^2} - 1)^2 }
=
\sqrt{2 (1 - \sqrt{1-\beta^2})}
.
\]

We now consider two cases of $\beta$:
\begin{enumerate}
    \item If $\beta \in [-\sin\psi, 0)$, 
    we consider $x' = x_{\perp}$. First observe that $x' \in \region(x)$. Indeed,
    \[
    \| x - x' \|_2 =
    \sqrt{2 (1 - \sqrt{1-\beta^2})}
    \leq
    \sqrt{ 2 (1 - \cos\psi) }
    = r.
    \]
Meanwhile, $\sign(\inner{w^*}{x'}) = \sign(0) = 1 \neq -1 = \sign(\beta) = \sign(\inner{w^*}{x})$, which establishes the claim.
\item If $\beta \in [0, \sin\psi)$, we first observe that $\| x - x_\perp \| = \sqrt{2 (1 - \sqrt{1-\beta^2})} < \sqrt{2(1-\cos\psi)} = r$. 
Therefore, there exists a small enough $\gamma > 0$, such that $x' = -\gamma w^* + \sqrt{1-\gamma^2} x_{\perp}$ is close enough to $x_\perp$, and hence lie in $\region(x)$. Now, 
$\sign(\inner{w^*}{x'}) = \sign(-1)= -1 \neq 1 = \sign(\beta) = \sign(\inner{w^*}{x})$, which establishes the claim.
\end{enumerate}

\paragraph{($\Leftarrow$):} Assume toward contradiction that $\inner{w^*}{x} \in [-1, -\sin\psi) \cup [\sin \psi, +1]$. 
Without loss of generality (due to spherical symmetry) suppose that $w^* = (1,0,\ldots,0)$ and $x = (\sin\theta, \cos\theta, 0,\ldots,0)$ with $\theta \in [-\frac\pi2, -\psi) \cup [\psi, \frac\pi2]$.

Consider any $z \in \Xcal \cap \region(x)$. We have:
\[
\sum_{i=1}^d z_i^2 = 1,
\]
\[
(z_1 - \sin\theta)^2 + (z_2 - \cos\theta)^2 + \sum_{i=3}^d z_i^2 \leq r^2,
\]
holding simultaneously. Combining the above two equations, we get
\[
\sin\theta z_1 + \cos\theta z_2 \geq 1 - \frac{r^2}{2} = \cos\psi. 
\]

We now consider two cases of $\theta$:
\begin{enumerate}
    \item $\theta \in [\psi, \frac\pi2]$. In this case, $\cos\theta \leq \cos\psi$. And so, $\sin \theta z_1 \geq \cos\psi - \cos \theta z_2 \geq 0$. Therefore, for all $z \in \region(x)$, $\sin \theta \cdot z_1 \geq 0$ and hence $z_1 \geq 0$. In this case, 
    $\sign(\inner{w^*}{x}) = \sign(\sin \theta) = 1 = \sign(z_1) = \sign(\inner{w^*}{z})$. 
    
    \item $\theta \in [-\frac\pi2, -\psi)$. In this case, $\cos\theta < \cos\psi$. And so, $\sin \theta z_1 \geq \cos\psi - \cos \theta z_2 > 0$. Therefore, for all $z \in \region(x)$, $\sin \theta \cdot z_1 > 0$ and hence $z_1 < 0$. In conclusion,
    $\sign(\inner{w^*}{x}) = \sign(\sin \theta) = -1 = \sign(z_1) = \sign(\inner{w^*}{z})$. 
\end{enumerate}

In either case, $\sign(\inner{w^*}{x}) = \sign(\inner{w^*}{z})$ holds for all $z \in \region(x)$, which contradicts the assumption that $\exists x' \in \region(x) \centerdot \sign(\inner{w^*}{x'}) \neq \sign(\inner{w^*}{x})$. This concludes the proof.
\end{proof}

Recall that we define $\Lambda_\alpha(x) = \mathds{1}(\abs{\inner{w^*}{x}} > \alpha)$  and assume a uniform prior over homogeneous linear model class $\hypothesis$ and that $\features$ is the origin-centered unit sphere in $\RR^d$. With this, we show that the trend of monotonicity exists in this ``nice'' setting and we can also develop direct upper bounds on $\Pi$.

To do this, we first begin by characterizing the version space,

\begin{lemma}[Restatement of Lemma~\ref{lemma: circle_vs}]
\label{lemma:circle_vs_restated}
Fix $\alpha \in [0,1)$. Recall that $\consistent = \{h \in \Hcal \mid h(x') = h^*(x'), \;  \forall x' \in \Ecal_{h^*}(\Xcal, \alpha) \} $ is the version space induced by explanation $\Ecal_{h^*}(\Xcal, \alpha)$. $\consistent$ can be equivalently written as:  
\[
\consistent = \cbr{ h_w \mid \| w \|_2 = 1,  w \cdot w^* \geq \sqrt{ 1 - \alpha^2} }.
\]
\end{lemma}

\begin{proof}
First observe that $w^* \in \Ecal_{h^*}(\Xcal, \alpha)$.
We will show
\[
(\forall x \in \Ecal_{h^*}(\Xcal, \alpha) \centerdot \sign(\inner{w}{x}) = \sign(\inner{w^*}{x}) ) \Longleftrightarrow \inner{w}{w^*} \geq \sqrt{1-\alpha^2}.
\]


We show the implications in both directions:

\paragraph{($\Rightarrow$)} 
First, since $w^* \in \Ecal_{h^*}(\Xcal, \alpha)$, we must have $\inner{w}{w^*} \geq 0$.

Assume towards contradiction that $\inner{w}{w^*} < \sqrt{1-\alpha^2}$, then $w$ can be represented as 
$w = \sqrt{1-\beta^2} w^\star + \beta w_{\perp}$, where $\beta > \alpha$ and $w_{\perp}$ is a unit vector perpendicular to $w^\star$. We now show that there is an $x_0 \in \Ecal_{h^*}(\Xcal, \alpha)$ such that $\sign(\inner{w^*}{x_0}) \neq \sign(\inner{w}{x_0})$, which will reach contradiction.

Choose $\gamma \in (\alpha, \beta)$, and define $x_0 = \gamma w^\star - \sqrt{1-\gamma^2} w_{\perp}$.
It can be readily checked that $\inner{w^*}{x_0} = \gamma > \alpha$, so $x_0 \in \Ecal_{h^*}(\Xcal, \alpha)$. 
Meanwhile, because $\gamma < \beta$,
\[
\inner{w}{x_0} = \sqrt{1-\beta^2} \gamma - \beta \sqrt{1 - \gamma^2}
= 
\sqrt{1-\beta^2} \gamma \del{ 1 - \frac{\beta}{\gamma} \cdot \frac{\sqrt{1-\gamma^2}}{\sqrt{1-\beta^2}} }
< 0,
\]
implying
$\sign(\inner{w}{x_0}) = -1 \neq 1 = \sign(\inner{w^*}{x_0})$.

\paragraph{($\Leftarrow$)}
If $\inner{w}{w^*} \geq \sqrt{1-\alpha^2}$, then $w$ can be represented as 
$w = \sqrt{1-\beta^2} w^\star + \beta w_{\perp}$, where $\beta \leq \alpha$ and $w_{\perp}$ is a unit vector perpendicular to $w^\star$.

Now consider any $x \in \Ecal_{h^*}(\Xcal, \alpha)$; we would like to show that $\sign(\inner{w^*}{x_0}) = \sign(\inner{w}{x_0})$.
First, since $x \in \Ecal_{h^*}(\Xcal, \alpha)$,
$x$ can be represented as $x = \xi w^\star + \sqrt{1-\xi^2} x_{\perp}$, where $\xi \in [-1,-\alpha) \cup (\alpha,+1]$ and $x_{\perp}$ is a unit vector perpendicular to $w^\star$.

Without loss of generality, assume that $\xi \in (\alpha, +1]$; the case of $\xi \in [-1, \alpha)$ is symmetric.
In this case, we have $\sign(\inner{w^\star}{x}) = 1$. Meanwhile, 
\begin{align*}
\inner{w}{x} = & \inner{ \sqrt{1-\beta^2} w^\star + \beta w_{\perp}}{\xi w^\star + \sqrt{1-\xi^2} x_{\perp}} \\
= & \sqrt{1-\beta^2}\xi + \beta \sqrt{1-\xi^2} \inner{w_{\perp}}{x_{\perp}} \\
\geq & \sqrt{1-\beta^2}\xi - \beta \sqrt{1-\xi^2} \\ 
= & \sqrt{1-\beta^2}\xi (1 - \frac{\beta}{\xi} \cdot \frac{\sqrt{1-\xi^2}}{\sqrt{1-\beta^2}} )
> 0,
\end{align*}
where the first inequality is by Cauchy-Schwarz; the second inequality uses the observation that $\beta \leq \alpha < \xi$. The above implies that $\sign(\inner{w}{x}) = 1 = \sign(\inner{w^\star}{x})$.
\end{proof}

It is clear that increasing margin thickness $\psi$ leads to a strictly bigger margin region, and a higher $\max_{(x, x') \in \margin} \pi_{\alpha}(x, x')$. We derive an analytical form of this. 



\begin{theorem}[Restatement of Theorem~\ref{thm:max-pi-sphere}]
$\max_{(x, x') \in \margin} \pi_{\alpha}(x, x')$ can be written as:
\[
\max_{(x, x') \in \margin} \pi_{\alpha}(x, x') = \begin{cases}
\frac{\int_0^{\psi/2} F(\theta) d\theta}{\int_0^\phi F(\theta) d\theta }  & \psi \leq 2\phi \\
1 & \psi > 2\phi,
\end{cases}
\]
where $F(\theta) = (1 - \frac{\cos^2 \phi}{\cos^2 \theta})^{(d-2) / 2}$; therefore, it is strictly increasing for $\psi$ in $[0, 2\phi]$.
\end{theorem}

\begin{proof}
Denote by $F_+(\theta) = (1 - \frac{\cos^2 \phi}{\cos^2 \theta})_+^{(d-2) / 2}$, where $(z)_+ := \max(z, 0)$. 
Note that $F_+(\theta) = 0$ if $\theta \notin [-\phi, \phi]$. 

To show the theorem statement,
note that $ \int_{-\pi}^{\pi} F_+(\theta) d\theta = 2 \int_0^\phi F(\theta) d\theta$; it therefore suffices to show that,
\[
\max_{(x, x') \in \margin} \pi_{\alpha}(x, x') 
=
\frac{ \int_{-\psi/2}^{\psi/2} F_+(\theta) d\theta }{ \int_{-\pi}^{\pi} F_+(\theta) d\theta }
\]


We show the left hand side is both at most and at least the right hand side, respectively. Without loss of generality, let $w^* = (1,0,\ldots,0)$. 
\begin{enumerate}
    \item $\mathrm{LHS} \geq \mathrm{RHS}$: We choose $x' = (\sin\frac{\psi}{2}, \cos\frac{\psi}{2}, 0, \ldots, 0)$, $x = (-\sin\frac{\psi}{2}, \cos\frac{\psi}{2}, 0, \ldots, 0)$. It can be seen that $\| x - x' \|_2 = 2\sin\frac{\psi}{2} = r$, and $\inner{w^*}{x'} > 0$, $\inner{w^*}{x} < 0$, and therefore $(x,x')$ is indeed a boundary pair (i.e. in $\Mcal_r(\Xcal)$). 
    
    
      In addition, for $w = (w_1, w_2)$, denote by $\phi(w) \in (-\pi, \pi]$ its polar angle with respect to $(1,0)$ (so that $\phi((1,0)) = 0$).
    
    \begin{figure}
        \centering
        \includegraphics[width=0.5\textwidth]{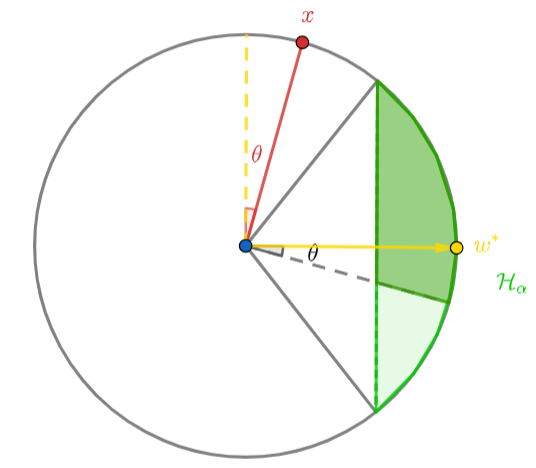}
        \caption{An illustration of $\Pr_{h_w \sim \Ucal(\Hcal_C)}(\inner{w}{x'} \geq 0)$ in the proof of Theorem~\ref{thm:max-pi-sphere}. Suppose $x$ (red dot) has angle $\frac{\pi}{2} - \theta$ with $w^*$, and we project $\Ucal(\Hcal_C)$ to the 2-dimensional plane spanned by $w^*$ and $x$; $\Ucal(\Hcal_C)$ (after projection) is supported on the green circle segment (the union of the dark and light green regions), whereas the subset $\cbr{h_w \in \Hcal_C: \inner{w}{x'} \geq 0}$ corresponds to the dark green region.}
        \label{fig:proof-thm-4}
    \end{figure}
    
    In this case, by Claim~\ref{claim:phi-angle} given below (see also Figure~\ref{fig:proof-thm-4} for an illustration), we have:
    \begin{align*}
    \Pr_{h_w \sim \Ucal(\Hcal_C)}(\inner{w}{x'} \geq 0)
    = & 
    \Pr_{h_w \sim \Ucal(\Hcal_C)}( \phi(w) \in [-\psi/2, \pi/2] ) \\
    = & 
    \frac{ \int_{-\psi/2}^{\pi/2} F_+(\theta) d\theta }{ \int_{-\pi}^{\pi} F_+(\theta) d\theta },
    \end{align*}
    and
    \begin{align*}
    \Pr_{h_w \sim \Ucal(\Hcal_C)}(\inner{w}{x} \geq 0)
    = & 
    \Pr_{h_w \sim \Ucal(\Hcal_C)}( \phi(w) \in [\psi/2, \pi/2] ) \\
    = & 
    \frac{ \int_{\psi/2}^{\pi/2} F_+(\theta) d\theta }{ \int_{-\pi}^{\pi} F_+(\theta) d\theta },
    \end{align*}
    and therefore,
    \[
    \max_{(x, x') \in \margin} \pi_{\alpha}(x, x') 
    \geq
    \Pr_{h_w \sim \Ucal(\Hcal_C)}(\inner{w}{x'} \geq 0)
    -
    \Pr_{h_w \sim \Ucal(\Hcal_C)}(\inner{w}{x} \geq 0)
    = 
    \frac{ \int_{-\psi/2}^{\psi/2} F_+(\theta) d\theta }{ \int_{-\pi}^{\pi} F_+(\theta) d\theta }
    \]
    
    
    \item $\mathrm{LHS} \leq \mathrm{RHS}$: First, for every $z \in \RR^d$, denote by $\theta(w^*,z) = \arccos(\frac{\inner{w^*}{z}}{\|w^*\| \| z\|}) \in [0, \pi]$ the angle between $z$ and $w^*$. 
    \[
    \Pr_{h_w \sim \Ucal(\Hcal_C)}(\inner{w}{z} \geq 0)
    = 
    \frac{ \int_{\theta(w,z) - \frac\pi 2}^{\frac \pi 2} F_+(\theta) d\theta }{ \int_{-\pi}^{\pi} F_+(\theta) d\theta }
    \]

    To see this, without loss of generality, let $z = (z_1, z_2, 0, \ldots, 0)$. Then, by Claim~\ref{claim:phi-angle} (given below), we have
    \[
    \Pr_{h_w \sim \Ucal(\Hcal_C)}(\inner{w}{z} \geq 0)
    = 
    \Pr_{h_w \sim \Ucal(\Hcal_C)}\del{ \phi((z_1, z_2)) \in [\theta(w,z) - \frac\pi2, \frac \pi 2] }
    = 
    \frac{ \int_{\theta(w,z) - \frac\pi 2}^{\frac \pi 2} F_+(\theta) d\theta }{ \int_{-\pi}^{\pi} F_+(\theta) d\theta }.
    \]
    
    Therefore, for every $(x,x') \in \Mcal_r(\Xcal)$, 
    \begin{align*}
    & \Pr_{h_w \sim \Ucal(\Hcal_C)}(\inner{w}{x'} \geq 0)
    -
    \Pr_{h_w \sim \Ucal(\Hcal_C)}(\inner{w}{x} \geq 0) \\
    = & 
    \frac{ \int_{\theta(w,x') - \frac\pi 2}^{ \theta(w,x) - \frac\pi 2} F_+(\theta) d\theta }{ \int_{-\pi}^{\pi} F_+(\theta) d\theta } \\
    \leq & 
    \frac{ \max \cbr{ \int_{a}^{b} F_+(\theta) d\theta :b-a \leq \psi } }{ \int_{-\pi}^{\pi} F_+(\theta) d\theta },
    \end{align*}
    where the inequality follows by observing $\theta(w,x) - \theta(w,x') \leq \theta(x,x') \leq \psi$, which follows from $2\sin\frac{\theta(x,x')}{2} = \| x - x' \| \leq r = 2\sin\frac{\psi}{2}$ and that $\psi/2$ is acute by definition, which means that $\theta(x,x') / 2 \leq \psi/2$ and $\psi / 2$ are both acute. 
    It suffices to show that for every $a,b$ such that $b - a \leq \psi$, 
    \begin{equation}
    \int_{a}^{b} F_+(\theta) d\theta
    \leq 
    \int_{-\psi/2}^{\psi/2} F_+(\theta) d\theta.
    \label{eqn:f-plus-integral}
    \end{equation}
    As $F_+(\theta) \geq 0$ for any $\theta \in \RR$, the max must be achieved at $b - a = \psi$ and so it suffices to show $\forall c$,
    \[ 
    \int_{c - \psi/2}^{c + \psi/2} F_+(\theta) d\theta
    \leq 
    \int_{-\psi/2}^{\psi/2} F_+(\theta) d\theta.
    \]
    Let $F(c) = \int_{c - \psi/2}^{c + \psi/2} F_+(\theta) d\theta$; it can be seen that $F'(a) = F_+(c + \psi/2) - F_+(c - \psi/2)$. Therefore,
    \[
    F'(c) 
    \begin{cases}
    \geq 0 & c \leq -\psi/2 \\
    \geq 0 & -\psi/2 \leq c \leq 0 \\ 
    \leq 0 & 0 \leq c \leq \psi/2 \\
    \leq 0 & c \geq \psi/2,
    \end{cases}
    \]
    and hence $\max_{c \in \RR} F(c) = F(0) = \int_{-\psi/2}^{+\psi/2} F_+(\theta) d\theta$, which concludes the proof of Equation~\eqref{eqn:f-plus-integral}, and concludes that $\mathrm{LHS} \leq \mathrm{RHS}$. 
    \qedhere
    
\end{enumerate}


\end{proof}

\begin{fact}
\label{fact:2d-proj}
The probability density function of the uniform distribution over unit sphere projected onto the first two dimensions is 
\[ 
p(w_1, w_2) = \frac{d-2}{2\pi} (1-w_1^2-w_2^2)^{\frac{d-4}{2}}. 
\]
\end{fact}

\begin{claim}
\label{claim:phi-angle}
In the notation of the proof of Theorem~\ref{thm:max-pi-sphere} above, for every $a < b$ such that $[a,b] \subset (-\pi, \pi]$, 
\[
\Pr_{h_w \sim \Ucal(\Hcal_C)}\del{ \phi( (w_1, w_2) ) \in [a,b] }
=
\frac{ \int_{a}^{b} F_+(\theta) d\theta }{ \int_{-\pi}^{\pi} F_+(\theta) d\theta }
\]
\end{claim}
\begin{proof}
Recall Lemma~\ref{lemma: circle_vs} that characterizes $\Hcal_C$ (see also Figure~\ref{fig:version_space}), we have:
\[
\Pr_{h_w \sim \Ucal(\Hcal_C)}\del{ \phi( (w_1, w_2) ) \in [a,b] }
= 
\frac{ \Pr_{h_w \sim \Ucal} \del{ w_1 \geq \sqrt{1-\alpha^2}, \phi( (w_1, w_2) ) \in [a,b] } }{ \Pr_{h_w \sim \Ucal} \del{ w_1 \geq \sqrt{1-\alpha^2} } }
\]
From Fact~\ref{fact:2d-proj} above, we can express the numerator and the denominator in integral form. 
For the denominator, by changing of variables to the polar coordinates,
\begin{align*}
& \Pr_{h_w \sim \Ucal} \del{ w_1 \geq \sqrt{1-\alpha^2} } \\
= & 
\int_{-\phi}^{\phi} \del{ \int_{\frac{\cos\phi}{\cos\theta}}^1 \frac{d-2}{2\pi} (1-r^2)^{\frac{d-4}{2}} r dr } d\theta \\
= &
\frac{1}{2\pi} \int_{-\phi}^{\phi} \del{ 1-\frac{\cos^2\phi}{\cos^2\theta}}^{\frac{d-2}{2}}  d\theta \\
= & \frac{1}{2\pi} \int_{-\pi}^{\pi} F_+(\theta) d\theta. 
\end{align*}


For the numerator,
\begin{align*}
& \Pr_{h_w \sim \Ucal} \del{ w_1 \geq \sqrt{1-\alpha^2}, \phi( (w_1, w_2) ) \in [a,b] } \\
= &
\int_{\max(-\phi, a)}^{\min(\phi, b)} \del{ \int_{\frac{\cos\phi}{\cos\theta}}^1 \frac{d-2}{2\pi} (1-r^2)^{\frac{d-4}{2}} r dr } d\theta \\
= & 
\frac{1}{2\pi} \int_{\max(-\phi, a)}^{\min(\phi, b)} \del{ 1-\frac{\cos^2\phi}{\cos^2\theta}}^{\frac{d-2}{2}}  d\theta \\
= & 
\frac{1}{2\pi} \int_{a}^{b} F_+(\theta) d\theta. 
\end{align*}
The lemma follows by combining two equalities above.
\end{proof}

\begin{theorem}[Restatement of Theorem~\ref{thm:monotonicity}]
\label{thm:monotonicity-restated}
$\Pi(\alpha)$ is decreasing in $\alpha$, for $\alpha \in [0, 1)$, and is strictly decreasing in $[\sin(\psi/2), 1)$.
\end{theorem}
\begin{proof}
Consider $\Pi(\alpha)$ for $\alpha = \sin\phi \in [\sin(\psi/2), 1]$, which, from the proof of Theorem~\ref{thm:max-pi-sphere}, has the following form:
\begin{align*}
\Pi(\alpha)
= & \Pr_{h_w \sim \Ucal(\Hcal_C)}\del{ \phi(w) \in [-\psi/2, \psi/2] } \\
= &
\frac{ \Pr_{h_w \sim \Ucal} \del{ w_1 \geq \sqrt{1-\alpha^2}, \phi( (w_1, w_2) ) \in [-\psi/2, \psi/2] } }{ \Pr_{h_w \sim \Ucal} \del{ w_1 \geq \sqrt{1-\alpha^2} } } \\
= &
\frac{ \int_{\sqrt{1-\alpha^2}}^1 \del{ \int_0^{w_1 \tan\psi} p(w_1, w_2) d w_2 } d w_1 }{ \int_{\sqrt{1-\alpha^2}}^1 \del{ \int_0^{\sqrt{1-w_1^2}} p(w_1, w_2) d w_2 } d w_1 },
\end{align*}
where $p(w_1, w_2) = \frac{d-2}{2\pi} (1-w_1^2-w_2^2)^{(d-4)/2}$ is the pdf of $(w_1, w_2)$ when $h_w \sim \Ucal$ (Fact~\ref{fact:2d-proj}).

 Consider $f(w_1) = \int_0^{w_1 \tan\psi} p(w_1, w_2) d w_2$, and $g(w_1) = \int_0^{\sqrt{1-w_1^2}} p(w_1, w_2) d w_2$, and 
 $F(t) = \frac{\int_t^1 f(w_1) dw_1}{\int_t^1 g(w_1) dw_1}$. 
 ;with this,  
 $\Pi(\alpha) = F(\sqrt{1-\alpha^2})$. It suffices to show that $F(t)$ is monotonically increasing, i.e. $F'(t) \geq 0$ for all $t$.

 To show this, first observe that 
 $\frac{f(w_1)}{g(w_1)}$ is monotonically increasing: indeed,
 \[
 \frac{f(w_1)}{g(w_1)}
 =
 \frac{ \int_0^{\frac{w_1\tan\psi}{\sqrt{1-w_1^2}}} (1 - v^2)^{\frac{d-4}{2}} dv }{ \int_0^1 (1 - v^2)^{\frac{d-4}{2}} dv  },
 \]
 which is increasing in $w_1$. As a consequence, 
 \begin{equation}
 \int_{t}^1 f(w_1) d w_1
 = 
 \int_{t}^1 g(w_1) \cdot (\frac{f(w_1)}{g(w_1)}) d w_1
 \geq 
 \frac{f(t)}{g(t)} \cdot \int_{t}^1 g(w_1) d w_1
 \label{eqn:integral-compare-f-g}
 \end{equation}
 
 Therefore,
 \[
 F'(t)
 =
 \frac{-f(t) \int_t^1 g(w_1) dw_1 + g(t) \int_t^1 f(w_1) dw_1 }{ (\int_t^1 g(w_1) dw_1)^2 }
 \geq 0,
 \]
 where the last inequality is from Equation~\eqref{eqn:integral-compare-f-g}.
\end{proof}

Below, we derive bounds on $\Pi(\alpha)$ given specific assumptions on $\phi$ and $\psi$.


\begin{theorem}[Refined version of Theorem~\ref{thm:pi-bounds}]
\label{thm:pi-bounds-refined}
We have the following:
\begin{enumerate}
    \item If $\cos\phi \leq  \frac{1}{2d^{1/4}} $, then
    $
    \Pi(\alpha) \leq 6 \cdot \del{
    \psi (1 + d^{\frac12} \cos\phi) } $. 
    \label{item:ratio-ub}
    \item For any $c_1, c_2 > 0$, there exists $c_3 > 0$ such that the following holds:
    given any $\phi \in \intco{c_1, \frac\pi2}$, and 
    \begin{equation}
        \psi \geq c_3 \max\del{ \cos\phi, \frac{1}{d^{\frac12} \cos\phi} \sqrt{ \ln \frac{4}{c_2} + \ln\del{ 1+\frac{1}{d^{\frac12} \cos\phi}} } },
        \label{eqn:psi-lb}
    \end{equation} 
    then $\Pi(\alpha) \geq 1- c_2$.

    \label{item:ratio-lb}
\end{enumerate}
\end{theorem}

Before presenting the proof of Theorem~\ref{thm:pi-bounds-refined}, we first show how it concludes the proof of Theorem~\ref{thm:pi-bounds}.

\begin{proof}[Proof of Theorem~\ref{thm:pi-bounds}]
We show the two items respectively.
\begin{enumerate}

    \item Recall that $\alpha = \sin\phi$.
    If $\alpha \geq 1 - \frac{1}{8d}$, then  $\cos^2\phi = 1 - \alpha^2 \leq \frac{1}{4d}$, implying that $\cos\phi \leq \frac{1}{2\sqrt{d}}$. As $\frac{1}{2\sqrt{d}} \leq \frac{1}{2d^{1/4}}$, the conditions of item~\ref{item:ratio-ub} of Theorem~\ref{thm:pi-bounds-refined}  is satisfied. As a result,
    \[
    \Pi(\alpha) \leq 6 \cdot \del{ 
    \psi (1 + d^{\frac12} \cos\phi) }
    \leq 9 \psi. 
    \]
    \item
    Let $C_1 \in (0,1)$. Choose $\phi' := \arccos(\frac{1}{d^{1/4}})$. 
    Note that $\phi' \geq \phi$, since $1 - \cos^2\phi = \alpha^2 = (1 - \frac{1}{\sqrt{d}})^2 \leq 1 - \frac{1}{\sqrt{d}} = 1 - \cos^2\phi'$. Denote by $\alpha := \sin \phi$ and $\alpha' := \sin \phi'$; we have $\alpha' \geq \alpha$. 
    
    In addition, as $\phi' = \arccos(\frac{1}{d^{1/4}})$, there exists some numerical constant $c_1 > 0$ such that $\phi' \geq c_1$. Now, by item 2 of Theorem~\ref{thm:pi-bounds-refined}, there exists some $c_3 > 0$, such that when $\psi \geq \frac{ c_3  \sqrt{ \ln\frac{8}{C_1} }}{d^{1/4}} \geq c_3 \max\del{ \frac{1}{d^{1/4}}, \frac{ \sqrt{ \ln\frac{4}{C_1} + \ln\del{1 + \frac{1}{d^{1/4}}} }}{d^{1/4}} }$, $\Pi(\alpha') \geq 1 - C_1$.
    Now, as $\Pi(\cdot)$ is monotonically decreasing in $\alpha$, $\Pi(\alpha) \geq \Pi(\alpha') \geq 1 - C_1$.  
    Therefore, the theorem statement holds with $C_2 = c_3 \sqrt{ \ln\frac{8}{C_1} }$. 
    \qedhere
\end{enumerate}
\end{proof}

We now present the proof of Theorem~\ref{thm:pi-bounds-refined}.

\begin{proof}

Recall that
\[
\Pi(\alpha) 
= 
\begin{cases}
\frac{ \int_0^{\psi/2} F(\theta) d\theta }{ \int_0^{\phi} F(\theta) d\theta }, & \arcsin\alpha = \phi \geq  \psi/2 \\
1, & \arcsin\alpha = \phi < \psi/2. 
\end{cases}
\]

\begin{enumerate}
\item First we note that $\cos\phi \leq \frac1{2d^{1/4}}$ implies that $\phi \geq \frac\pi 3$.

If $\phi \leq \psi/2$, then $\psi \geq \frac{2}{3} \pi$. Therefore,
$\Pi(\alpha) = 1 \leq 6 \psi \leq 6 \cdot \del{
    \psi (1 + d^{\frac12} \cos\phi) }$ holds. 

For the rest of the proof, we focus on the case of $\phi > \psi/2$.
In this case, $\Pi(\alpha)$ equals the integral ratio $\frac{ \int_0^{\psi/2} F(\theta) d\theta }{ \int_0^{\phi} F(\theta) d\theta }$.
With foresight, define $\theta' = \min\del{ \frac{\phi}{2}, \arctan(\frac{1}{d^{\frac12} \cos\phi}), \arccos(d^{\frac14} \cos\phi) }$. As we will see below, this is a ``critical threshold'' of the integral $\int_0^\phi F(\theta) d\theta$, in the sense that the contribution of $[\theta', \psi]$ to the integral is negligible.

    
By our assumption that $\cos\phi \leq \frac{1}{2d^{\frac14}}$, $\arccos(d^{\frac14} \cos\phi) \geq \frac\pi3$. 
    In addition, $\arctan(\frac{1}{d^{\frac12} \cos\phi}) \geq \min\del{ \frac{\pi}{4}, \frac{1}{ 2 d^{\frac12} \cos\phi} }$ by Lemma~\ref{lem:arctan} given after the proof. 
    Moreover, recall that $\phi \geq \frac \pi 3$.
    Combining the above bounds, $\theta' \geq  \min\del{ \frac \pi 6, \frac{1}{ 2 d^{\frac12} \cos\phi} }$. 
    
    We now upper bound $\Pi(\alpha)$. 
    First we upper bound the numerator: 
    \[ 
    \int_0^{\psi/2} F(\theta) d\theta \leq \psi/2 \cdot F(0) = \frac{\psi}{2} (1 - \cos^2\phi)^{\frac{d-2}{2}} \leq \frac{\psi}{2} \exp\del{-\frac{d-2}{2} \cos^2\phi}.
    \]
    
    We next lower bound the denominator. As $\theta' \leq \frac{\phi}{2} \leq \frac{\pi}{4}$ (since by definition, $\phi / 2 \leq \pi / 2$), this implies that $\cos^2\theta' \geq \frac12$ and hence $\phi \geq \pi/ 3 \Rightarrow \frac{\cos^2\phi}{\cos^2\theta'} \in [0,\frac12]$. 
    Therefore,
    \[
    \int_0^{\phi} F(\theta) d \theta 
    \geq
    \int_0^{\theta'} F(\theta) d \theta
    \geq 
    \theta' F(\theta') 
    =
    \theta' \del{ 1 - \frac{\cos^2\phi}{\cos^2\theta'} }^{\frac{d-2}{2}}
    \geq
    \theta' \exp\del{ - \frac{d-2}{2} \del{ \frac{\cos^2\phi}{\cos^2\theta'} + \frac{\cos^4\phi}{\cos^4\theta'}  } },
    \]
    where the last inequality uses the elementary fact that $1-x \geq \exp(-x-x^2)$ for $x \in [0,\frac12]$. 
    
    Combining the upper and lower bounds, we get that the integral ratio is bounded by:
    \[
    \frac{ \int_0^{\psi/2} F(\theta) d\theta }{ \int_0^{\phi} F(\theta) d\theta } 
    \leq 
    \frac{\psi}{2 \theta'} \exp\del{ \frac{d-2}{2} \del{ \cos^2\phi \tan^2\theta' + \frac{\cos^4\phi}{\cos^4\theta'}} }
    \]
    
    From our choice of $\theta'$, it can be easily seen that: (1) $\cos^2\phi \tan^2\theta' \leq \cos^2\phi \cdot \frac{1}{d \cos^2\phi} \leq \frac 1 d$, and (2) $\frac{\cos^4\phi}{\cos^4\theta'} \leq \frac{\cos^4\phi}{ (d^{\frac14} \cos\phi)^4 } \leq \frac 1 d$. This implies that the exponential term is at most $\exp\del{ \frac{d-2}{2} \cdot \frac{2}{d} } \leq e$.
    
    In conclusion, we have that:
    \[
    \frac{ \int_0^{\psi/2} F(\theta) d\theta }{ \int_0^{\phi} F(\theta) d\theta } 
    \leq
    \frac{e}{2} \cdot
    \frac{\psi}{\theta'}
    \leq
    6 \cdot \del{
    \psi (1 + d^{\frac12} \cos\phi) },
    \]

    where in the last inequality we recall that 
    $\theta' \geq  \min\del{ \frac \pi 6, \frac{1}{ 2 d^{\frac12} \cos\phi} }$, and use that for $A, B > 0$,$\max(A,B) \leq A + B$. 

\item Fix $c_1, c_2 > 0$, and let $\phi \geq c_1$.

If $\phi \leq \psi/2$, then $\Pi(\alpha) = 1 \geq 1 - c_2$ holds. 

For the rest of the proof, we focus on the case of $\phi > \psi/2$.
As $\phi \geq c_1 > 0$, $\cos\phi \leq \cos c_1 < 1$. 

Therefore there exists some small constant $c_5 > 0$ such that $\cos\phi \leq 1 - 2 c_5$; meanwhile there exists some small enough constant $c_4 < \frac14$ such that $\cos^2 (c_4 \psi) \geq 1 - c_5$ since $c_4 \psi \leq \pi/4$; as a consequence,
$\cos^2 \phi / \cos^2 (c_4 \psi) \leq \frac{1 - 2c_5}{1 - c_5} \leq 1 - c_5$. 
In summary, there exist some small enough constants $c_4, c_5 > 0$ (independent of $\phi$), such that $c_4 < \frac14$ and $\frac{\cos^2 \phi}{\cos^2 (c_4 \psi)} \leq 1- c_5$.

By Lemma~\ref{lem:1-x-approx} (deferred after the proof), there exists some constant $c_6 > 0$ (independent of $\phi$) such that \begin{equation}
    1-\frac{\cos^2\phi}{\cos^2(c_4 \psi)} \geq \exp\del{ -\del{\frac{\cos\phi}{\cos(c_4 \psi)}}^2 - c_6 \del{\frac{\cos\phi}{\cos(c_4 \psi)}}^4 }.
    \label{eqn:numerator-lb}
\end{equation}

Therefore,
  \begin{align*}
    \frac{\int_{0}^{\psi/2} F(\theta) d\theta}{\int_{\psi/2}^{\phi} F(\theta) d\theta} 
    \geq & \frac{\int_{0}^{c_4 \psi} F(\theta) d\theta}{\int_{\psi/2}^{\phi} F(\theta) d\theta}  \\
    \geq & \frac{c_4 \psi \cdot F(c_4 \psi)}{\phi \cdot F(\psi / 2)} \\
    \geq & \frac{2 c_4 \psi}{\pi} \cdot \frac{ \del{ 1 - \frac{\cos^2 \phi}{\cos^2 (c_2 \psi )}}^{(d-2)/2}}{\del{ 1 - \frac{\cos^2 \phi}{\cos^2 (\psi/2) } }^{(d-2)/2}} \\
    \geq & \frac{2 c_4 \psi}{\pi} \cdot \frac{\exp \del{ -\frac{d-2}{2} (\frac{\cos^2 \phi}{\cos^2 (c_4 \psi)} + c_6 (\frac{\cos^2 \phi}{\cos^2 (c_4 \psi)})^2) } }{\exp\del{ -\frac{d-2}{2} \frac{\cos^2 \phi}{\cos^2 (\psi/2)} } } \\
    = &\frac{2 c_4 \psi}{\pi} \cdot \exp\left(\frac{d-2}{2} \cos^2 \phi \del{ \frac{1}{\cos^2 (\psi/2)} - \frac{1}{\cos^2 (c_4 \psi)} - c_6  \frac{\cos^2 \phi}{\cos^4 (c_4 \psi)}} \right),
\end{align*}
where the first inequality is because $c_4 \leq \frac14$; the second inequality is because $F(\theta)$ is monotonically decreasing for $\theta \geq 0$; the third inequality follows from the definition of $F(\theta)$, and $\phi \leq \frac{\pi}{2}$; the fourth inequality is from Equation~\eqref{eqn:numerator-lb} as well as using $1-x \leq \exp(-x)$ to upper bound the denominator; the equality is by algebra. 

Observe: 
\begin{align*}
\frac{1}{\cos^2 (\psi/2)} - \frac{1}{\cos^2 (c_4 \psi)}
= &
\frac{\cos^2(c_4 \psi) - \cos^2(\psi/2)}{\cos^2(c_4 \psi) \cdot \cos^2(\psi/2)} \\
= & 
\frac{\sin^2(\psi/2) - \sin^2(c_4 \psi)}{\cos^2(c_4 \psi) \cdot \cos^2(\psi/2)} \\
= & \frac{(\sin(\psi/2) + \sin(c_4 \psi)) (\sin(\psi/2) - \sin(c_4 \psi)) }{\cos^2(c_4 \psi) \cdot \cos^2\psi} \\
\geq & \frac{ \frac{\psi}{2\pi} \cdot \cos(\psi/2) \frac{\psi}{4}  }{  \cos^2(c_4 \psi) \cdot \cos^2\psi } \\
\geq & \frac{\psi^2}{8\pi}. 
\end{align*}
where the first inequality uses, $\sin(\psi/2) \geq \frac{\psi}{2\pi}$, and the Lagrange mean value theorem and the choice of $c_4$, such that $c_4 \leq \frac14$ so that $\sin(\psi/2) - \sin(c_4 \psi) = (\psi/2 - c_4 \psi) \cos\xi$ for some $\xi \in [c_4 \psi, \psi/2]$, which in turn is $\geq \frac{\psi}{4} \cos(\psi/2)$; the second inequality uses that $\cos(c_4 \psi) \geq \cos(\psi/2)$, and $\cos\gamma \leq 1$ for any $\gamma$.
 
With foresight, we will choose $c_3 \geq 16\sqrt{c_6}$, and defer the exact setting of $c_3$ to the next paragraph. By the assumption of lower bound on $\psi$ (Equation~\eqref{eqn:psi-lb}),
We have $\psi \geq 16 \sqrt{c_6} \cos\phi$, and therefore $\frac{\psi^2}{8 \pi} \geq 8 c_6 \cos^2 \phi$. 
In addition, recall that $c_4 \leq \frac14$, $c_6 \frac{\cos^2 \phi}{\cos^4 (c_4\psi)} \leq c_6 \cdot \frac{\cos^2 \phi}{\cos^4(\frac \pi 8)} \leq 4 c_6 \cos^2 \phi$. Hence, 
\[
\frac{1}{\cos^2 \psi} - \frac{1}{\cos^2 (c_2 \psi)} - c_4 \frac{\cos^2 \phi}{\cos^4 (c_2\psi)}
\geq \frac{\psi^2}{8\pi} \cdot (1 - \frac12) \geq \frac{\psi^2}{16 \pi}. 
\]

We would also like to set $c_3 > 0$ such that
\begin{equation}
 \exp\del{ \frac{d-2}{2} \cos^2 \phi \cdot \frac{\psi^2}{16 \pi} } \geq  \frac{\pi}{c_2 c_4 \psi},
 \label{eqn:c_3-setting-2}
\end{equation}
because this would imply that 
\[
\frac{\int_{0}^{\psi/2} F(\theta) d\theta}{\int_{\psi/2}^{\phi} F(\theta) d\theta} 
\geq  
\frac{2c_4 \psi}{\pi} \cdot \exp\del{ \frac{d-2}{2} \cos^2 \phi \cdot \frac{\psi^2}{16 \pi} }
\geq \frac{2}{c_2},
\]
which in turn implies
\[
\frac{\int_{0}^{\psi/2} F(\theta) d\theta}{\int_{0}^{\phi} F(\theta) d\theta} =
\frac{1}{1 + \frac{\int_{\psi/2}^{\phi} F(\theta) d\theta}{\int_{0}^{\psi/2} F(\theta) d\theta}}
=
\frac{1}{1 + c_2 / 2}
\geq
1 - c_2. 
\]
We analyze a sufficient condition for Equation~\eqref{eqn:c_3-setting-2} to hold:
\begin{align*}
    & \exp\del{ \frac{d-2}{2} \cos^2 \phi \cdot \frac{\psi^2}{16 \pi} } \geq  \frac{\pi}{c_2 c_4 \psi} \\
    \Leftarrow \; & 
    \frac{d-2}{2} \cos^2\phi \cdot \frac{\psi^2}{16\pi} \geq \ln\del{ \frac{\pi}{c_2 c_4} \cdot \frac{1}{\psi} } \\
    \Leftarrow \; & 
    \psi^2 \geq \frac{96\pi}{d \cos^2\phi} \ln\del{ \frac{2\pi}{c_2 c_4} \frac{1}{\psi^2} } \\
     \Leftarrow \; &
     \psi^2 \geq \frac{192\pi}{d \cos^2\phi} \del{ \ln\frac{8\pi}{c_2 c_4} + \ln\del{ 1 + \frac{96\pi}{d \cos^2\phi} } } \\
     \Leftarrow \; &
     \psi \geq \sqrt{ \frac{192\pi}{d \cos^2\phi} \del{ \ln\frac{8\pi}{c_2 c_4} + \ln\del{ 1 + \frac{96\pi}{d \cos^2\phi} } } }
\end{align*}
Therefore, choosing $c_3 = \max\del{ 16 \sqrt{c_6}, 2, 1 + \frac{\ln(96\pi) + \ln(\frac{2\pi}{c_4})}{ \ln\frac{4}{c_2} } }$ (which is independent of $\phi$), and by algebra, it satisfies 
$c_3 \frac{1}{d^{\frac12} \cos\phi} \sqrt{ \ln \frac{4}{c_2} + \ln\del{ 1+\frac{1}{d^{\frac12} \cos\phi}} } \geq \sqrt{ \frac{192\pi}{d \cos^2\phi} \del{ \ln\frac{8\pi}{c_2 c_4} + \ln\del{ 1 + \frac{96\pi}{d \cos^2\phi} } } }$, we have that Equation~\eqref{eqn:c_3-setting-2} is satisfied, and therefore $\Pi(\alpha) = \frac{\int_{0}^{\psi/2} F(\theta) d\theta}{\int_{0}^{\phi} F(\theta) d\theta} \geq 1 - c_2$.
\qedhere
\end{enumerate}
\end{proof}

\begin{lemma}
\label{lem:alnx}
For $a,b > 0$, $\zeta \in (0,1)$, if $a \geq 2b \del{ \ln\frac{4}{\zeta} + \ln(1+\frac{1}{b})}$, then $a \geq b \ln\frac{1}{\zeta a}$.
\end{lemma}
\begin{proof}
If $a \geq 2b \del{ \ln\frac{4}{\zeta} + \ln(1+\frac{1}{b})} = 2b \del{ \ln\frac{1}{\zeta} + \ln(4+\frac{4}{b})}$, then $a \geq 2b\ln\frac{1}{\zeta}$ and $a \geq 2b \ln(\max(e, \frac1{2b}))$ hold simultaneously. 

The latter condition implies that 
$\frac{1}{a} \leq \frac{\frac 1 {2b}}{ \ln(\max(e, \frac{1}{2b})) }$.
By Lemma~\ref{lem:xlnx}, this gives $\frac{1}{a} \ln\frac{1}{a} \leq \frac{1}{2b}$, in other words, $a \geq 2b \ln\frac{1}{a}$.

Now combine this with $a \geq 2b\ln\frac1\zeta$ by taking average on both sides, we get 
$a \geq \frac{1}{2}( 2b \ln\frac1\zeta + 2b\ln\frac1a) = b \ln\frac{1}{a \zeta}$. The lemma follows.
\end{proof}

\begin{lemma}
For $y > 0$, and $x \leq  \frac{y}{\ln(\max(e,y))}$, then $x \ln x \leq y$.
\label{lem:xlnx}
\end{lemma}
\begin{proof}
Define $x_0 := \frac{y}{\ln(\max(e,y))}$. We first verify that $x_0 \ln x_0 \leq y$.

\begin{enumerate}
\item If $y \leq e$, then $x_0 = y$; in this case, $x_0 \ln x_0 = y \ln y \leq y$ holds.

\item Otherwise, $y > e$. In this case, $x_0 = \frac{y}{\ln y} \leq y$. Therefore, $x_0 \ln x_0 \leq x_0 \ln y = y$. 
\end{enumerate}

Now, given $x \leq x_0$, we consider two cases of $x$: 
\begin{enumerate}
    \item If $x \leq \frac1e$, then $x \ln x < 0 < y$ holds.
    \item Otherwise, $x > \frac{1}{e}$, and since $f(x) = x \ln x$ is monotonically increasing in $(\frac1e, +\infty)$, we have that $x \ln x \leq x_0 \ln x_0 \leq y$.
\end{enumerate} 
In summary, if $x \leq x_0$, we must have $x \ln x \leq y$. 
\end{proof}

\begin{lemma}
\label{lem:1-x-approx}
For any $c_5 > 0$, there exists $c_6 > 0$ such that 
\[
1-x \geq \exp(-x - c_6 x^2), \quad  \forall x \in [0,1-c_5].
\]
\end{lemma}
\begin{proof}
It suffices to choose $c_6 > 0$ such that
\[
-\ln(1-x) \leq x + c_6 x^2, \quad  \forall x \in [0,1-c_5].
\]
By Taylor's expansion,
\begin{align*}
    -\ln(1-x)
    = & x + \sum_{i=2}^\infty \frac{x^i}{i} \\
    \leq & x + \frac{x^2}{2} \del{ \sum_{i=0}^\infty x^i } \\
    \leq & x + \frac{x^2}{2(1-x)},
\end{align*}
therefore, it suffices to choose $c_6 = \frac{1}{2c_5}$ such that the above is at most $x + c_6 x^2$ for all $x \in [0,1-c_5]$.
\end{proof}

\begin{lemma}
For $x \geq 0$,
$\arctan(x) \geq \min( \frac{\pi}{4}, \frac{x}{2} )$. 
\label{lem:arctan}
\end{lemma}
\begin{proof}
We consider two cases:
\begin{enumerate}
    \item If $x \geq 1$, $\arctan(x) \geq \frac{\pi}{4} \geq \min( \frac{\pi}{4}, \frac{x}{2} )$.
    \item If $x < 1$, by mean value theorem, there exists some $\xi \in [0,x]$, such that 
    $\arctan(x) = 0 + x \cdot \eval[1]{(\arctan(z))'}_{z = \xi} = \frac{x}{1+\xi^2} \geq \frac{x}{2} \geq \min( \frac{\pi}{4}, \frac{x}{2} )$. 
\end{enumerate}
The lemma follows by combining the two cases.
\end{proof}

\setcounter{proposition}{0}

\subsection{Section 5 Proofs}

In this section, we provide complementary negative results to the positive results obtained under the assumptions that: 1) $\features$ is a sphere; and 2) $\Ucal$ is the uniform distribution over $\hypothesis$, the class of homogeneous linear models. We show that removing one of the two conditions, i.e either allowing for non-spherical features (Proposition~\ref{prop:non-spherical-neg-result}) or allowing $\Ucal$ to be non-uniform over $\hypothesis$ (Proposition~\ref{prop:non-uniform-neg-result}), leads to non-monotonicity.

\begin{proposition}\label{prop:non-spherical-neg-result}

Suppose $d=2$. We have uniform prior over homogeneous linear models $\hypothesis = \{h_w \mid w \in \mathbb{R}^d, \|w\| = 1\}$, there exists a feature space $\features$ and thresholds $0 < \alpha_2 < \alpha_1$ such that $\Pi(\alpha_2) < \Pi(\alpha_1)$.
\end{proposition}

\begin{proof}
Define
$\Xcal = \cbr{x^1, x^2, x^3, z^1, z^2}$, with the choices of $x^1, x^2, x^3, z^1, z^2$ specified shortly.

\begin{figure}
    \centering
    \includegraphics[width=0.5\textwidth]{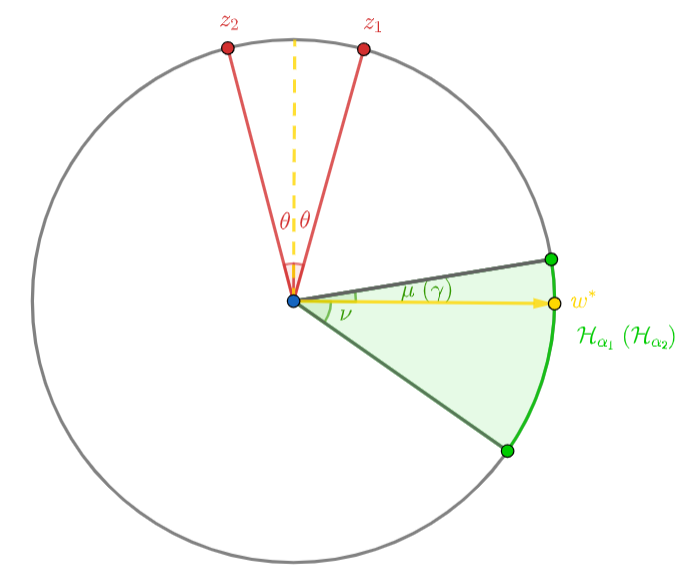}
    \caption{An illustration of $\Hcal_{\alpha_1}$ and $\Hcal_{\alpha_2}$ in the proof of Proposition~\ref{prop:non-spherical-neg-result}.}
    \label{fig:my_label}
\end{figure}

Let $w^\star = (1,0)$, and therefore $h^\star((x_1, x_2)) = \sign(x_1)$.
Let $\theta \in (0,\frac{\pi}{4})$ be an angle.
Define $z^1 = (\frac{r}{2}\sin\theta, \frac{r}{2} \cos\theta)$, $z^2 = (-\frac{r}{2}\sin\theta, \frac{r}{2} \cos\theta)$; it can be readily seen that $\| z^1 - z^2 \| \leq r$ and $\sign(h^\star(z^1))= +1 \neq -1 = \sign(h^\star(z^2))$; therefore $(z^1, z^2) \in \Mcal_r(\Xcal)$. As we will see shortly, this is the only pair in $\Mcal_r(\Xcal)$ up to reordering.

Let $\alpha_1', \alpha_2'$ be such that $0 < r < \alpha_2' < \alpha_1'$, and angles $\gamma, \mu, \nu$ be such that $\gamma < \mu < \theta < \nu$, and $\theta + \nu < \frac\pi2$.
Define $x^1 = (\alpha_1', -\alpha_1' \cot\mu)$, $x^2 = (\alpha_1', \alpha_1' \cot\nu)$, and $x^3 = (\alpha_2', -\alpha_2' \cot\gamma)$. It can be seen that $h^*(x^1) = h^*(x^2) = h^*(x^3) = +1$; in addition, note that all of $\| x^1 - z^2 \|$, $\| x^2 - z^2 \|$, $\| x^3 - z^2 \|$ are $> r$, ensuring that $\Mcal_r(\Xcal) = \cbr{ (z^1, z^2), (z^2, z^1) }$.

Let $\alpha_2 = \alpha_2'/2$ and $\alpha_1 = (\alpha_1' + \alpha_2') / 2$. Observe that $\cbr{x \in \Xcal: \Lambda_{\alpha_1}(x) = 1} = \cbr{x^1, x^2}$,
and $\cbr{x \in \Xcal: \Lambda_{\alpha_2}(x) = 1} = \cbr{x^1, x^2, x^3}$.

\paragraph{Numerical Example.} For concreteness, we can take $\alpha_1' = 10$, $\alpha_2' = 5$, $\alpha_1 = 7.5$, $\alpha_2 = 2.5$, $r = 1$, $\gamma = \frac{\pi}{16}$, $\mu = \frac{\pi}{12}$, $\theta = \frac{\pi}{8}$, and $\nu = \frac{\pi}{4}$, which satisfy all requirements above.


Given $w = (w_1, w_2) \in \RR^2$, denote by $\phi(w) \in (-\pi, \pi]$ its polar angle with respect to $(1,0)$ (so that $\phi((1,0)) = 0$).

We now calculate $\Pi(\alpha_1)$.
First, observe that
\begin{align*}
\Hcal_{\alpha_1} = & \cbr{h \in \Hcal: h(x^1) = 1, h(x^2) = 1} = \cbr{h_w: \|w\|_2 = 1, \phi(w) \in [-\nu,\mu]} 
\end{align*}
Therefore,
\begin{align*}
\Pi(\alpha_1) 
= &
\max_{(x,x') \in \margin}
\del{ \PP_{h_w \sim \Ucal(\Hcal_{\alpha_1})}( \inner{w}{x} \geq 0 ) - \PP_{h_w \sim \Ucal(\Hcal_{\alpha_1})}( \inner{w}{x'} \geq 0 ) }
\\
= & \abs{ \PP_{h_w \sim \Ucal(\Hcal_{\alpha_1})}( \inner{w}{z^1} \geq 0 ) - \PP_{h_w \sim \Ucal(\Hcal_{\alpha_1})}( \inner{w}{z^2} \geq 0 ) }
\\
= & \abs{ \frac{\mu + \theta}{\mu + \nu} - 0 }
= \frac{\mu + \theta}{\mu + \nu}.
\end{align*}




We now calculate $\Pi(\alpha_2)$. First observe that
\[
\Hcal_{\alpha_2} = \cbr{h \in \Hcal: h(x^1) = 1, h(x^2) = 1, h(x^3) = 1 } = \cbr{h_w: \| w \|_2 = 1, \phi(w) \in [-\nu, \gamma]}
\]
Therefore,
\begin{align*}
\Pi(\alpha_2) 
= &
\max_{(x,x') \in \margin}
\del{ \PP_{h_w \sim \Ucal(\Hcal_{\alpha_2})}( \inner{w}{x} \geq 0 ) - \PP_{h_w \sim \Ucal(\Hcal_{\alpha_2})}( \inner{w}{x'} \geq 0 ) } \\
= & 
\abs{ \PP_{h_w \sim \Ucal(\Hcal_{\alpha_2})}( \inner{w}{z^1} \geq 0 ) - \PP_{h_w \sim \Ucal(\Hcal_{\alpha_2})}( \inner{w}{z^2} \geq 0 )
}
\\
= & 
\abs{ \frac{\gamma + \theta}{\gamma + \nu} - 0
}
= \frac{\gamma + \theta}{\gamma + \nu}.
\end{align*}

In conclusion, 
\[
\Pi(\alpha_1)
= 
\frac{\mu+\theta}{\mu+\nu} 
\geq 
\frac{\gamma+\theta}{\gamma+\nu}
=
\Pi(\alpha_2).
\qedhere
\]


\end{proof}

\begin{proposition}\label{prop:non-uniform-neg-result}
Suppose $\features$ is the $d$-dimensional unit sphere with $d \geq 3$. There exists a non-uniform distribution $\Ucal$ over homogeneous linear models $\hypothesis$,  such that there exists thresholds $0 < \alpha_2 < \alpha_1$ with $\Pi(\alpha_2) < \Pi(\alpha_1)$.
\end{proposition}
\begin{proof}
WLOG, we assume that $w^* = (1, 0, \ldots, 0)$.
Define 
$x = (-\sin (\psi/2), \cos (\psi/2), 0, \ldots, 0)$ and $x' = (\sin (\psi/2), \cos (\psi/2), 0, \ldots, 0)$ which will be used later. 
It can be seen that $x, x'$ and $w^*$ are on the same 2-dimensional plane.

Let $\alpha_2, \alpha_1$ be such that $0 < \alpha_2 < \alpha_1 < 1$ and with $\phi_1 = \arcsin\alpha_1$ and $\phi_2 = \arcsin\alpha_2$, $\phi_1 > \phi_2 > \psi / 2$. We know from Lemma~\ref{lemma: circle_vs} that 
\[ 
\hypothesis_{\alpha_2} = \cbr{h_w: \| w \|_2 = 1, \langle w,  w^* \rangle \geq \sqrt{1 - \alpha_2^2}} \subset \hypothesis_{\alpha_1} = \cbr{h_w: \| w \|_2 = 1, \langle w,  w^* \rangle \geq \sqrt{1 - \alpha_1^2}},
\]
and that $\hypothesis_{\alpha_1} \backslash \hypothesis_{\alpha_2} = \cbr{h_w: \| w \|_2 = 1, \langle w,  w^* \rangle \in [ \sqrt{1 - \alpha_1^2}, \sqrt{1 - \alpha_2^2}) }$.

We define the density of the non-uniform prior $\mathcal{U}$ as follows.
Let $\mathcal{U}$ be uniform when restricted to  $\Hcal_{\alpha_2}$. 
And let $\mathcal{U}$ have positive density that is uniform over $\{h_w: w \in \hypothesis_{\alpha_1} \backslash \hypothesis_{\alpha_2}, -1 = \sign(w \cdot x) \neq \sign(w \cdot x') = +1\}$; note that this is an non-empty set as it comprises of all $w$'s whose projection onto $w^*$ has value in $[ \sqrt{1 - \alpha_1^2}, \sqrt{1 - \alpha_2^2}]$ and has polar angle wrt $w^*$ in  $[-\psi/2, \psi/2]$. Finally, let $\Ucal$ have zero density over all other parts of $w \in \hypothesis_{\alpha_1} \backslash \hypothesis_{\alpha_2}$. The density of $\mathcal{U}$ outside $\hypothesis_{\alpha_1}$ can be chosen arbitrarily. See Figure~\ref{fig:distn-u} for an illustration. 

\begin{figure}
    \centering
    \includegraphics[width=0.5\textwidth]{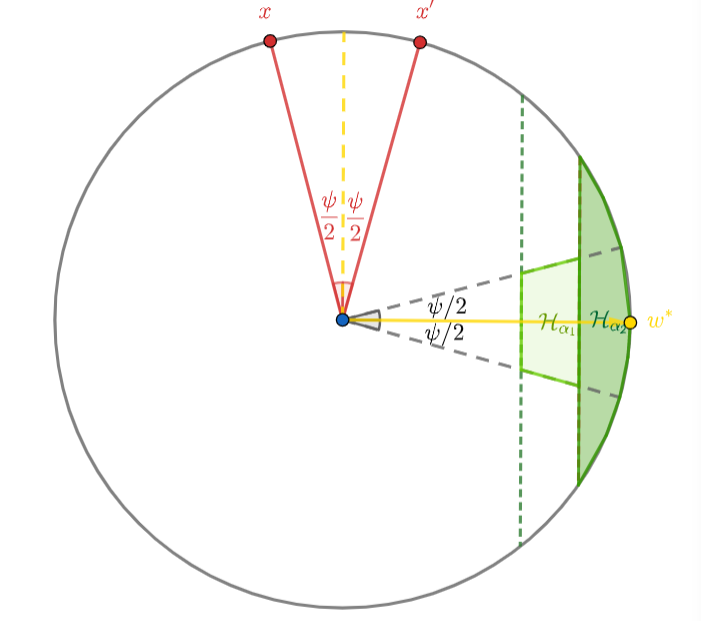}
    \caption{In the proof of Proposition~\ref{prop:non-uniform-neg-result}, a projection of $\Ucal$ onto the 2-dimensional plane spanned by $w^*$, $x$ and $x'$; it is uniform when restricted to $\Hcal_{\alpha_2}$ (the dark green region), and is concentrated in $\{h_w \in \hypothesis_{\alpha_1} \backslash \hypothesis_{\alpha_2}: -1 = \sign(w \cdot x) \neq \sign(w \cdot x') = +1\}$ (the light green region) when restricted to $\Hcal_{\alpha_1} \setminus \Hcal_{\alpha_2}$.}
    \label{fig:distn-u}
\end{figure}

By the definition of $x, x'$, and the fact that $\Ucal$ is uniform when restricted to $\Hcal_{\alpha_2}$, from the proof of Theorem~\ref{thm:max-pi-sphere}, $(x, x') \in \arg\max_{(x,x') \in \Mcal_r(\Xcal)} \pi_{\alpha_2}(x, x')$; in other words, $\pi_{\alpha_2}(x, x') = \Pi(\alpha_2)$. 



 

With this, we know that since $\phi_0 > \psi$, $\Pi(\alpha_2) = \pi_{\alpha_2}(x,x') < 1$. Then, 
\begin{align*}
    \Pi(\alpha_1) \geq \pi_{\alpha_1}(x, x') & = \PP_{h_w \sim \Ucal(\Hcal_{\alpha_1})}( \inner{w}{x'} \geq 0 ) - \PP_{h_w \sim \Ucal(\Hcal_{\alpha_1})}( \inner{w}{x} \geq 0 ) \\
    & = \del{ \PP_{h_w \sim \Ucal(\Hcal_{\alpha_2})}( \inner{w}{x'} \geq 0 ) - \PP_{h_w \sim \Ucal(\Hcal_{\alpha_2})}( \inner{w}{x} \geq 0 ) } \cdot \PP_{h_w \sim \Ucal(\Hcal_{\alpha_1})}(w \in \Hcal_{\alpha_2}) + \\
    & \quad  \del{ \PP_{h_w \sim \Ucal(\hypothesis_{\alpha_1} \backslash \hypothesis_{\alpha_2})}( \inner{w}{x'} \geq 0 ) - \PP_{h_w \sim \Ucal(\hypothesis_{\alpha_1} \backslash \hypothesis_{\alpha_2})}( \inner{w}{x} \geq 0 ) } \cdot \PP_{h_w \sim \Ucal(\Hcal_{\alpha_1})}(w \in \hypothesis_{\alpha_1} \backslash \hypothesis_{\alpha_2}) \\
    &= \pi_{\alpha_2}(x, x') \cdot \PP_{h_w \sim \Ucal(\Hcal_{\alpha_1})}(w \in \Hcal_{\alpha_2}) + \PP_{h_w \sim \Ucal(\Hcal_{\alpha_1})}(w \in \hypothesis_{\alpha_1} \backslash \hypothesis_{\alpha_2}) \\
    & >  \pi_{\alpha_2}(x, x') = \Pi(\alpha_2),
\end{align*} 
where the first inequality is from the definition of $\Pi(\alpha_1)$;
the 
first equality is by the definition of $\pi(\alpha_1)$; the second equality is by the total law of probability; the third equality is by the construction that $\Ucal$ has zero density in $\{h_w: w \in \hypothesis_{\alpha_1} \backslash \hypothesis_{\alpha_2},  \sign(w \cdot x) = +1 \vee \sign(w \cdot x') = -1\}$, so that $\PP_{h_w \sim \Ucal(\hypothesis_{\alpha_1} \backslash \hypothesis_{\alpha_2})}( \inner{w}{x'} \geq 0 ) = 1$ and $\PP_{h_w \sim \Ucal(\hypothesis_{\alpha_1} \backslash \hypothesis_{\alpha_2})}( \inner{w}{x} \geq 0 ) = 0$, along with the definition of $\pi_{\alpha_2}(x,x')$; 
the last inequality is strict because $\PP_{h_w \sim \Ucal(\Hcal_{\alpha_1})}(w \in \hypothesis_{\alpha_1} \backslash \hypothesis_{\alpha_2}) > 0$ and that $\Pi(\alpha_2) = \pi_{\alpha_2}(x,x') < 1$.
\end{proof}

Lastly, fixing assumptions 1 and 2, one may also wonder if it is possible to achieve any threshold $\kappa$ in the more general, non-homogeneous linear models. We saw that this is not so asymptotically in the homogeneous case (Theorem~\ref{thm:pi-bounds}). Here, we demonstrate that this does not hold in general.

\begin{figure}
    \centering
    \includegraphics[scale=0.5]{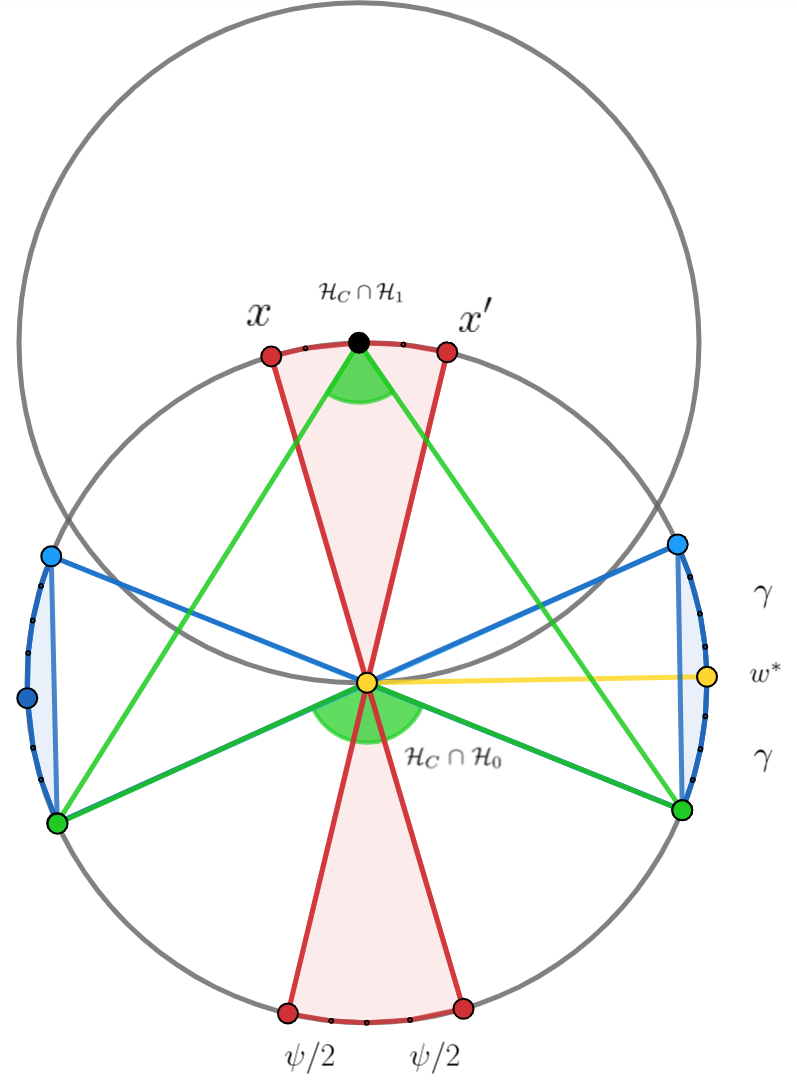}
    \caption{The construction in Proposition~\ref{prop:non-homog-neg-result}. In blue are the explanations, in green are the decision boundaries of models in the version space, in red is the margin region and in yellow is $w^*$.}
    \label{fig: nonmonotonic}
\end{figure}


\begin{proposition}\label{prop:non-homog-neg-result}
There exists a class of 2-dimensional non-homogeneous linear models, with spherical $\features$  such that $\Pi(\alpha)$ decreases monotonically (and strictly so at some point) with increasing $\alpha$, and yet $\Pi(\alpha) \geq 1 / 3$ for all $\alpha \in [0,1)$ and $\psi \in (0,\pi]$.
\end{proposition}

\begin{proof}
Let the hypothesis class of interest be $\hypothesis = \hypothesis_{1} \cup \hypothesis_0$, where
\[ \hypothesis_0 = \cbr{ x \mapsto \sign(w_1 x_1 + w_2 x_2): \|w \|_2 = 1} \] 
is its homogeneous part, 
and
\[ 
\hypothesis_1 = \cbr{ x \mapsto \sign(w_1 x_1 + w_2 (x_2 - 1)): \|w \|_2 = 1}
\]
is its non-homogeneous part.

We will take same setting as before $\features$ is a unit circle centered at $(0,0)$ and $\Ecal_{h^*}(\Xcal, \alpha) = \{x \in \features \mid \Lambda_{\alpha}(x) = 1\}$.  We assume an uniform prior $\Ucal$ over $\hypothesis$, i.e. drawing $i \sim \Bernoulli(\frac12)$, and chooses a classifier uniformly at random from $\Hcal_i$ induces $\Ucal$. 

Let $h^*(x) = x \mapsto \sign(x_1)$, which is a member of $\Hcal$. We consider a boundary pair  $(x, x') \in \Mcal_r(\Xcal)$ where $\|x' - x \|_2 \leq r$, $h^*(x') = +1 \neq -1 = h^*(x)$. 


Given $w = (w_1, w_2) \in \RR^2$, denote by $\phi(w) \in (-\pi, \pi]$ its polar angle with respect to $(1,0)$ (so that $\phi((1,0)) = 0$).

Given a value of $\alpha \in [0,1)$, the induced explanation set 
\[ 
\Ecal_{h^*}(\Xcal,\alpha) = \cbr{x \in \Xcal: \phi(x) \in [-\pi, -\pi + \gamma) \cup (-\gamma, \gamma) \cup (\pi - \gamma, \pi]},
\]
with $\gamma = \arccos \alpha \in (0, \frac \pi 2]$.


We will examine the structure of version space $\consistent$ and count how much of it predicts $(x, x')$ differently. Please refer to Figure~\ref{fig: nonmonotonic} for an illustration. We will look at $\consistent \cap \hypothesis_1$ and $\consistent \cap \hypothesis_0$ respectively. 



\paragraph{Part 1: $\consistent \cap \hypothesis_1$.} For any $h \in \consistent \cap \hypothesis_1$, it always holds that $h(x) = +1$ and $h(x') = -1$ as long as $\gamma > 0$. This is because if the explanation is nonempty, then it includes points $(-1, 0)$ and $(1, 0)$, which enforces that any $h \in \consistent \cap \hypothesis_1$ must be a subset of $h \in \hypothesis_1$ with polar angle in interval $[- \pi / 4, \pi /4]$ and all such $h$'s predict $(x, x')$ differently. More specifically, 
\[
\consistent \cap \hypothesis_1 
= 
\cbr{ x \mapsto \sign(w_1 x_1 + w_2 (x_2 - 1)): \|w \|_2 = 1, \phi(w) \in \intcc{-(\frac \pi 4 - \frac \gamma 2), \frac \pi 4 - \frac \gamma 2} },
\]
whose total arc length of $\frac \pi 2 - \gamma$. 
To summarize,
\[
\PP_{h \sim \Ucal} \del{ h \in \Hcal_C \cap \Hcal_1 } 
= \frac 12 \cdot \frac{\frac \pi 2 - \gamma}{2\pi} 
=
\frac{\frac\pi2-\gamma}{4\pi},
\]
and
\[
\PP_{h \sim \Ucal(\Hcal_C \cap \Hcal_1)}( h(x') = +1) - \PP_{h \sim \Ucal(\Hcal_C \cap \Hcal_1)}( h(x) = +1 ) = 1.
\]

\paragraph{Part 2: $\consistent \cap \hypothesis_0$.} As we showed in Lemma 1, 
\[ \hypothesis_0 = \cbr{ x \mapsto \sign(w_1 x_1 + w_2 x_2): \|w \|_2 = 1, \phi(w) \in \intcc{-(\frac\pi2-\gamma), \frac\pi2-\gamma} }, \] 
whose total arc length is $\pi-2\gamma$.

In addition, by Theorem~\ref{thm:max-pi-sphere} with $d=2$ with $\phi = \frac\pi2 - \gamma$, we have
\[
\max_{(x, x') \in \margin} \del{\PP_{h \sim \Ucal(\Hcal_C \cap \Hcal_0)}( h(x') = +1) -  \PP_{h \sim \Ucal(\Hcal_C \cap \Hcal_0)} ( h(x) = +1 )} = \begin{cases}
\frac{\psi}{2(\frac\pi2 - \gamma)}  & \psi \leq 2(\frac\pi2 - \gamma), \\
1 & \psi > 2(\frac\pi2 - \gamma).
\end{cases}
\]
To summarize,
\[
\PP_{h \sim \Ucal} \del{ h \in \Hcal_C \cap \Hcal_0 } = \frac{2(\frac\pi2-\gamma)}{4\pi}
\]

which is twice $\PP_{h \sim \Ucal} \del{ h \in \Hcal_C \cap \Hcal_1 }$ and,

\[
\max_{(x, x') \in \margin} \del{ \PP_{h \sim \Ucal(\Hcal_C \cap \Hcal_0)}( h(x') = +1) -  \PP_{h \sim \Ucal(\Hcal_C \cap \Hcal_0)} ( h(x) = +1 ) }
 = \min\del{ 1, \frac{\psi}{2(\frac\pi2 - \gamma)}}
 \]

Combining the two parts, observe that 
$\PP_{h \sim \Ucal(\Hcal_C)}(h \in \Hcal_C \cap \Hcal_0) = \frac23$, and  
by the law of total probability, 
\begin{align*}
\Pi(\alpha)
= &
\max_{(x,x') \in \Mcal_r(\Xcal)} \pi_\alpha(x,x') \\
= &
\max_{(x,x') \in \Mcal_r(\Xcal)} \del{ \PP_{h \sim \Ucal(\Hcal_C)}(h(x) = +1)
-
\PP_{h \sim \Ucal(\Hcal_C)}(h(x') = +1)} \\
= &
\max_{(x,x') \in \Mcal_r(\Xcal)} \left(
\PP_{h \sim \Ucal(\Hcal_C)}(h \in \Hcal_C \cap \Hcal_0) \cdot \del{ \PP_{h \sim \Ucal(\Hcal_C \cap \Hcal_0)}( h(x') = +1) - \PP_{h \sim \Ucal(\Hcal_C \cap \Hcal_0)}( h(x) = +1 )} \right. \\
& \left. + 
\PP_{h \sim \Ucal(\Hcal_C)}(h \in \Hcal_C \cap \Hcal_1) \cdot \del{ \PP_{h \sim \Ucal(\Hcal_C \cap \Hcal_1)}( h(x') = +1) - \PP_{h \sim \Ucal(\Hcal_C \cap \Hcal_1)}( h(x) = +1 )} \right) \\
= & 
\frac23 \cdot \min\del{ 1, \frac{\psi}{2(\frac\pi2 - \gamma)}} + \frac13  \\ 
\geq &  \frac13
\end{align*}

through which we see that $\Pi(\alpha)$ is increasing in $\gamma$ and strictly so for when $\pi /2  - \gamma > \psi/2$.
In other words, $\Pi(\alpha)$ is identically $1$ for $\alpha \in [0, \sin(\psi/2)]$, and is strictly decreasing in $\alpha$ for $\alpha \in [\sin(\psi/2), 1)$.
\end{proof}

\section{Additional Experiments}

\subsection{Fair accessibility to explanations}

A notable concern that may arise with margin distancing is that omission of prototypical explanations is necessary for regions close to the margin. Thus, this could disproportionately affect individuals in those regions, since they will not have their representative explanation be in the explanation set. We plot the composition of margin set in Figure~\ref{fig: margin_train_and_moral_hazard} with a threshold of $0.03$ 
for both logistic and SVM models and note that there is some disproportionate effect. Verily, this is another important factor that needs to be taken into account in the explanation generation process.

\begin{figure}
\begin{center}
\begin{subfigure}[b]{0.33\textwidth}
    \centering
    \includegraphics[width=\linewidth]{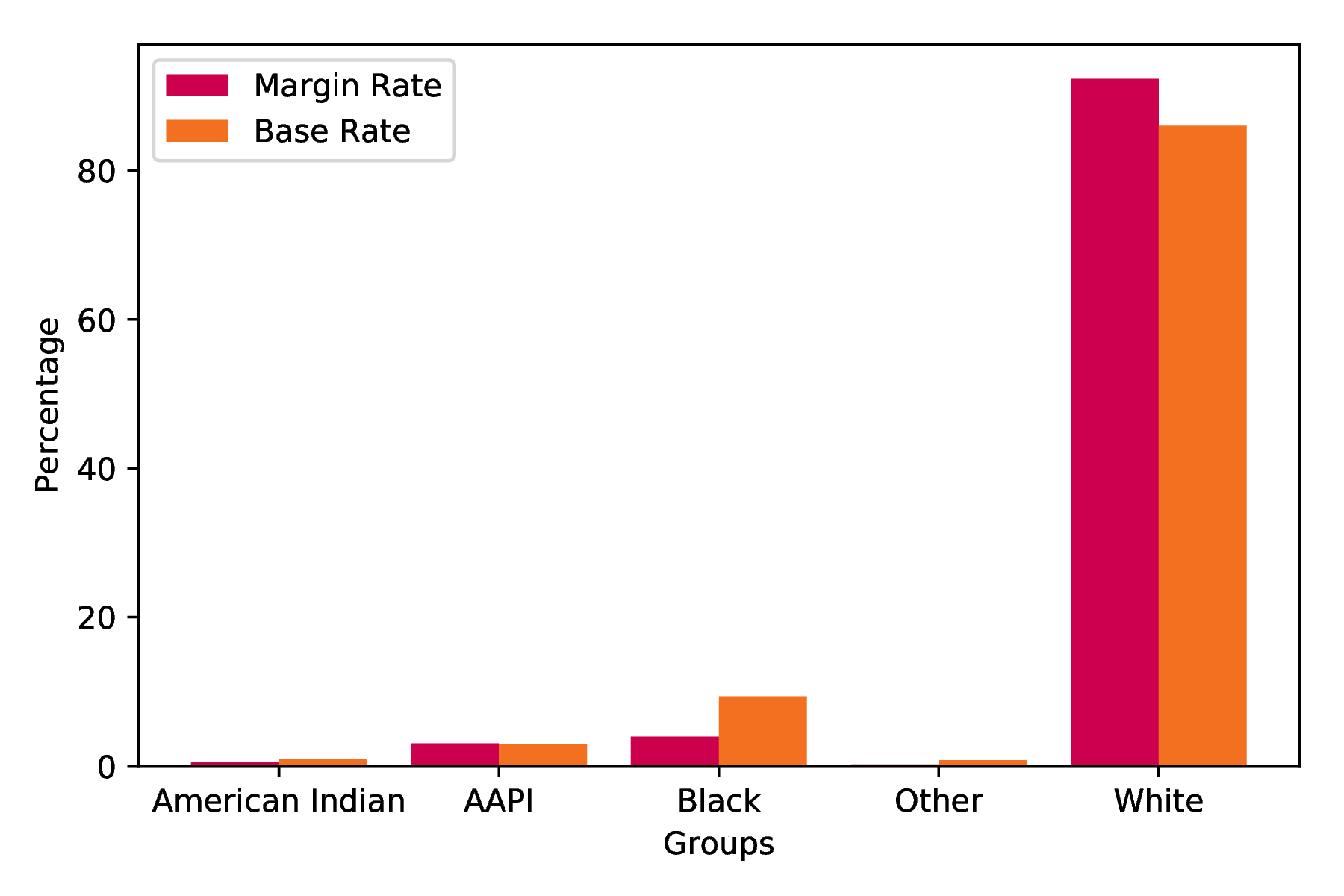}
\end{subfigure}%
\begin{subfigure}[b]{0.33\textwidth}
    \centering
    \includegraphics[width=\linewidth]{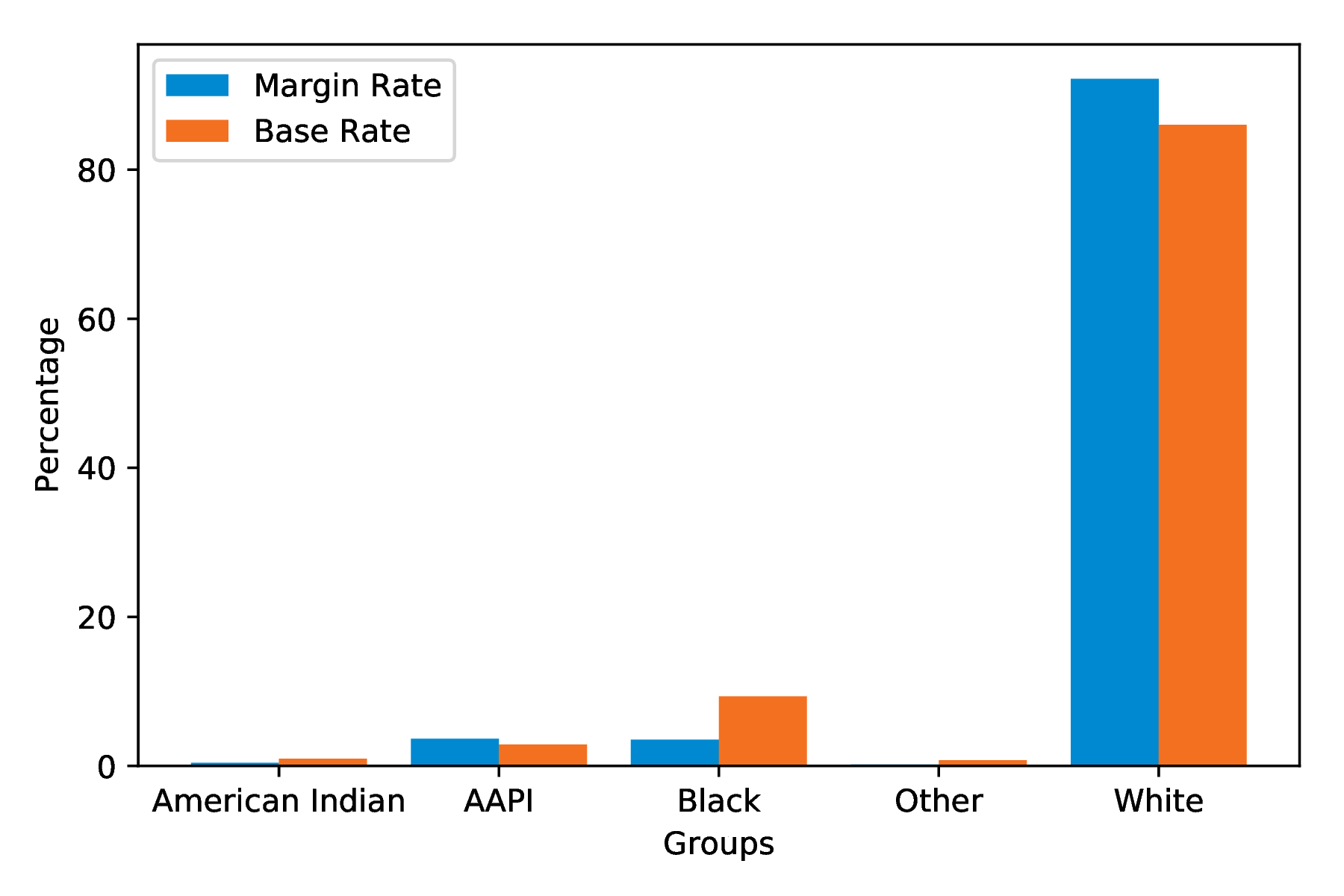}
\end{subfigure}%
\caption{Racial composition of margin points under LR (left) and SVM (right).}
\label{fig: margin_train_and_moral_hazard}
\end{center}
\end{figure}

\subsection{MMD Explanations}

We include results on the trend of the three metrics under MMD-Critic explanations to further empirically trace how the boundary certainty varies with explanation omission. Similar to the MLP results under $k$-medoid, we see that in Figure~\ref{fig:mmd_mlp} the trend is almost monotonic everywhere. One difference however, is that the boundary certainty does not drop off as fast as in the $k$-medoid setting. This suggests that the search strategy of trying small omission percentages may work with some explanation methods such as the $k$-medoid, but will not with others like MMD-Critic.

\subsection{Effects of Larger Models} 

We include results on the trend of the three metrics for a two hidden-layer MLP to showcase the effects of larger models. In Figure~\ref{fig:mlp_bigger}, we see similar trends under both explanations, but with higher values across the board in comparison with the one-layer case. Again, as in the one-layer MLP case, under MMD-critic explanations, the drop in the metrics are slower than the drop under $k$-medoid explanations.

\subsection{Monotonicity Tables}
\label{sec:monot_tables}

We present tables charting the differences between the percentage of explanations omitted calculated through binary search and the optimal percentage of explanation calculated through a left-to-right linear search, for ten, equally spaced out values of target boundary certainty corresponding to Figure~\ref{fig: credit_linear} in Tables~\ref{tab:max-r01} through~\ref{tab:avg-r03}.

\begin{table}
\centering
\begin{tabular}{||c c c c||} 
 \hline
 Target Certainty & Binary Search & Optimal & Difference \\ [0.5ex] 
 \hline\hline
 0.036 & 45 & 10 & 35 \\ 
 \hline
 0.046 & 45 & 10 & 35  \\
 \hline
 0.055 & 10 & 10 & 0 \\
 \hline
 0.065 & 10 & 10 & 0 \\
 \hline
 0.075 & 10 & 10 & 0  \\ 
 \hline
 0.084 & 10 & 10 & 0 \\
 \hline
 0.094 & 5 & 5 & 0 \\
 \hline
 0.103 & 5 & 5 & 0 \\
 \hline
 0.113 & 5 & 5 & 0 \\
 \hline
 0.122 & 5 & 5 & 0 \\ 
 \hline
\end{tabular}
\caption{\label{tab:max-r01} Difference table with the max metric and at $r=0.1$}
\end{table}

\begin{table}
\centering
\begin{tabular}{||c c c c||} 
 \hline
 Target Certainty & Binary Search & Optimal & Difference \\ [0.5ex] 
 \hline\hline
 0.071 & 70 & 15 & 55 \\ 
 \hline
 0.11 & 45 & 10 & 35  \\
 \hline
 0.15 & 10 & 10 & 0 \\
 \hline
 0.18 & 10 & 10 & 0 \\
 \hline
 0.22 & 10 & 10 & 0  \\ 
 \hline
 0.26 & 5 & 5 & 0 \\
 \hline
 0.30 & 5 & 5 & 0 \\
 \hline
 0.33 & 5 & 5 & 0 \\
 \hline
 0.37 & 5 & 5 & 0 \\
 \hline
 0.41 & 5 & 5 & 0 \\ 
 \hline
\end{tabular}
\caption{\label{tab:max-r02} Difference table with the max metric and at $r=0.2$}
\end{table}

\begin{table}
\centering
\begin{tabular}{||c c c c||} 
 \hline
 Target Certainty & Binary Search & Optimal & Difference \\ [0.5ex] 
 \hline\hline
 0.16 & 65 & 65 & 0 \\ 
 \hline
 0.23 & 25 & 25 & 0  \\
 \hline
 0.31 & 10 & 10 & 0 \\
 \hline
 0.39 & 5 & 5 & 0 \\
 \hline
 0.47 & 5 & 5 & 0  \\ 
 \hline
 0.55 & 5 & 5 & 0\\
 \hline
 0.63 & 5 & 5 & 0 \\
 \hline
 0.7 & 5 & 5 & 0 \\
 \hline
 0.78 & 5 & 5 & 0 \\
 \hline
 0.86 & 5 & 5 & 0 \\ 
 \hline
\end{tabular}
\caption{\label{tab:max-r03} Difference table with the max metric and at $r=0.3$}
\end{table}

\begin{table}
\centering
\begin{tabular}{||c c c c||} 
 \hline
 Target Certainty & Binary Search & Optimal & Difference \\ [0.5ex] 
 \hline\hline
 0.03 & 45 & 10 & 35 \\ 
 \hline
 0.04 & 45 & 10 & 35  \\
 \hline
 0.05 & 10 & 10 & 0 \\
 \hline
 0.06 & 10 & 10 & 0 \\
 \hline
 0.07 & 10 & 10 & 0 \\ 
 \hline
 0.08 & 10 & 10 & 0\\
 \hline
 0.09 & 5 & 5 & 0\\
 \hline
 0.1 & 5 & 5 & 0 \\
 \hline
 0.11 & 5 & 5 & 0\\
 \hline
 0.12 & 5 & 5 & 0 \\ 
 \hline
\end{tabular}
\caption{\label{tab:avg95-r01} Difference table with the top $5$ percentile average and at $r=0.1$}
\end{table}

\begin{table}
\centering
\begin{tabular}{||c c c c||} 
 \hline
 Target Certainty & Binary Search & Optimal & Difference \\ [0.5ex] 
 \hline\hline
 0.05 & 65 & 15 & 50 \\ 
 \hline
 0.07 & 40 & 10 & 30 \\
 \hline
 0.1 & 10 & 10 & 0 \\
 \hline
 0.12 & 10 & 10 & 0 \\
 \hline
 0.14 & 10 & 10 & 0 \\ 
 \hline
 0.17 & 5 & 5 & 0\\
 \hline
 0.19 & 5 & 5 & 0\\
 \hline
 0.21 & 5 & 5 & 0 \\
 \hline
 0.24 & 5 & 5 & 0\\
 \hline
 0.26 & 5 & 5 & 0\\
 \hline
\end{tabular}
\caption{\label{tab:avg95-r03} Difference table with the top $5$ percentile average and at $r=0.2$}
\end{table}

\begin{table}
\centering
\begin{tabular}{||c c c c||} 
 \hline
 Target Certainty & Binary Search & Optimal & Difference \\ [0.5ex] 
 \hline\hline
 0.11 & 65 & 65 & 0\\ 
 \hline
 0.17 & 25 & 25 & 0\\
 \hline
 0.23 & 10 & 10 & 0\\
 \hline
 0.3 & 10 & 10 & 0\\
 \hline
 0.36 & 5 & 5 & 0\\ 
 \hline
 0.42 & 5 & 5 & 0\\
 \hline
 0.48 & 5 & 5 & 0\\
 \hline
 0.54 & 5 & 5 & 0 \\
 \hline
 0.6 & 5 & 5 & 0\\
 \hline
 0.66 & 5 & 5 & 0\\
 \hline
\end{tabular}
\caption{\label{tab:avg95-r03} Difference table with the top $5$ percentile average and at $r=0.3$}
\end{table}

\begin{table}
\centering
\begin{tabular}{||c c c c||} 
 \hline
 Target Certainty & Binary Search & Optimal & Difference \\ [0.5ex] 
 \hline\hline
 0.008 & 45 & 15 & 30\\ 
 \hline
 0.014 & 45 & 10 & 35\\
 \hline
 0.019 & 40 & 10 & 30\\
 \hline
 0.025 & 10 & 10 & 0\\
 \hline
 0.031 & 5 & 5 & 0\\ 
 \hline
 0.037 & 5 & 5 & 0\\
 \hline
 0.042 & 5 & 5 & 0\\
 \hline
 0.048 & 5 & 5 & 0\\
 \hline
 0.054 & 5 & 5 & 0\\
 \hline
 0.06 & 5 & 5 & 0\\
 \hline
\end{tabular}
\caption{\label{tab:avg-r01} Difference table with the average and at $r=0.1$}
\end{table}

\begin{table}
\centering
\begin{tabular}{||c c c c||} 
 \hline
 Target Certainty & Binary Search & Optimal & Difference \\ [0.5ex] 
 \hline\hline
 0.013 & 65 & 10 & 55\\ 
 \hline
 0.02 & 45 & 10 & 35\\
 \hline
 0.027 & 10 & 10 & 0\\
 \hline
 0.034 & 10 & 10 & 0\\
 \hline
 0.041 & 5 & 5 & 0\\ 
 \hline
 0.049 & 5 & 5 & 0\\
 \hline
 0.056 & 5 & 5 & 0\\
 \hline
 0.063 & 5 & 5 & 0\\
 \hline
 0.07 & 5 & 5 & 0\\
 \hline
 0.077 & 5 & 5 & 0\\
 \hline
\end{tabular}
\caption{\label{tab:avg-r02} Difference table with the average and at $r=0.2$}
\end{table}

\begin{table}
\centering
\begin{tabular}{||c c c c||} 
 \hline
 Target Certainty & Binary Search & Optimal & Difference \\ [0.5ex] 
 \hline\hline
 0.044 & 65 & 65 & 0\\ 
 \hline
 0.075 & 40 & 30 & 10\\
 \hline
 0.106 & 10 & 10 & 0\\
 \hline
 0.137 & 10 & 10 & 0\\
 \hline
 0.168 & 5 & 5 & 0\\ 
 \hline
 0.199 & 5 & 5 & 0\\
 \hline
 0.229 & 5 & 5 & 0\\
 \hline
 0.26 & 5 & 5 & 0\\
 \hline
 0.291 & 5 & 5 & 0\\
 \hline
 0.322 & 5 & 5 & 0\\
 \hline
\end{tabular}
\caption{\label{tab:avg-r03} Difference table with the average and at $r=0.3$}
\end{table}


\begin{figure*}
  \begin{subfigure}[b]{.33\textwidth}
    \centering
    \includegraphics[width=\linewidth]{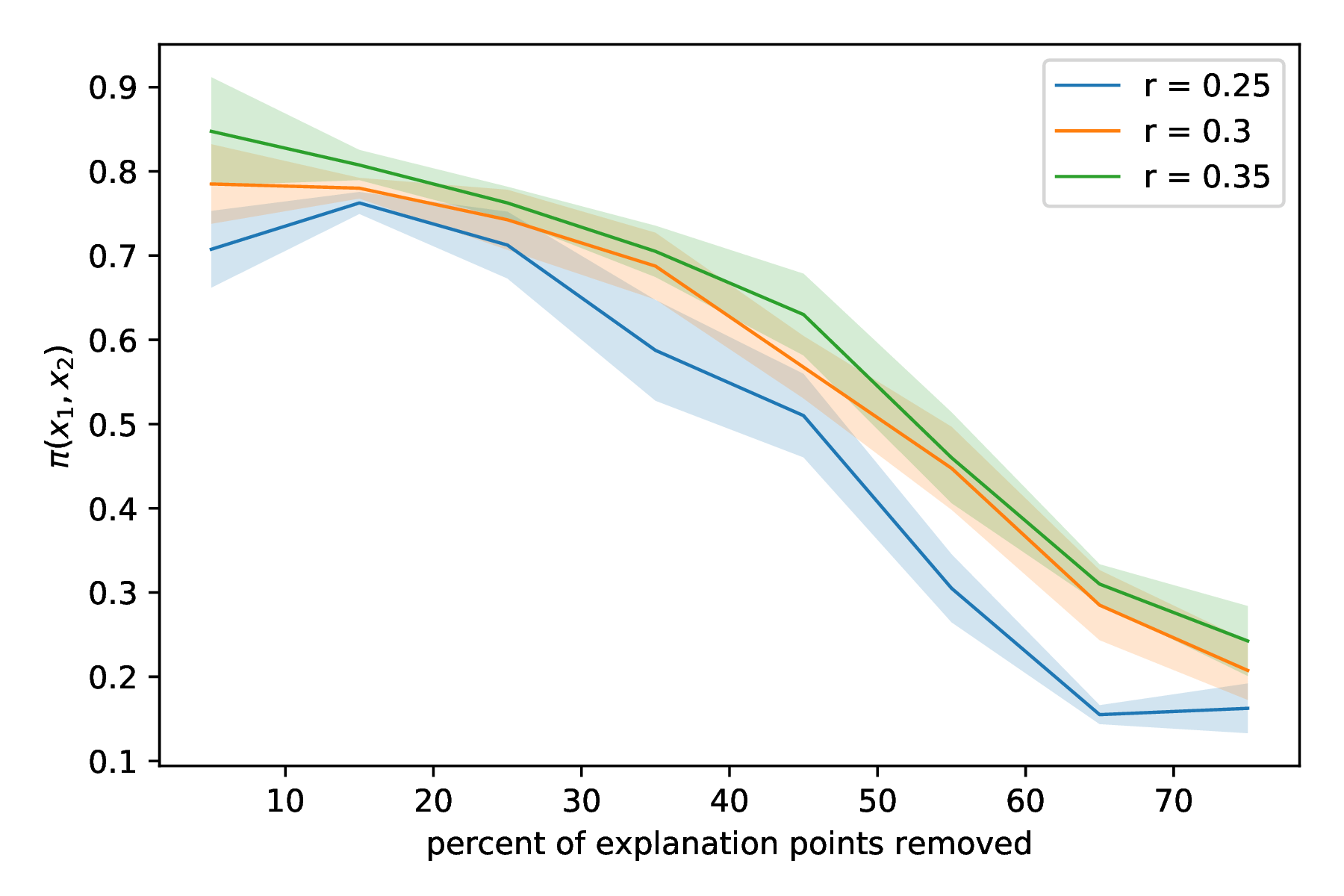}
  \end{subfigure}
  \begin{subfigure}[b]{.33\textwidth}
    \centering
    \includegraphics[width=\linewidth]{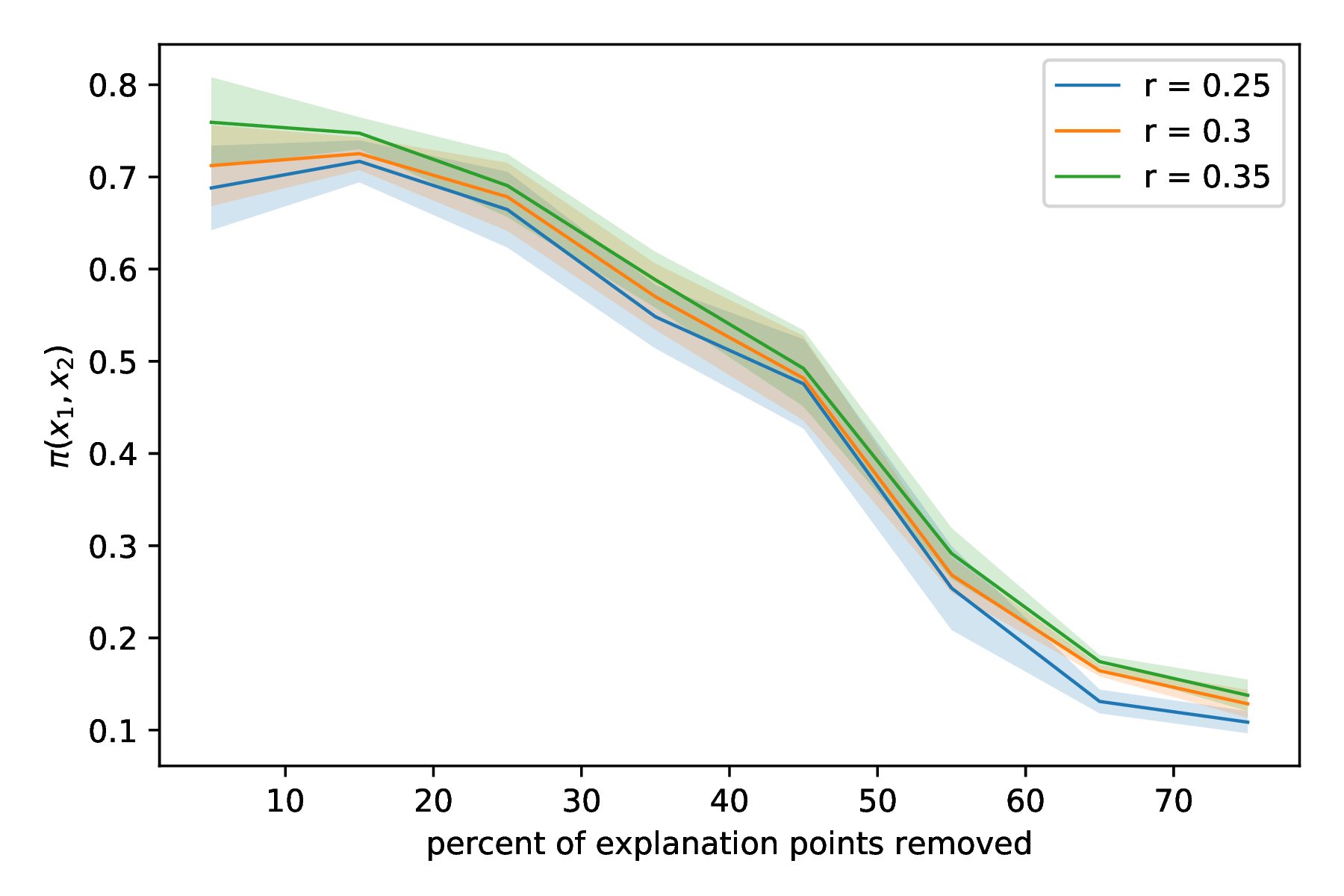}
  \end{subfigure}
  \begin{subfigure}[b]{.33\textwidth}
    \centering
    \includegraphics[width=\linewidth]{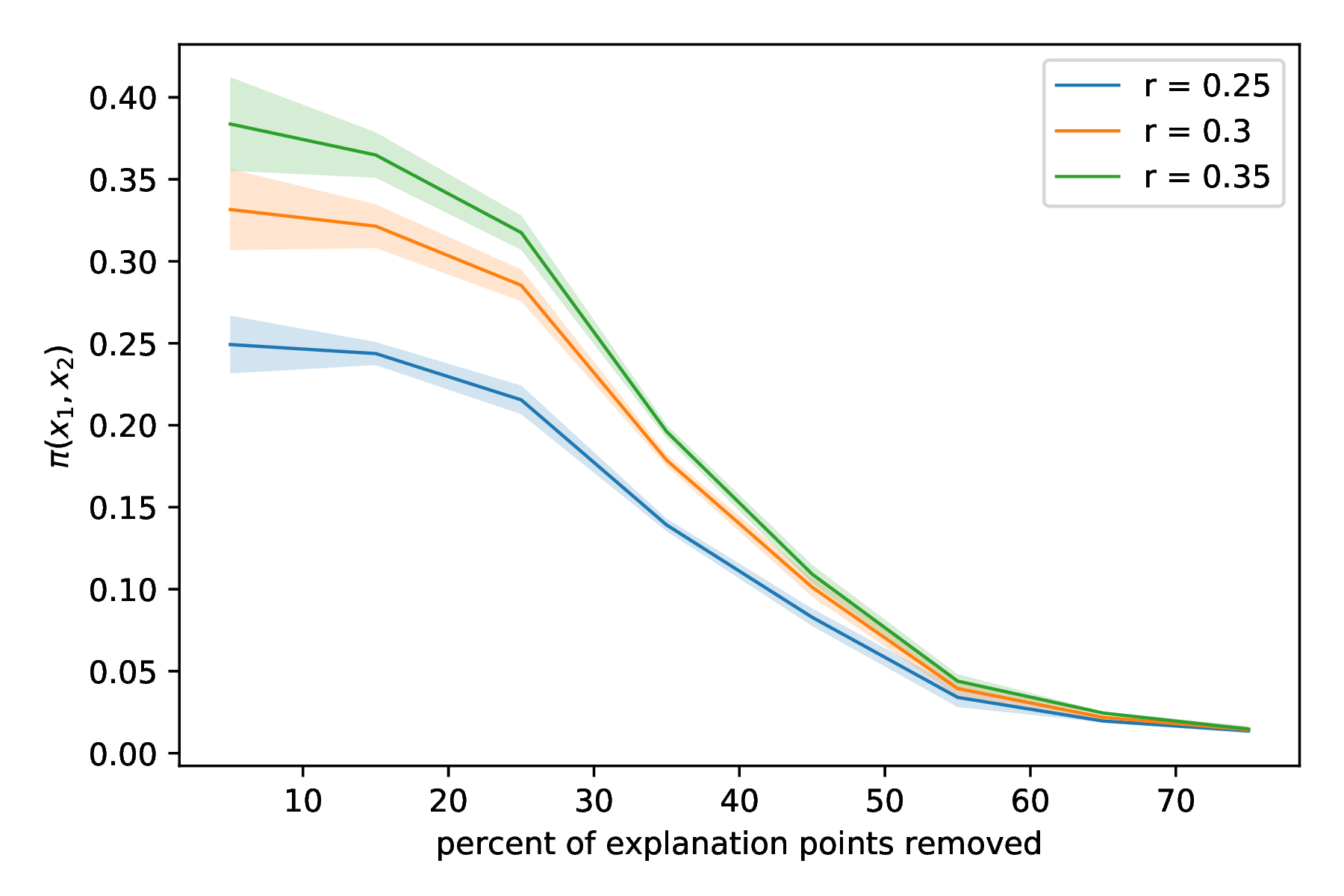}
  \end{subfigure}
  \caption{MLPs results with MMD-Critic explanations: max (left), top $5$ percentile average (middle), average $\pi(x,x')$ (right). We observe similar trends as in the $k$-medoid case with one difference being that the drop off rate is slower in the MMD-Critic case.}
  \label{fig:mmd_mlp}
\end{figure*}

\begin{figure}
  \begin{subfigure}[b]{.33\textwidth}
    \centering
    \includegraphics[width=\linewidth]{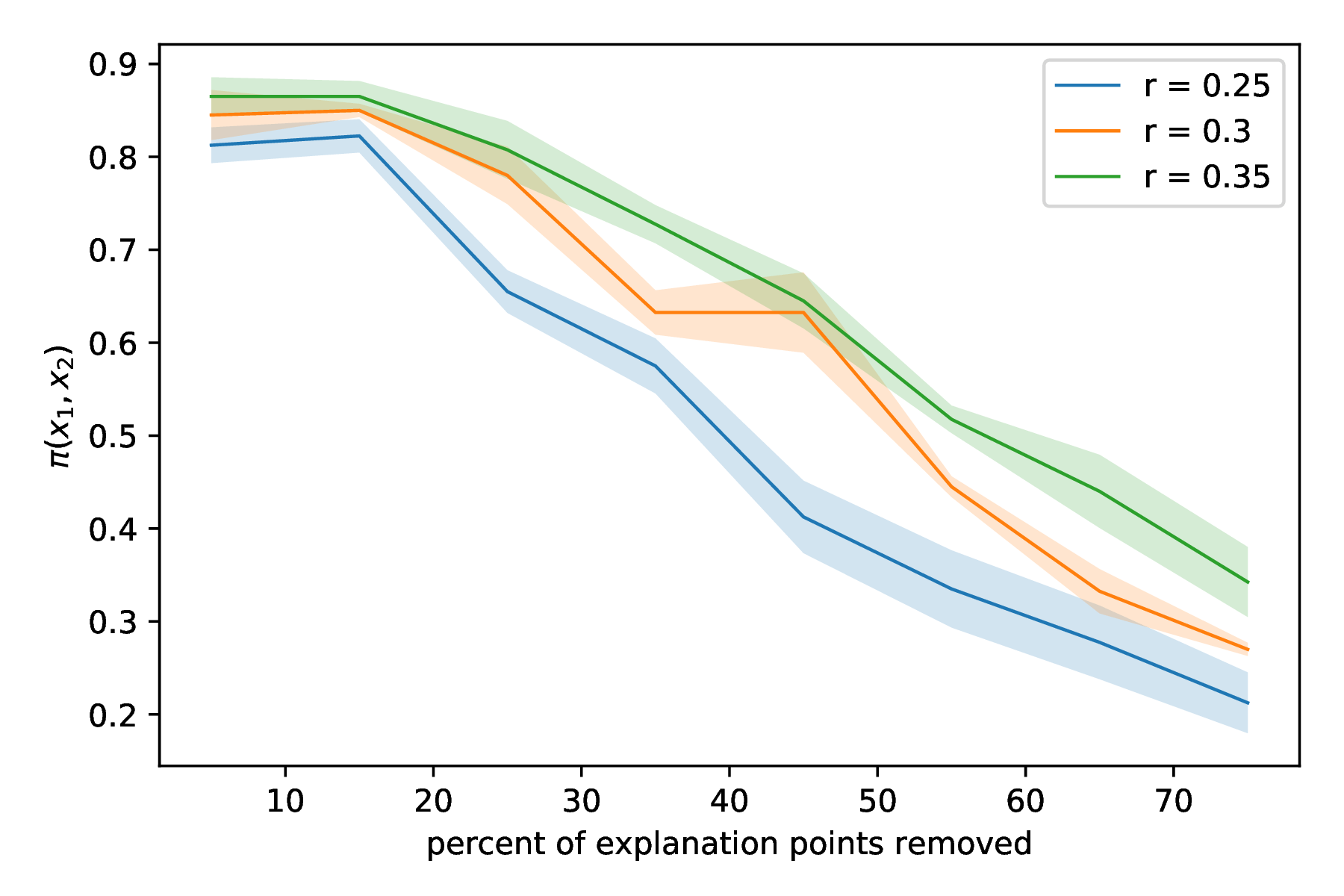}
  \end{subfigure}
  \begin{subfigure}[b]{.33\textwidth}
    \centering
    \includegraphics[width=\linewidth]{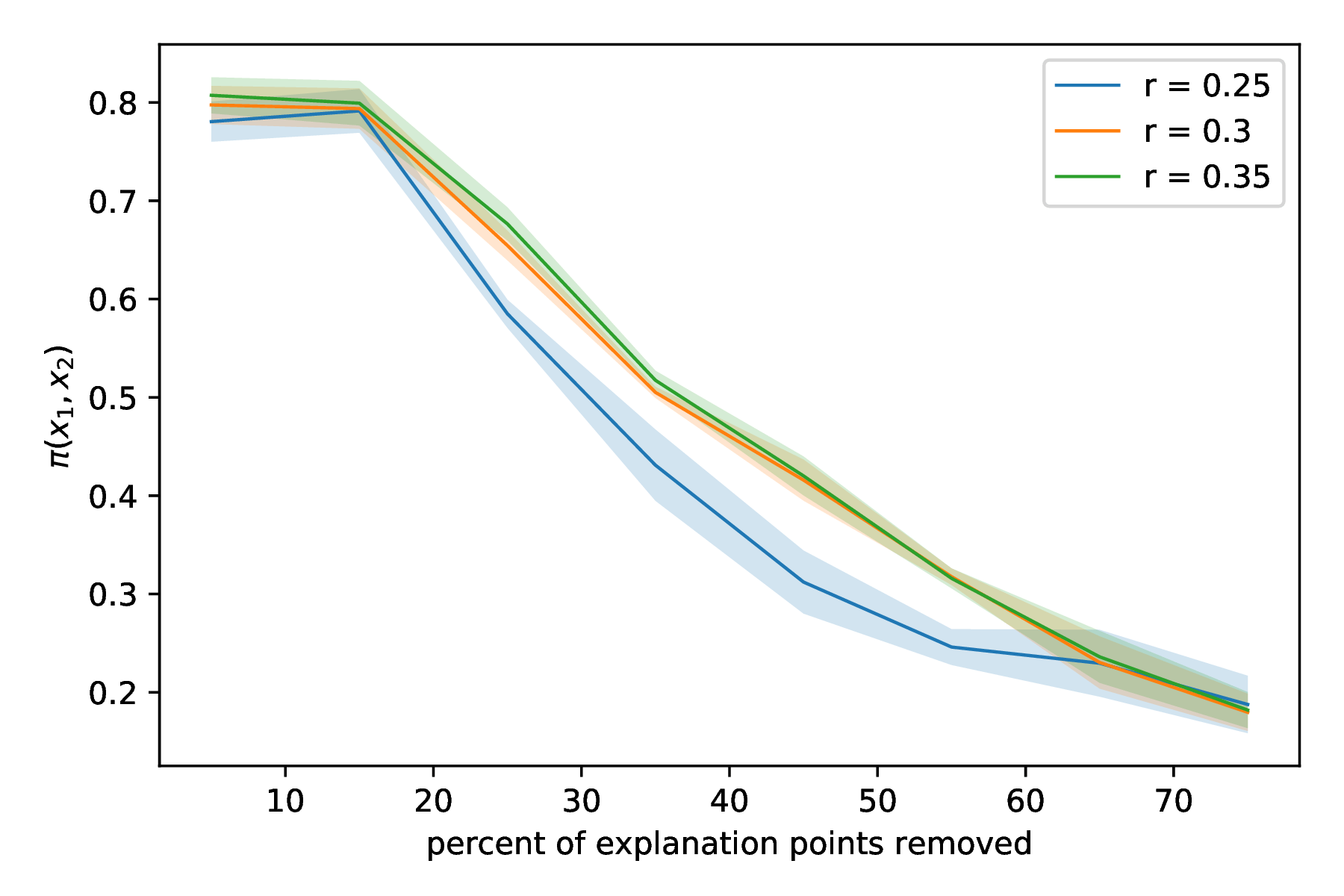}
  \end{subfigure}
  \begin{subfigure}[b]{.33\textwidth}
    \centering
    \includegraphics[width=\linewidth]{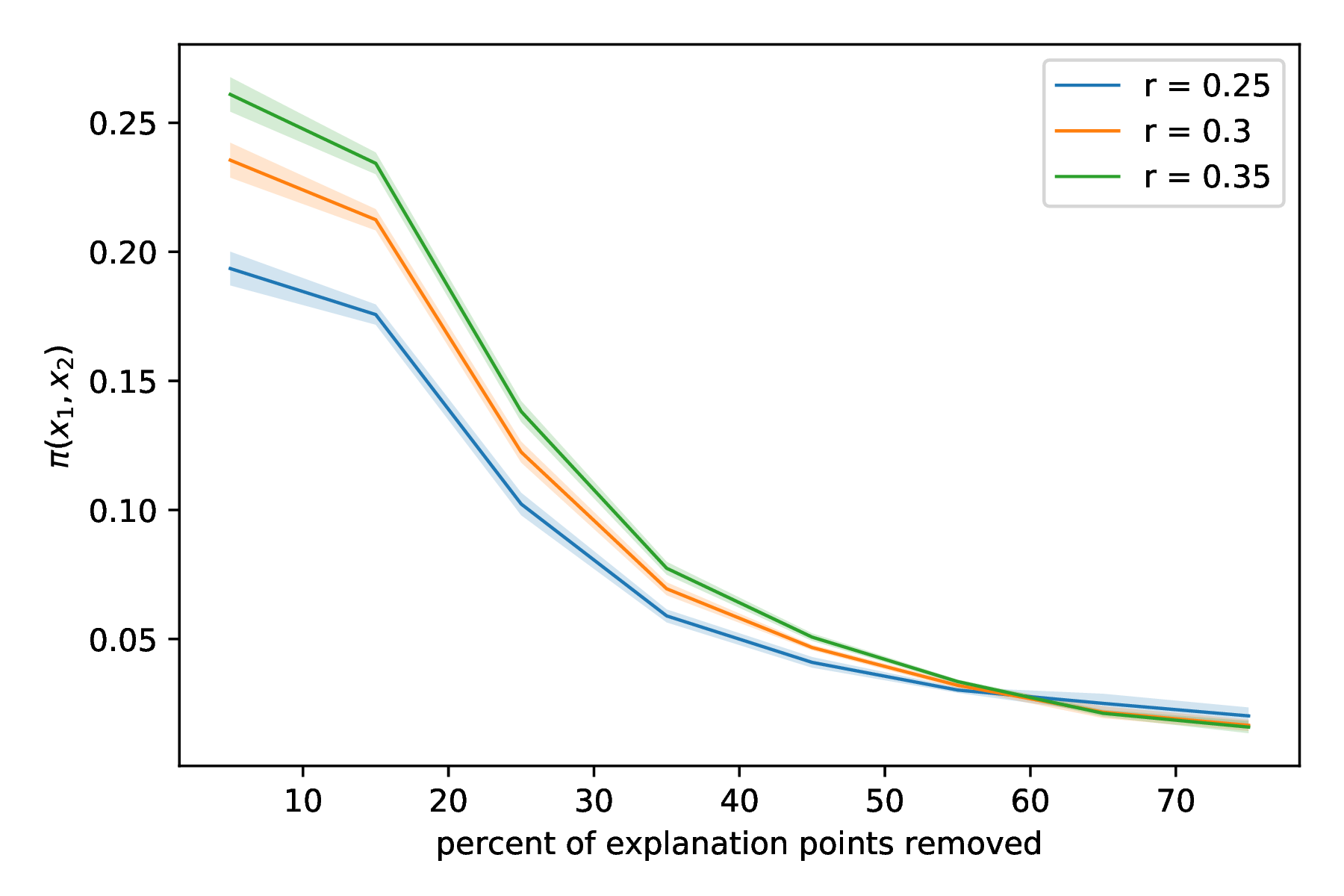}
  \end{subfigure}
  \smallskip
  \begin{subfigure}[b]{.33\textwidth}
    \centering
    \includegraphics[width=\linewidth]{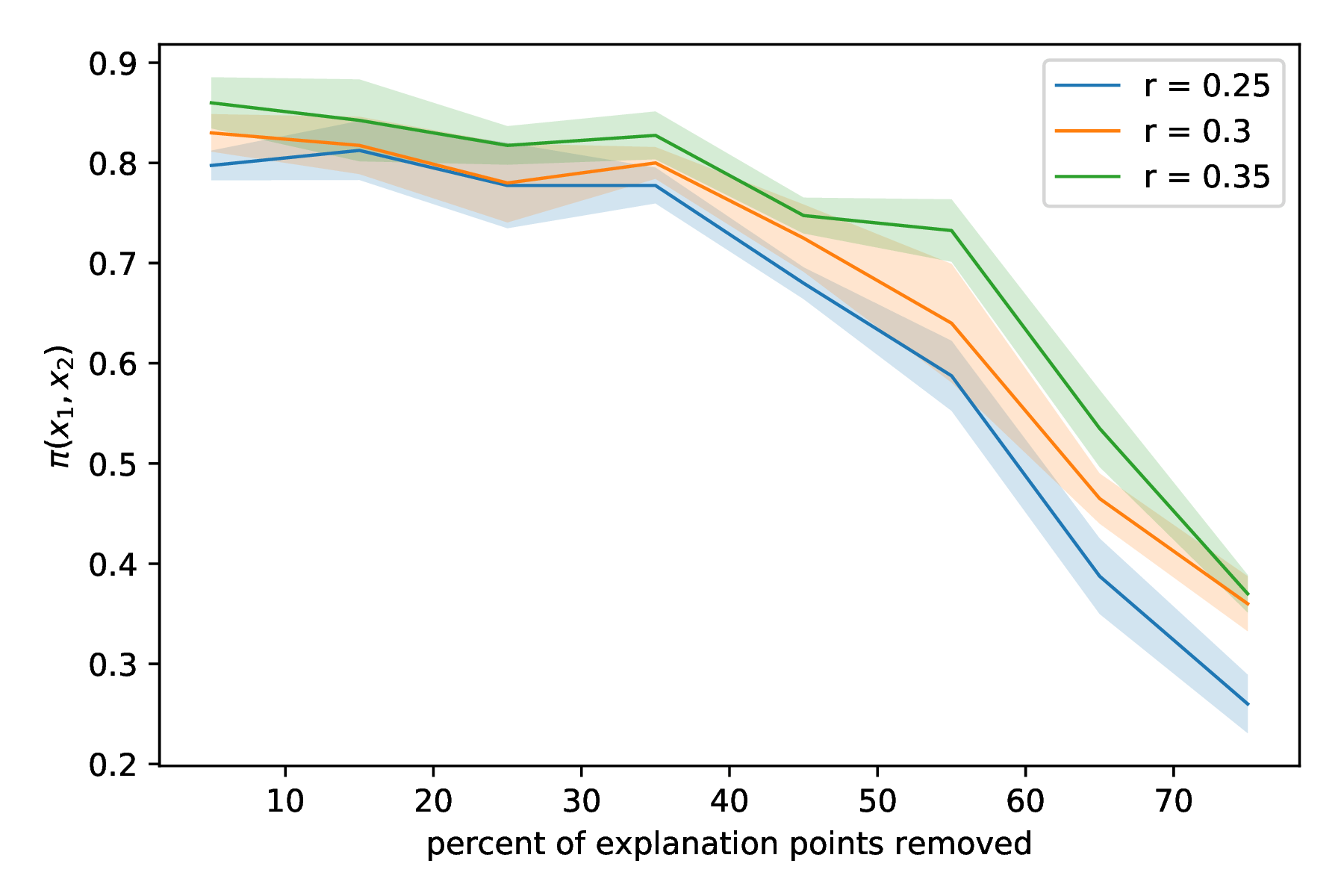}
  \end{subfigure}
  \begin{subfigure}[b]{.33\textwidth}
    \centering
    \includegraphics[width=\linewidth]{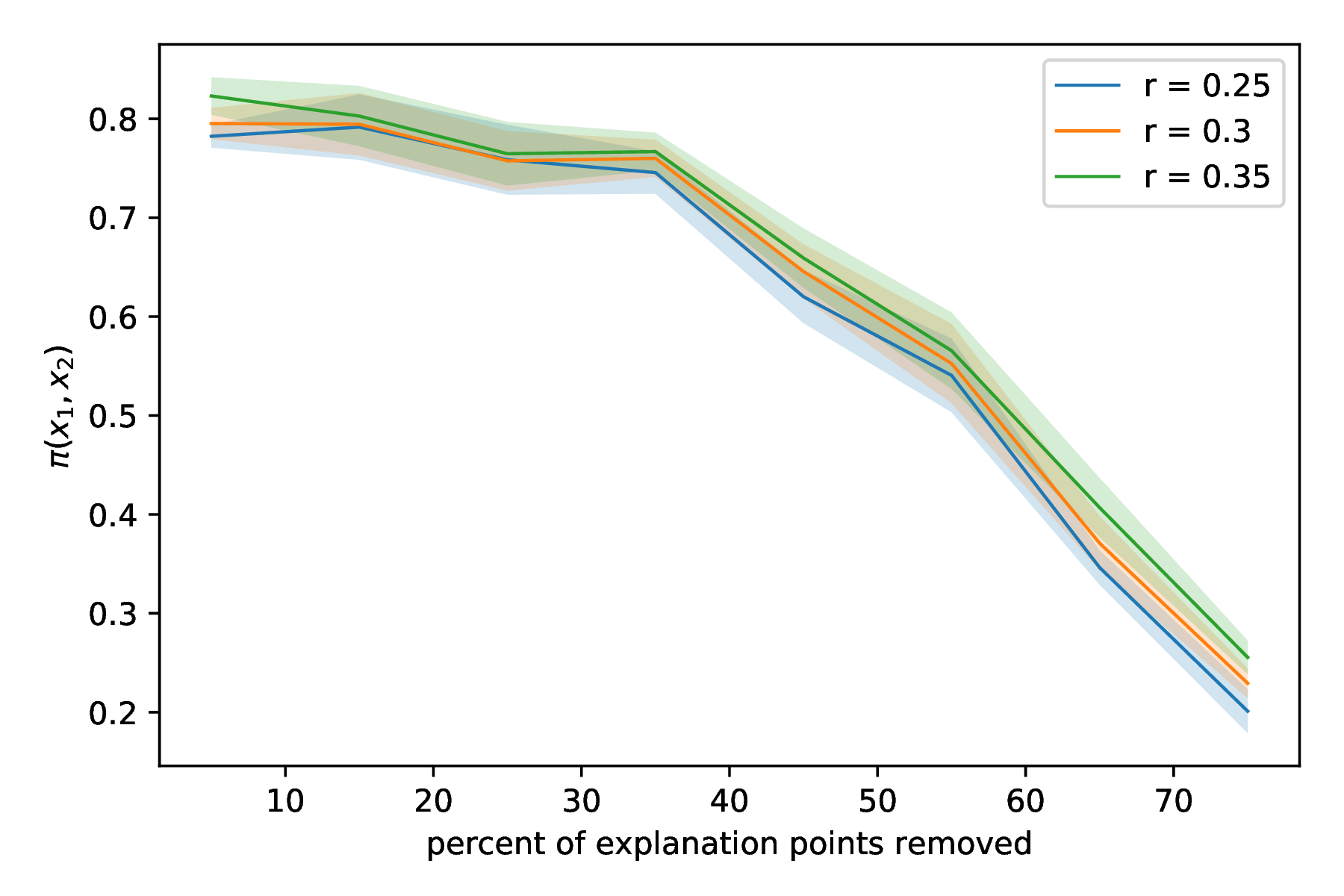}
  \end{subfigure}
  \begin{subfigure}[b]{.33\textwidth}
    \centering
    \includegraphics[width=\linewidth]{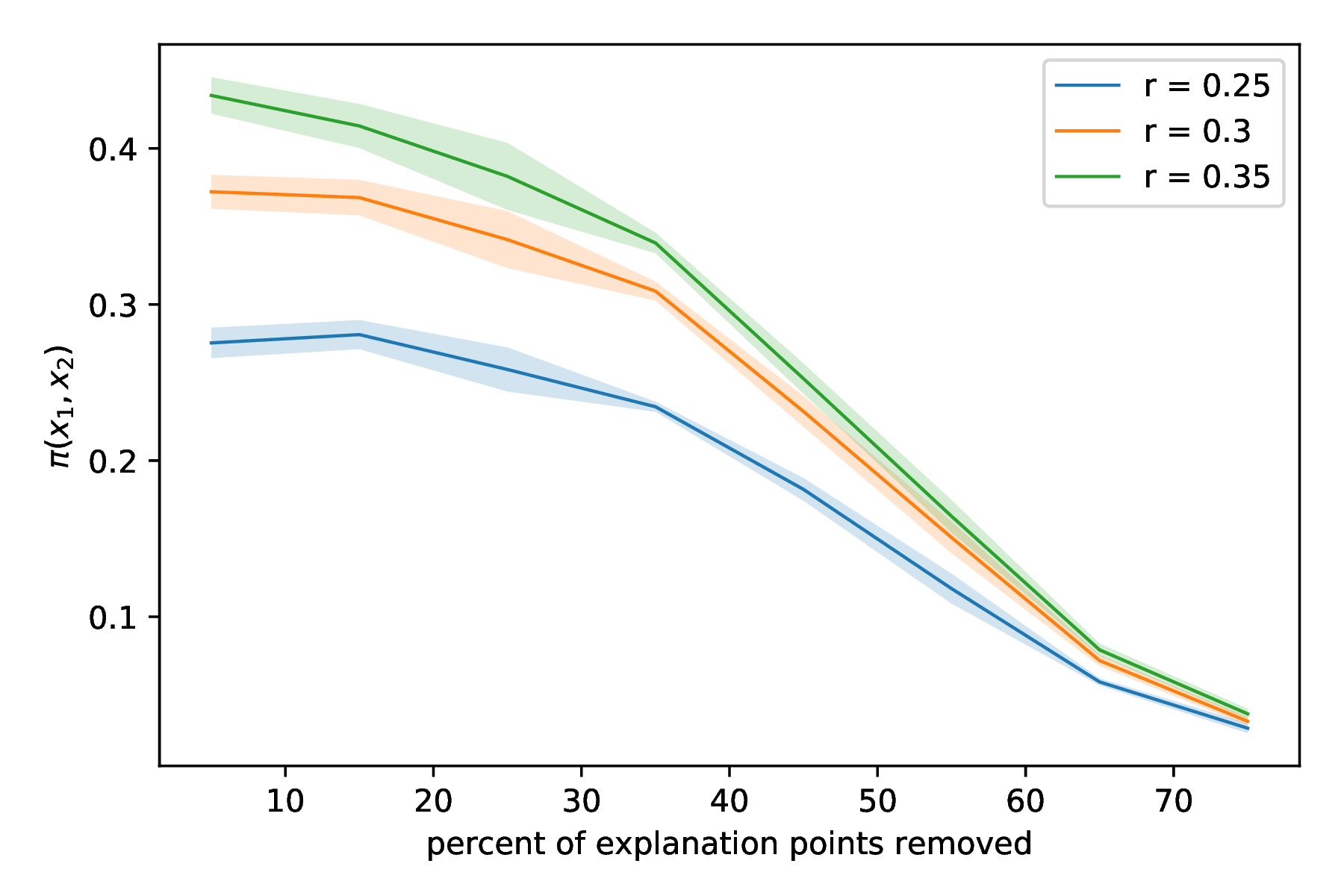}
  \end{subfigure}
  \caption{Two layer MLP results: under $k$-medoid explanations (top), under MMD explanations (bottom). The three metrics are in column: max (left), top $5$ percentile average (middle), average $\pi(x,x')$ (right).}
  \label{fig:mlp_bigger}
\end{figure}

\section{Additional Modeling Discussion}

One objection with our modeling assumption could be that if it is the case that most of the $\features$ is in $\margin$, then margin-distancing could remove most of the representative-based explanations $\explanation(\features)$. We assume this is not the case and that $\margin$ is only a small fraction of $\features$. 

Indeed, this assumes that the feature collection and modeling is done well and that most points are not within $r$ of another point with the opposite label. 



\section{Additional Related Works}
\label{sec:add-related-works}

\textbf{Improvement vs Gaming:} A crucial point about feature alteration is whether to think of it as causal (beneficial) or gaming \citep{miller2020strategic}. In our setting, the organization first offers individuals transparency into how the model ``works'' and predicts based on the reported features. We assume individuals are not aware of the underlying causal model. Hence, we view misreporting in the first stage as gaming.

\textbf{Explanation Manipulation:} There has been work focusing on how organizations may manipulate an unfair model's explanation to make it look more fair than it actually is \cite{aivodji2019fairwashing, anders2020fairwashing, slack2020fooling}. By contrast, we study how to provide explanations that are informative and cover as much of $\features$ as possible while protecting boundary points' label information.

\textbf{Security of ML models:} Our work is also related to model extraction literature \cite{tramer2016stealing, milli2019model} that assumes one can query an API for model prediction/gradient-based explanation on any point. We view our work as a study on how to ``limit'' the API so as to prevent a new type of attack -- individual-level gaming, which need not require the full model extraction in order to carry out the attack \cite{jagielski2020high}.

\textbf{Model Multiplicity:} The set of models consistent with labelled data is also referred to as version space~\cite{mitchell1977version}. Our paper thus pertains to a recent line of work highlighting the existence of the ``Rashomon effect'' \cite{semenova2019study, d2020underspecification} or model multiplicity \cite{marx2020predictive}. These papers do not focus on strategic manipulation, but study or raise the importance of developing sampling algorithms that can explore the version space.

\end{document}